%% file: main.tex
\documentclass{article}

\usepackage[preprint, nonatbib]{neurips_arxiv}

\usepackage[utf8]{inputenc} 
\usepackage[T1]{fontenc}    
\usepackage{hyperref}       
\usepackage{url}            
\usepackage{booktabs}       
\usepackage{amsfonts}       
\usepackage{nicefrac}       
\usepackage{microtype}      
\usepackage{xcolor}         
\usepackage{multirow}
\usepackage{algorithm}
\usepackage{algorithmic}
\usepackage{amsthm}

\usepackage{adjustbox}

\newcommand{\defeq}{\stackrel{\text { def. }}{=}}

\newcommand{\xmark}{}

\newcommand{\calD}{{\cal D}}

\newcommand{\calF}{{\cal F}}

\newcommand{\calX}{{\cal X}}
\newcommand{\calY}{{\cal Y}}
\newcommand{\calZ}{{\cal Z}}

\renewcommand{\hat}{\widehat}

\input{math_commands.tex}

\newcommand{\methodname}{\texttt{FWC}}

\usepackage[capitalize,noabbrev]{cleveref}

\newtheorem{theorem}{Theorem}[section]
\newtheorem{proposition}[theorem]{Proposition}
\newtheorem{lemma}[theorem]{Lemma}

\title{Fair Wasserstein Coresets}

\author{%
  Zikai Xiong\thanks{Operations Research Center, Massachusetts Institute of Technology, \texttt{zikai@mit.edu}}$\,$ $^\ddagger$
  \And
  Niccol\`o Dalmasso\thanks{J.P.Morgan AI Research, $\{$\texttt{niccolo.dalmasso, shubham.x2.sharma, freddy.lecue, daniele.magazzeni, vamsi.k.potluru, tucker.balch, manuela.veloso}$\}$\texttt{@jpmchase.com}}$\,$ $^\star$
  \And
  Shubham Sharma\footnotemark[2]
  \And
  Freddy LeCue\footnotemark[2]
  \And
  Daniele Magazzeni\footnotemark[2]
  \And
  Vamsi K. Potluru\footnotemark[2]
  \And
  Tucker Balch\footnotemark[2]
  \And
  Manuela Veloso\footnotemark[2]
}

\begin{document}

\maketitle

\begin{abstract}
Data distillation and coresets have emerged as popular approaches to generate a smaller representative set of samples for downstream learning tasks to handle large-scale datasets. At the same time, machine learning is being increasingly applied to decision-making processes at a societal level, making it imperative for modelers to address inherent biases towards subgroups present in the data. While current approaches focus on creating fair synthetic representative samples by optimizing local properties relative to the original samples, their impact on downstream learning processes has yet to be explored.  In this work, we present fair Wasserstein coresets (\texttt{FWC}), a novel coreset approach which generates fair synthetic representative samples along with sample-level weights to be used in downstream learning tasks. \texttt{FWC} uses an efficient majority minimization algorithm to minimize the Wasserstein distance between the original dataset and the weighted synthetic samples while enforcing demographic parity. We show that an unconstrained version of \texttt{FWC} is equivalent to Lloyd's algorithm for k-medians and k-means clustering. Experiments conducted on both synthetic and real datasets show that \texttt{FWC}:  (i) achieves a competitive fairness-utility tradeoff in downstream models compared to existing approaches, (ii) improves downstream fairness when added to the existing training data and (iii) can be used to reduce biases in predictions from large language models (GPT-3.5 and GPT-4).
\end{abstract}

\section{Introduction}

In the last decade, the rapid pace of technological advancement has provided the ability of collecting, storing and processing massive amounts of data from multiple sources \cite{sagiroglu2013big}. As the volume of data continues to surge, it often surpasses both the available computational resources as well as the capacity of machine learning algorithms. In response to this limitation, dataset distillation approaches aim to reduce the amount of data by creating a smaller, yet representative, set of samples; see \cite{yu2023dataset, lei2023comprehensive} for comprehensive reviews on the topic. Among those approaches, coresets provide a weighted subset of the original data that achieve similar performance to the original dataset in (usually) a specific machine learning task, such as clustering \cite{har2004coresets,feldman2020core}, Bayesian inference \cite{campbell2018bayesian}, online learning \cite{borsos2020coresets} and classification \cite{coleman2019selection}, among others. 

In tandem with these developments, the adoption of machine learning techniques has seen a surge in multiple decision-making processes that affect society at large \cite{sloane2022SiliconValleyLove, zhang2022ShiftingMachineLearning}.  This proliferation of machine learning applications has highlighted the need to mitigate inherent biases in the data, as these biases can significantly impact the equity of machine learning models and their decisions \cite{chouldechova2017fair}. Among many definitions of algorithmic fairness, demographic parity is one of the most prominently used metric \cite{hort2022BiasMitigationMachinea}, enforcing the distribution of an outcome of a machine learning model to not differ dramatically across different subgroups in the data.

Current methodologies for generating a smaller set of fair representative samples focus on the local characteristics of these samples with respect to the original dataset. For instance, \cite{chierichetti2017fair, huang2019coresets, backurs2019scalable, ghadiri2021socially} obtain representative points by clustering while enforcing each cluster to include the same proportion of points from each subgroup in the original dataset. In another line of work, \cite{jung2019center, mahabadi2020individual, negahbani2021better, vakilian2022improved, chhaya2022coresets} create representative points by ensuring that points in the original dataset each have at least one representative point within a given distance in the feature space. While these methods can successfully reduce clustering cost and ensure a more evenly spread-out distribution of representative points in the feature space, it is unclear whether such representative samples can positively affect performance or discrimination reduction in downstream learning processes. As the induced distribution of the representative points might be far away from the original dataset distribution, downstream machine learning algorithm might lose significant performance due to this distribution shift, without necessarily reducing biases in the original data (as we also demonstrate empirically in our experiments).

\paragraph{Contributions} In this work, we introduce \textbf{F}air \textbf{W}asserstein \textbf{C}oresets (\methodname), a novel coreset approach that not only generates synthetic representative samples but also assigns sample-level weights to be used in downstream learning tasks. \methodname\ generates synthetic samples by minimizing the Wasserstein distance between the distribution of the original datasets and that of the weighted synthetic samples, while simultaneously enforcing an empirical version of demographic parity. The Wasserstein distance is particularly suitable to this task due to its various connections with downstream learning processes and coresets generation (Section~\ref{sec: background}). Our contributions are as follows:

\begin{enumerate}
    \item we show how the \methodname\ optimization problem can be reduced to a nested minimization problem in which fairness constraints are equivalent to linear constraints (Section~\ref{sec: fair-wasserstein-coresets});
    \item we develop an efficient majority minimization algorithm \cite{ortega2000iterative, lange2016mm} to solve the reformulated problem (Section~\ref{sec: majority-minimization}). We analyze theoretical properties of our proposed algorithm and \methodname\ (Section~\ref{sec: convergence-guarantees}) and show that, in the absence of fairness constraints, our algorithm reduces to an equivalent version of Lloyd's algorithm for k-means and k-medians clustering, extending its applicability beyond fairness applications (Section~\ref{sec: lloyd-algo});
    \item we empirically validate the scalability and effectiveness of \methodname\ by providing experiments on both synthetic and real datasets (Section~\ref{sec: experiments}). In downstream learning tasks, \methodname\ result in competitive fairness-utility tradeoffs against current approaches, even when we enhance the fairness of existing approaches using fair pre-processing techniques, with an average disparity reduction of $53\%$ and $18\%$, respectively. In addition, we show \methodname\ can correct biases in large language models when passing coresets as examples to the LLM, reducing downstream disparities by $75\%$ with GPT-3.5 \cite{GPT35} by $35\%$ with GPT-4 \cite{achiam2023gpt}. Finally, we show \methodname\ can improve downstream fairness-utility tradeoffs in downstream models when added to the training data (via data augmentation, see Appendix~\ref{supp: exp-real-dataset-subsec}).
\end{enumerate}  

Finally, we refer the reader to the Appendix for more details on the optimization problem (Section~\ref{supp: cutting plane}), theoretical proofs (Section~\ref{supp: theory-proofs}) and further experiments and details (Section~\ref{supp: exp-details}).



\section{Background and Related Work} \label{sec: background}

\paragraph{Notation} 
We indicate the original dataset samples $\{Z_i\}_{i=1}^n$, with $Z_i = \left(D_i, X_i, Y_i\right) \in \calZ = (\calD \times \calX \times \calY) \subseteq \mathbb{R}^d$, where $X$ indicates the non-sensitive features, $D$ indicates one or more protected attributes such as ethnicity or gender, and $Y$ is a decision outcome. In this work, we assume $D$ and $Y$ to be discrete features, i.e., to have a finite number of levels so that $|\calD|\ll n$ and $|\calY| \ll n$. For example, $Y$ might indicate a credit card approval decision based on credit history $X$, with $D$ denoting sensitive demographic information. 
Given a set of weights $\{\theta\}_{i=1}^n$, define $p_{Z; \theta}$ the weighted distribution of a dataset $\{Z_i\}_{i=1}^n$ as 
$
p_{Z;\theta} = \frac{1}{n}\sum_{i=1}^n \theta_i \delta_{Z_i},
$
where $\delta_x$ indicates the Dirac delta distribution, i.e., the unit mass distribution at point $x \in \mathcal{X}$. Using this notation, we can express the empirical distribution of the original dataset by setting $\theta_i = e_i = 1$ for any $i$, i.e., $p_{Z;e} = \frac{1}{n} \sum_{i=1}^n e_i \delta_{Z_i}$.
For a matrix $A, A^{\top}$ denotes its transpose. For two vectors (or matrices) $\langle u, v\rangle \stackrel{\text { def. }}{=} \sum_i u_i v_i$ is the canonical inner product (the Frobenius dot-product for matrices). We define $\mathbf{1}_m \stackrel{\text { def. }}{=}(1 , \ldots, 1) \in$ $\mathbb{R}_{+}^m$. 

\paragraph{Wasserstein distance and coresets}
Given two probability distributions $p_1$ and $p_2$ over a metric space $\mathcal{X}$, the Wasserstein distance, or optimal transport metric, quantifies the distance between the two distributions as solution of the following linear program (LP):

\begin{equation}\label{pro original wasserstrein distance}
	\mathcal{W}_c(p_1, p_2) \stackrel{\text{def.}}{=} \min_{\pi \in \Pi(p_1, p_2)} \int_{\mathcal{X} \times \mathcal{X}} c(x_1, x_2) \mathrm{d} \pi(x_1, x_2),
\end{equation}

with $\Pi(p_1, p_2)$ indicating the set of all joint probability distributions over the product space $\mathcal{X} \times \mathcal{X}$ with marginals equal to $(p_1, p_2)$ \cite{kantorovitch1958translocation}. The operator $c(x,y)$ represents the ``cost'' of moving probability mass from $x$ and $y$, and reduces to a matrix $C$ if the underlying metric space $\calX$ is discrete.

The Wasserstein distance has several connections with downstream learning processes and coresets. Firstly, the higher the Wasserstein distance between the weighted representative samples $p_{\hat{Z}; \theta}$ and the original dataset $p_{Z;e}$, the higher the distribution shift, the more degradation we might expect in terms of downstream learning performance \cite{quinonero2022dataset}. Secondly, the Wasserstein distance between two probability distributions can also be used to bound the deviation of functions applied to samples from such distributions. Define the following deviation:
$$
d(p_{\hat{Z}; \theta}, p_{Z;e})\defeq \sup_{f\in\calF} \left|\mathbb{E}_{z\sim p_{\hat{Z}; \theta}} f(z)- \mathbb{E}_{z\sim p_{Z;e}} f(z) \right| \ .
$$

When $\calF$ is the class of Lipschitz-continuous functions with Lipschitz constant equal or less than $1$, the deviation $d(p_{\hat{Z}; \theta}, p_{Z;e})$ is equal to 1-Wasserstein distance $\mathcal{W}_1(p_{\hat{Z}; \theta}, p_{Z;e})$ \cite{santambrogio2015optimal,villani2009optimal}. The connection with learning processes and the downstream deviation $d(p_{(\hat{X}, \hat{D}); \theta}, p_{(X, D);e})$ is immediate when considering Lipschitz continuous classifiers with Lipschitz constant less than 1 (such as logistic regression or Lipschitz-constrained neural networks~\cite{anil2019sorting})\footnote{In such cases, $d(p_{(\hat{X}, \hat{D}); \theta}, p_{(X, D);e}) = \mathcal{W}_1(p_{(\hat{X}, \hat{D}); \theta}, p_{(X, D);e}) \leq \mathcal{W}_1(p_{\hat{Z}; \theta}, p_{Z;e})$, see Lemma~\ref{lemma: wass-dist-ub}.}. For other classifiers, we note that the 1-Wasserstein distance still bounds the downstream discrepancy: in Proposition~\ref{prop: wass-neuralnet} we show that the 1-Wasserstein distance bounds the downstream discrepancy for ReLu-activated multilayer perceptrons (MLPs), which we use in our experiments on real datasets (Section~\ref{sec: experiments}).

\begin{proposition}\label{prop: wass-neuralnet}
Let $g_\psi \in \mathcal{G}^K$ be the class of $K$-layer multilayer perceptrons with ReLu activations. Then, the downstream discrepancy in downstream performance of $g_\psi$ applied to samples from $p_{(\hat{X}, \hat{D}); \theta}$ and $p_{(X, D);e}$ is bounded by the 1-Wasserstein distance :

\begin{align}
    d(p_{(\hat{X}, \hat{D}); \theta}, & p_{(X, D);e}) = \notag \\ \sup_{g_\psi \in \mathcal{G^K}} &\left|\mathbb{E}_{(x, d)\sim p_{(\hat{X}, \hat{D}); \theta}} g_\psi(x, d)- \mathbb{E}_{(x,d) \sim p_{(X, D);e}} g_\psi(x, d) \right| \leq L_k \mathcal{W}_1(p_{\hat{Z}; \theta}, p_{Z;e}), \label{prop: downstream-discrepancy-equation}
\end{align}

where $L_k$ is the MLP Lipschitz constant upper bound defined in \cite[Section 6.1, Equation (8)]{virmaux2018lipschitz}.
\end{proposition}

We point out that minimizing the Wasserstein distance, hence bounding the downstream performance as in Equation (\ref{prop: downstream-discrepancy-equation}), is equivalent to the definition of a measure coresets proposed by \cite{claici2018wasserstein}. The Wasserstein distance is also connected to coresets in the sense of the best discrete approximation of a continuous distributions. Considering a feature space endowed with a continuous distribution, minimizing the $p$-Wasserstein distance across all the distributions of size $m$ is topologically equivalent to identify the samples of size $m$ that provide the best Voronoi tessellation of the space in $L_p$ sense \cite{pages2015introduction, liu2020characterization}. Although other definitions of coresets using Wasserstein distance have been proposed in the literature, they require either to solve the underlying optimal transport problems or the knowledge of the downstream classifier and loss function \cite{claici2018wasserstein, mirzasoleiman2020coresets, zhao2021dataset, loo2022efficient}. \methodname\ is agnostic to any downstream model or loss function, and uses an efficient implementation that does not actually incur in the usual high cost connected to optimal transport. 
By adapting the approach proposed by \cite{Xiong2023FairWASP}, we solve an equivalent re-formulated linear optimization problem, which is more computationally tractable than classic approaches such as the simplex or the interior point method (see Section~\ref{sec: fairwasp}).

\paragraph{Demographic parity}
Also known as statistical parity, demographic parity (DP) imposes the decision outcome and protected attributes to be independent \cite{dwork2012FairnessAwareness}. Using a credit card approval decision example, demographic parity enforces an automatic decision process to approve similar proportion of applicants across different race demographics. DP is one of the most extensively analyzed fairness criterion; we refer the reader to \cite{hort2022BiasMitigationMachinea} for a review. In this work, we use the demographic parity definition by \cite{calmon2017optimized}, which enforces the ratio between the conditional distribution of the decision outcome across each subgroups $p(y | D=d)$ and the marginal distribution of the decision outcome $p(y)$ to be close to $1$ in a given dataset. We refer to the ``empirical version'' of demographic parity to indicate that the conditional and marginal distributions are quantities estimated from the data. Note that the demographic parity definition we adopt enforces a condition on the weighted average of the conditional distributions across groups, which is different from a more recent approach of using Wasserstein distance, and specifically Wasserstein barycenters, to enforce demographic parity \cite{gordaliza2019obtaining, jiang2020wasserstein, gaucher2023fair, xu2023fair}.




\section{\methodname: Fair Wasserstein Coresets} \label{sec: fair-wasserstein-coresets}


Given a dataset $\{Z_i\}_{i=1}^n$, our goal is to find a set of samples $\{\hat{Z}_j\}_{j=1}^m$ and weights $\{ \theta_j\}_{j=1}^m$ such that $m\ll n$ and that the Wasserstein distance between $p_{Z;e}$ and $p_{\hat{Z};\theta}$ is as small as possible.
In addition, we use the fairness constraints proposed by \cite{calmon2017optimized} to control the demographic parity violation for the $p_{\hat{Z};\theta}$ distribution. Let $p_{\hat{Z};\theta}(y | d)$ indicate the conditional distribution $p_{\hat{Z};\theta}(\hat{Y} = y | \hat{D} = d)$. Imposing a constraint on the demographic disparity violation then reduces to requiring the conditional distribution under the weights $\{\theta_i\}_{i\in[n]}$ to be close to a target distribution $p_{Y_T}$ across all possible values of the protected attributes $D$,
\begin{equation}\label{type1}
	J\left(p_{\hat{Z};\theta}(y | d), p_{Y_T}(y)\right) = \left| \frac{p_{\hat{Z};\theta}(y | d)}{p_{Y_T}(y)} - 1\right| \leq \epsilon, \ \forall \ d \in \mathcal{D}, y \in\calY,
\end{equation}
where $\epsilon$ is a parameter that determines the maximum fairness violation, and $J(\cdot, \cdot)$ is the probability ratio between distributions as defined in \cite{calmon2017optimized}.

Using the notation above, our goal can then be formulated as the following optimization problem:
\begin{equation}\label{pro general optimization}
    \begin{aligned}
    & \min_{\theta \in\Delta_m, \hat{Z} \in \calZ^m} \ \mathcal{W}_c(p_{\hat{Z};\theta},  p_{Z;e}) \\
    &\ \  \ \quad \ \text{s.t. } \ \ J\left(p_{\hat{Z};\theta}(y | d), p_{Y_T}(y)\right) \leq \epsilon, \ \forall \ d \in \mathcal{D}, y \in\calY,
    \end{aligned}
\end{equation}
 where $\Delta_m$ indicates the set of valid weights $\{\theta \in \mathbb{R}^m_+: \sum_{i=1}^m \theta_i = m\}$. Note that the optimization problem in (\ref{pro general optimization}) shares some similarities with the optimization problem in \cite{Xiong2023FairWASP} in using the Wasserstein distance as a distance metric between distributions and using (\ref{type1}) to enforce demographic parity. However, \cite{Xiong2023FairWASP} only provide sample-level integer weights for the original dataset and do not generate any new samples, while our approach provides a separate set of samples $\{\hat{Z}_j\}_{j=1}^m$ with associated real-valued weights $\{ \theta_j\}_{j=1}^m$, with $m \ll n$.

 We now take the following steps to solve the optimization problem in (\ref{pro general optimization}): (i) we reduce the dimensionality of the feasible set by fixing $\hat{Y}$ and $\hat{D}$ a priori, (ii) we formulate the fairness constraints as linear constraints, (iii) we add artificial variables to express the objective function and (iv) we simplify the optimization problem to minimizing a continuous non-convex function of the $\{\hat{X}_j\}_{j=1}^m$. 

 \paragraph{Step 1. Reduce the feasible set of the optimization problem }

As in practice all possible $Y_i$ and $D_i$ are known a priori, and there are only a limited number of them, we can avoid optimizing over them and instead manually set the proportion of each combination of $\hat{Y}$ and $\hat{D}$. This reduces the optimization problem feasible set only over $\Delta_m$ and $\calX^m$. The following lemma shows that this in fact does not affect the optimization problem:


\begin{lemma}\label{lem: optimization-redux}
    For any $m > 0$, the best fair Wasserstein coreset formed by $m$ data points $\{\hat{Z}_i: i\in[m]\}$ is no better (i.e., the optimal Wasserstein distance value is no lower) than the best fair Wasserstein coreset formed by $m|\calD||\calY|$ data points $ \{(d,X_i, y)_i: i\in[m], d \in \calD, y \in \calY\} $.
\end{lemma}

Hence, we simply set the proportions of $\{(\hat{D}_i, \hat{Y})_i\}_{i\in[m]}$ in the coresets to be similar to their respective proportions in the original dataset.
The optimization problem then reduces to
\begin{equation}\label{pro general optimization model}
    \begin{aligned}
		\min_{\theta \in \Delta_m, \hat{X} \in \calX^m} \ & \mathcal{W}_c(p_{\hat{Z};\theta},  p_{Z;e})\\
		\text{ s.t.}\ \ \ \ \ \ \ \ & J\left(p_{\hat{Z};\theta}(y | d), p_{Y_T}(y)\right) \leq \epsilon, \  \forall \ d \in \mathcal{D}, y \in\calY \ ,
	\end{aligned}
\end{equation}
whose  solutions are the features in the coreset $\{\hat{X}_j\}_{j=1}^m$ and the corresponding weights $\{\theta_j\}_{j=1}^m$.

\paragraph{Step 2. Equivalent linear constraints}

Following \cite{Xiong2023FairWASP}, the fairness constraint in Equation (\ref{type1}) can be expressed as  $2 |\calY| |\calD|$ linear constraints on the weights $\theta$, as the disparity reduces to the following for all $d \in \mathcal{D}, y \in\calY$:
{\small\begin{equation*}
       {\sum_{i\in[m]:\\ \hat{D}_i = d,\hat{Y}_i=y}\theta_i} 
		\le (1+\epsilon)\cdot  p_{Y_T}(y) \cdot
		{\sum_{i\in[m]:\hat{D}_i = d}\theta_i} \ \text{,} \ {\sum_{i\in[m]:\hat{D}_i = d,\hat{Y}_i=y}\theta_i} 
		\ge (1 - \epsilon)\cdot p_{Y_T}(y) \cdot
		{\sum_{i\in[m]:\hat{D}_i = d}\theta_i} \ .
\end{equation*}}


We can express these by using a $2 |\calY| |\calD|$-row matrix $A$ as $A\theta \ge \mathbf{0}$. 

\paragraph{Step 3. Reformulate the objective function by introducing artificial variables}
When keeping the samples $\hat{X}$ fixed, we can follow \cite{peyre2019computational} to derive an equivalent formulation of the Wasserstein distance in the objective as a linear program with $m  n$ variables. By indicating the transportation cost matrix $C(\hat{X})$, we define its components as follows,
$$
C(\hat{X})_{ij} \defeq c(Z_i, \hat{Z}_j), \text{ for } i \in [n], j \in [m]\ .
$$
Note that $C(\hat{X})$ is a convex function of $\hat{X}$ when, e.g., using any $L^p$ norm to define the transportation cost.
Therefore, now the  problem  (\ref{pro general optimization model}) is equivalent to 
\begin{equation}\label{pro target problem}
	\begin{aligned}
		&\min_{\hat{X} \in \calX^m, \theta \in \Delta_m, P \in \mathbb{R}^{n\times m}} \  \langle C(\hat{X}) , P\rangle  
		\\
		& \quad \quad \quad \ \ \ \text{ s.t.} \ \ \   P \mathbf{1}_m = \frac{1}{n}\cdot\mathbf{1}_n, \ P^\top \mathbf{1}_n = \frac{1}{m}\cdot \theta, \ P \ge \mathbf{0}, \ A\theta \ge \mathbf{0} \ .
	\end{aligned}
\end{equation}

\paragraph{Step 4. Reduce to an optimization problem of $\hat{X}$}
As from one of the constraints we get $\theta= m \cdot P^\top \mathbf{1}_n$, we further simplify problem (\ref{pro target problem}) as: 
\begin{equation}\label{pro simplified target problem}
	\begin{aligned}
		\min_{\hat{X}\in\calX^m, P\in\mathbb{R}^{n\times m}} \ &  \langle C(\hat{X}) , P\rangle  
		\\
		\text{ s.t.} \ \ \ \ \ \  \ \ \ \ \ 
		& P \mathbf{1}_m =\frac{1}{n} \cdot \mathbf{1}_n, \  P \ge \mathbf{0} , \
		 A P^\top \mathbf{1}_n \ge \mathbf{0} \ . 
	\end{aligned}
\end{equation}
Let $F(C)$, as a function $F$ of $C$, denote the optimal objective value of the following optimization problem
\begin{equation}\label{pro define FC}
    	\begin{aligned}
		\min_{P\in \mathbb{R}^{n\times m}} \ &  \langle C , P\rangle  
		\\
		\text{ s.t.} \ \ \ \ \  &  P \mathbf{1}_m = \frac{1}{n}\cdot \mathbf{1}_n, \  P \ge \mathbf{0} , \ A P^\top \mathbf{1}_n \ge \mathbf{0} \ 
	\end{aligned}
\end{equation}
and then problem (\ref{pro simplified target problem}) is equivalent to
\begin{equation}\label{pro simplified target problem 2}
	\min_{\hat{X}\in\calX^m} \ F(C(\hat{X})) \ . 
\end{equation}
In (\ref{pro simplified target problem 2}) the objective is continuous but nonconvex with respect to $\hat{X}$.
Once the optimal $\hat{X}^\star$ is solved, then the optimal $P^\star$ of the problem (\ref{pro simplified target problem}) is obtained by solving problem (\ref{pro define FC}) with $C$ replaced with $C(\hat{X}^\star)$. Finally, the optimal $\theta^\star$ follows by the equation $\theta^\star =m\cdot  (P^\star)^\top \mathbf{1}_n$. We now provide a majority minimization algorithm for solving problem (\ref{pro simplified target problem 2}).

\section{Majority Minimization for Solving the Reformulated Problem}  \label{sec: majority-minimization}



Majority minimization aims at solving nonconvex optimization problems, and refers to the process of defining a convex surrogate function that upper bounds the nonconvex objective function, so that optimizing the surrogate function improves the objective function \cite{ortega2000iterative, lange2016mm}. As the algorithm proceeds, the surrogate function also updates accordingly, which ensures the value of the original objective function keeps decreasing.  
Following this framework, we define the surrogate function $g(\cdot; \hat{X}^k)$ as follows for the $k$-th iterate $\hat{X}^k \in \calX^m $:
\begin{equation}\label{eq surrogate function}
    g(\hat{X}; \hat{X}^k)\defeq \langle C(\hat{X}), P_k^\star \rangle \ ,
\end{equation}
in which $P_k^\star$ is the minimizer of  problem (\ref{pro define FC}) with the cost $C = C(\hat{X}^k)$\footnote{We show this surrogate function is adequate, i.e., is convex and an upper bound of the original objective function, in Section~\ref{sec: convergence-guarantees}.}.

With this surrogate function, Algorithm \ref{alg MM method} summarizes the overall algorithm to minimize problem (\ref{pro simplified target problem 2}). In each iteration of Algorithm \ref{alg MM method}, line 3 is straightforward since it only involves computing the new cost matrix using the new feature vectors $\hat{X}^k$. We separately discuss how to solve the optimization problems in lines 4 and 5 below.

\begin{algorithm*}
  \caption{Majority Minimization for Solving (\ref{pro simplified target problem 2})}\label{alg MM method}
  \begin{algorithmic}[1] 
    \STATE Initial feature vectors $\hat{X}^k$ and $k = 0$
    \WHILE{True}
      \STATE $C \gets C(\hat{X}^k)$;  $\hfill \rhd$ \textit{update the cost matrix $C$}
      \STATE $P_k^\star \gets $ optimal solution of problem (\ref{pro define FC}); $\hfill \rhd$  \textit{updating the surrogate function (Section~\ref{sec: fairwasp})}
    \STATE $\hat{X}^{k+1} \gets \arg\min_{\hat{X}\in\calX^m}g(\hat{X}; \hat{X}^k)$; $\hfill \rhd$ \textit{updating feature vectors (Section~\ref{sec: update-feature-vec})}
    \IF{ $g(\hat{X}^{k+1}; \hat{X}^k) = g(\hat{X}^k; \hat{X}^k)$ }  
        \STATE $\theta_k^\star \gets m\cdot  (P_k^\star)^\top \mathbf{1}_n;$ $\hfill \rhd$ \textit{if algorithm has converged, compute optimal weights}
        \STATE \textbf{return} $\hat{X}^k$, $\theta_k^\star$  $\hfill \rhd$ \textit{return coresets and sample-level weights}
    \ENDIF
    \STATE $k \gets k + 1$;
    \ENDWHILE
  \end{algorithmic}
\end{algorithm*}

\subsection{Updating the Surrogate Function (Line 4)}\label{sec: fairwasp} 

To update the surrogate function, we need to solve problem (\ref{pro define FC}), which is a large-scale linear program. Rather than solving the computationally prohibitive dual problem we solve a lower-dimensional dual problem by using a variant of the \texttt{FairWASP} algorithm proposed by \cite{Xiong2023FairWASP}. We adapt \texttt{FairWASP} for cases where $m \neq n$ to find the solution of (\ref{pro define FC}) via applying the cutting plane methods on the Lagrangian dual problems with reduced dimension\footnote{As opposed to the scenario where $m=n$, which was tackled by \cite{Xiong2023FairWASP}.}. We choose \texttt{FairWASP} over established commercial solvers due to its computational complexity being lower than other state of the art approaches such as interior-point or simplex method; see Lemma~\ref{cor our complexity-app} in Appendix~\ref{supp: cutting plane} for more details.

\subsection{Updating Feature Vectors (Line 5)} \label{sec: update-feature-vec}
To update the feature vectors, we need to obtain the minimizer of the  surrogate function $g(\hat{X}; \hat{X}^k)$, i.e.,
\begin{equation}\label{eq minimizer of surrogate function main paper}
    \min_{\hat{X} \in \calX^m} g(\hat{X}; \hat{X}^k) \ .
\end{equation}  
The above can be written as the following problem:
\begin{equation}\label{pro equivalent minimization}
    \min_{\hat{X}_j \in \calX: j\in[m]}  \sum_{i\in[n]}\sum_{j\in[m]}c({Z}_i,\hat{Z}_j)P_{ij}     \ 
\end{equation}
for $P = P_k^\star$, in which each component of $P$ is nonnegative and $\hat{Z}_j = (\hat{d}_j, \hat{X}_j, \hat{y}_j)$, for the known fixed $\hat{d}_j$ and $\hat{y}_j$. Furthermore, the matrix $P$ is sparse, containing at most $n$ non-zeros (as when updating $P_k^\star$ for problem (\ref{pro define FC}), see Appendix~\ref{supp: cutting plane}). 
Moreover, problem (\ref{pro equivalent minimization}) can be separated into the following $m$ subproblems,
\begin{equation}\label{pro equivalent minimization 2}
    \min_{\hat{X}_j \in \calX}   \sum_{i\in[n]}c(Z_i,\hat{Z}_j)P_{ij}  \ \text{, \ for } j \in [m] \ .
\end{equation}
Each subproblem computes the weighted centroid of $\{Z_i:i\in[n], P_{ij} > 0 \}$ under the distance function $c$.  Therefore, (\ref{eq minimizer of surrogate function main paper}) is suitable for parallel and distributed computing. Additionally, since the cost matrix $C(\hat{X})$ is a convex function of $\hat{X}$, each subproblem is a convex problem so  gradient-based methods  could converge to global minimizers. Furthermore, under some particular conditions, solving these small subproblems can be computationally cheap:
\begin{enumerate}
    \item If $\calX$ is convex and $c(Z,\hat{Z})\defeq \|Z-\hat{Z}\|_2^2$, then the minimizer  of  (\ref{pro equivalent minimization 2}) is the weighted average $\sum_{i\in[n]}P_{ij}X_i/\sum_{i\in[n]}P_{ij}$.
    \item If $\calX$ is convex and $c(Z,\hat{Z})\defeq\|Z-\hat{Z}\|_1$, then the minimizer  of  (\ref{pro equivalent minimization 2}) requires sorting the costs coordinate-wisely and finding the median.  
    \item If creating new feature vectors is not permitted and  $\calX = \{X_i: i\in[n]\}$, solving (\ref{pro equivalent minimization 2})  requires finding the smallest $\sum_{i\in[n]:P_{ij}\neq 0}c(Z_i,(\hat{d}_j,X,\hat{y}_j))P_{ij} $ for $X$ within the finite set $\calX$. The matrix $P$ is highly sparse so this operation is not computationally expensive. 
\end{enumerate}


\section{Theoretical Guarantees} \label{sec: convergence-guarantees}

In this section we provide theoretical insights on \methodname\ complexity, convergence behavior of Algorithm~\ref{alg MM method} as well as generalizability of \methodname\ performance on unseen test sets.

\subsection{Computational Complexity}

First, we consider the \texttt{FairWASP} variant used in Algorithm~\ref{alg MM method}, line 4. The initialization requires $O(mn)$ flops and uses $O(n|Y||D|)$ space for storing the cost matrix. After that, the per-iteration time and space complexities are both only $O(n|Y||D|)$. Lemma~\ref{main: our complexity-app} analyzes the computational complexity of our adaptation of the \texttt{FairWASP} algorithm when solving problem (\ref{pro define FC}).

\begin{lemma}\label{main: our complexity-app}
    With efficient computation and space management, the cutting plane method 
    has a computational complexity of
    \begin{equation}\label{eq complexity of cutting plane n-not-m}
        \tilde{O}
    \left(nm + |\calD|^2|\calY|^2 n \cdot \log(R/\eps)
    \right) 
    \end{equation}
    flops and $ {O}( n|\calD||\calY| ) $ space. Here $R$ denotes the size of an optimal dual solution of (\ref{pro define FC}), and $\tilde{O}(\cdot)$ absorbs $m$, $n$, $|\calD|$, $|\calY|$ in the logarithm function.
\end{lemma}

Hence, the overall complexity of \methodname\ is $\tilde{O}(mn + |\mathcal{D}|^2|\mathcal{Y}|^2 n \cdot \log(R/\epsilon))$. Note that in practice, both $|\mathcal{D}|$ and $|\mathcal{Y}|$ are very small compared with the coreset size $m$ and dataset size $n$, so the overall complexity is almost as low as $O(mn)$.

\subsection{Convergence Guarantees}

First, we establish that our proposed surrogate function is indeed convex and a valid upper bound. We then show  Algorithm \ref{alg MM method} converges to a first-order stationary point, within finite iterations if the minimizer of problem (\ref{eq minimizer of surrogate function main paper}) is unique. Note that because $g(\hat{X}; \hat{X}^k) = \langle C(\hat{X}), P_k^\star \rangle$, the minimizer is unique whenever the cost matrix $C(\cdot)$ is strongly convex. 


\begin{lemma}\label{lm valid surrogate function}
    The function $g(\hat{X}; \hat{X}^k)$ is convex function of $\hat{X}$ and a valid upper bound, i.e., $g(\hat{X}; \hat{X}^k) \ge F(C(\hat{X}))$. This inequality holds at equality when $\hat{X} = \hat{X}^k$.
\end{lemma}

\begin{theorem}\label{theo: convergence}
    The objective value is monotonically decreasing, i.e., $F(C(\hat{X}^{k+1})) \le F(C(\hat{X}^k))$ for any $k \ge 0$. 
    And once the algorithm stops and $C(\hat{X}^k)$ is smooth at $\hat{X}^k$, then $\hat{X}^k$ is a first-order stationary point of (\ref{pro simplified target problem 2}). 
\end{theorem}



\begin{theorem}\label{theo: finite termination}
    When the minimizer of (\ref{eq minimizer of surrogate function main paper}) is unique, Algorithm \ref{alg MM method} terminates within finite iterations.
\end{theorem}

\subsection{Generalization Guarantees}

Proposition~\ref{prop: generalization} below bounds the distance and demographic parity between the \methodname\ samples and the true underlying distribution of the data, from which the original dataset of size $n$ was observed. This generalizes the performance of \methodname\ to unseen test sets sampled from the data generating distribution.


\begin{proposition}\label{prop: generalization}
Let $\lambda$ indicate the distance between $p_{\hat{Z}; \theta}$ and $p_{Z; e}$ after convergence of \methodname\, i.e., $\mathcal{W}(p_{\hat{Z}; \theta}, p_{Z; e}) = \lambda$.
Let $q_Z$ be the true underlying distribution of the data supported over $\mathbb{R}^d$, with marginal distribution over $y$ bounded away from zero, so that $\rho = \min_{y \in \mathcal{Y}} q_Z(y) > 0$. Then with probability $1 - \alpha$:
\begin{align}
    \mathcal{W}_c(p_{\hat{Z}}, q_Z) &\leq \lambda + \mathcal{O}(\log(1/\alpha)^{1/d} n^{-1/d}) \\
    \sup_{y \in \mathcal{Y}, d \in \mathcal{D}} J(p_{\hat{Z}} (y|d) , q_Y(y)) &\leq \frac{\epsilon}{\rho} + \mathcal{O}\left( \sqrt{\frac{\log 2/\alpha }{n \rho^2 }} \right)
\end{align}

\end{proposition}




In addition, in Appendix~\ref{app: hardness-learning} we consider the task of learning using \methodname\ samples. We show that the error in downstream learning tasks can be seen as the sum of (i) the approximation error \methodname\ samples make with respect to the original dataset and (ii) how well $\hat{Y}$ can be learnt from $\hat{X}$ and $\hat{D}$ from \methodname\ samples. However, as \methodname\ samples $\{\hat{Z}_j \}_{j=1}^m$ are not i.i.d., standard sample complexity results in e.g., empirical risk minimization, do not apply, highlighting the hardness in developing finite-sample learning bounds in this setting.



\section{An Alternative View: Generalized clustering algorithm}\label{sec: lloyd-algo}

When the fairness constraints are absent, problem (\ref{pro define FC}) reduces to:
\begin{equation}\label{pro define FC unfair}
    	\begin{array}{ll}
		\min_{P\in \mathbb{R}^{n\times m}} \ &  \langle C , P\rangle  
		\\
		\text{ s.t.} \ \ \ \ \ 
		& P \mathbf{1}_m = \frac{1}{n}\cdot \mathbf{1}_n, \  P \ge \mathbf{0}_{n\times m}  \ .
	\end{array}
\end{equation}
The minimizer $P^\star$ of (\ref{pro define FC unfair}) can be written in closed form. For each $i\in [n]$, let $C_{ij_i^\star}$ denote a smallest component on the $i$-th row of $C$. Then the components of a minimizer $P^\star$ can be written as $P_{ij}^\star = \frac{1}{n} \cdot \mathbf{I}(j = j_i^\star)$ (where $\mathbf{I}$ is the indicator function).
Hence, without fairness constraints, \methodname\ corresponds to Lloyd's algorithm for clustering. Specifically, Lloyd's algorithm iteratively computes the centroid for each subset in the partition and subsequently re-partitions the input based on the closeness to these centroids \cite{lloyd1982least}; these are the same operations \methodname\ does in optimizing the surrogate function and solving problem (\ref{pro define FC unfair}).
Thus, when $c(x,y)$ is correspondingly defined as $\|x-y\|_1$ or $\|x-y\|_2^2$,  \methodname\ corresponds to Lloyd's algorithm applied to k-medians or k-means problems, except the centroids have fixed values for $\hat{D}$ and $\hat{Y}$ (see Section~\ref{sec: fair-wasserstein-coresets}).\\

\textbf{Comparison with k-means and k-medoids} \methodname\ and Lloyds’ algorithm for k-means or k-median share similar per-iteration complexity, with the main difference in complexity due to solving problem (\ref{pro define FC}). We solve this problem efficiently by utilizing a variant of the \texttt{FairWASP} approach by \cite{Xiong2023FairWASP} (Section~\ref{sec: fairwasp}),  hence avoiding the usual complexity in solving optimal transport problems. As shown in Section~\ref{sec: convergence-guarantees}, the leading term in the runtime complexity is $\mathcal{O}(nm)$, which comes from calculating and storing the cost matrix $C$. This level of complexity is the same as those in k-means and k-medoids. In addition, from our experiments we also see that the per-iteration complexity of \methodname\ is roughly linear with the original dataset size $n$ (see the runtime experiment in Section~\ref{sec: experiments} and Appendix~\ref{supp: exp-details}).

\section{Experiments} \label{sec: experiments}

\paragraph{Runtime analysis} We evaluate the runtime performance of \methodname\ by creating a synthetic dataset of dimension $n$ and features of dimension $p$, with the goal of creating a coreset of size $m$ (see Appendix~\ref{supp: synthetic-dataset} for details). 
We fix two out of the three parameters to default values $(n, m, p) = (5000, 250, 25)$ and vary the other across suitable ranges, to analyse the runtime and total number of iterations. Figure~\ref{fig: main-results}, top left, and Table~\ref{tab: supp-runtime}, in Appendix~\ref{supp: synthetic-dataset}, show the runtime and number of iterations when increasing the dataset size $n$ from 1,000 to 1,000,000, with averages and standard deviations over 10 separate runs; both the runtime and number of iterations grow proportionally to the sample size $n$. Figure~\ref{fig: supp-runtime} in Appendix~\ref{supp: synthetic-dataset} also shows that requiring a larger coreset size $m$ implies the need of fewer iterations but longer iteration runtime, as more representatives need to be computed.



\paragraph{Real datasets results} We evaluate the performance of \methodname\, on 4 datasets widely used in the fairness literature \cite{fabris2022algorithmic}: (i) Adult \cite{misc_adult_2}, (ii) German Credit \cite{misc_statlog_(german_credit_data)_144}, (iii) Communities and Crime \cite{misc_communities_and_crime_183} and (iv) Drug \cite{fehrman2017five}. For each dataset, we consider 3 different coreset sizes $m= 5\%, 10\%, 20\%$ (apart from the Adult dataset, in which we select $m$ equal to 0.5\%, 1\% and 2\% due to the large dataset size). We compare our approach with: (a) \texttt{Fairlets} and \texttt{IndFair}, two fair clustering approaches by \cite{backurs2019scalable} and \cite{chhaya2022coresets}, (b) \texttt{K-Median Coresets}, a coreset approach by \cite{bachem2018one}, (c) k-means \cite{lloyd1982least} and k-medoids \cite{maranzana1963location, park2009simple}, two classic clustering approaches and (d) \texttt{Uniform Subsampling} of the original dataset. For \methodname, we consider three different values of the fairness violation hyper-parameters $\epsilon$ for the optimization problem in (\ref{pro general optimization model}). We compute the fairness-utility tradeoff by first training a 2-layer multilayer perceptron (MLP) classifier with ReLu activations on the coresets created by each approach and then evaluating the classifier demographic disparity (fairness) and AUC (utility). Figure~\ref{fig: main-results} shows the model with the best fairness-utility tradeoff across the three coreset sizes $m$, for each approach. \methodname\, obtains equal or better fairness-utility tradeoffs (smaller disparity at the same level of utility, higher utility with the same disparity, or both) across all datasets, and performance remains competitive even when using a fairness pre-processing approach \cite{kamiran2012data}. Appendix~\ref{supp: exp-real-dataset-subsec} includes more experiments and details, which highlight that: (a) \methodname\, consistently achieves coresets that are closer in distribution to the original dataset with respect to the other methods and, although not natively minimizing clustering cost, also provide competitive performance for smaller datasets (Tables~\ref{tab: supp-wass_cost_numbers} and \ref{tab: supp-clust_cost_numbers}); (b) when added to the training data using the data augmentation schema proposed by \cite[Section 2.1]{sharma2020data} \methodname\ generally either increase the performance or reduce the demographic disparity in the downstream learning process.

\begin{figure*}[!ht]
    \centering
    \includegraphics[width=0.325\linewidth]{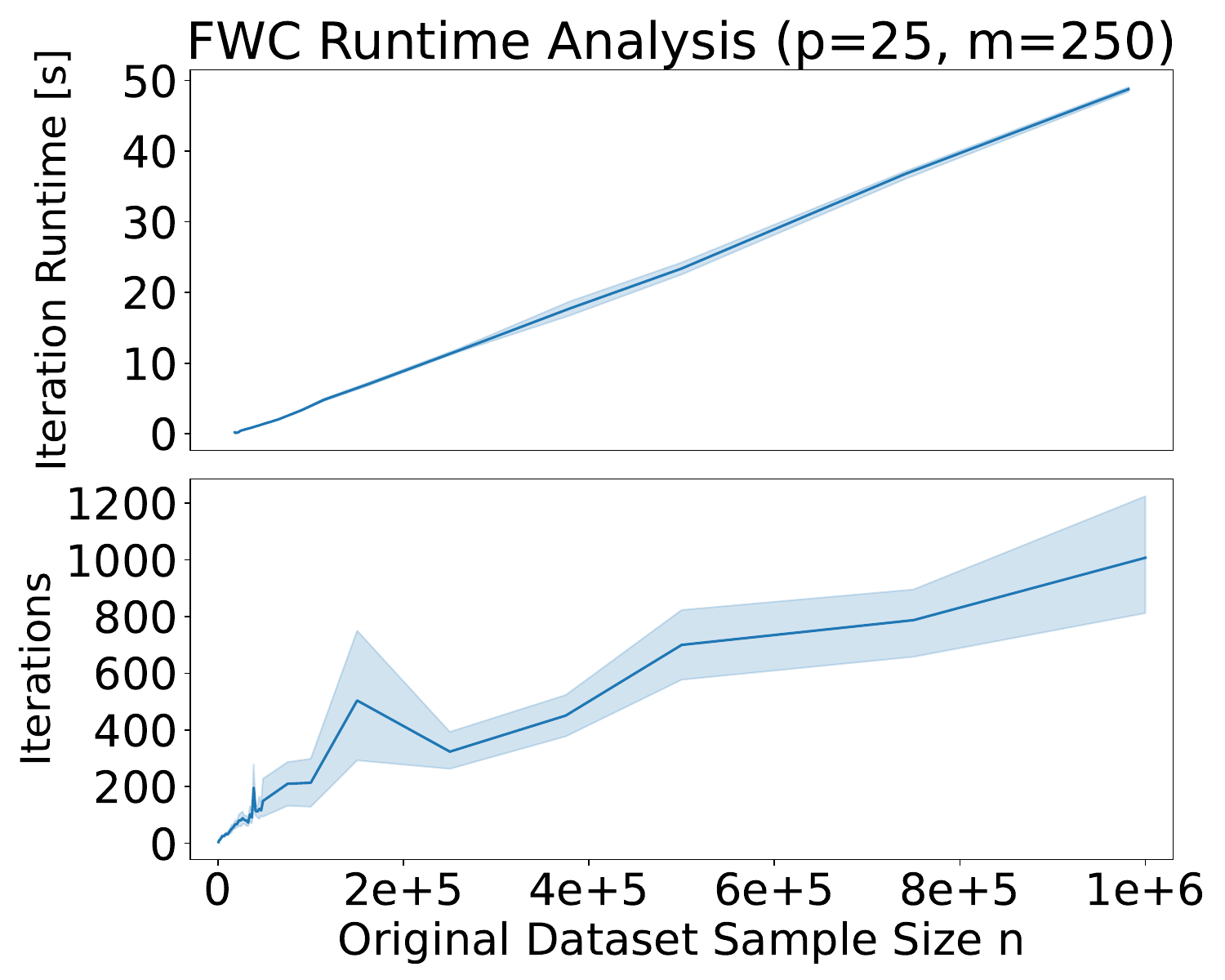}
    \includegraphics[width=0.325\linewidth]{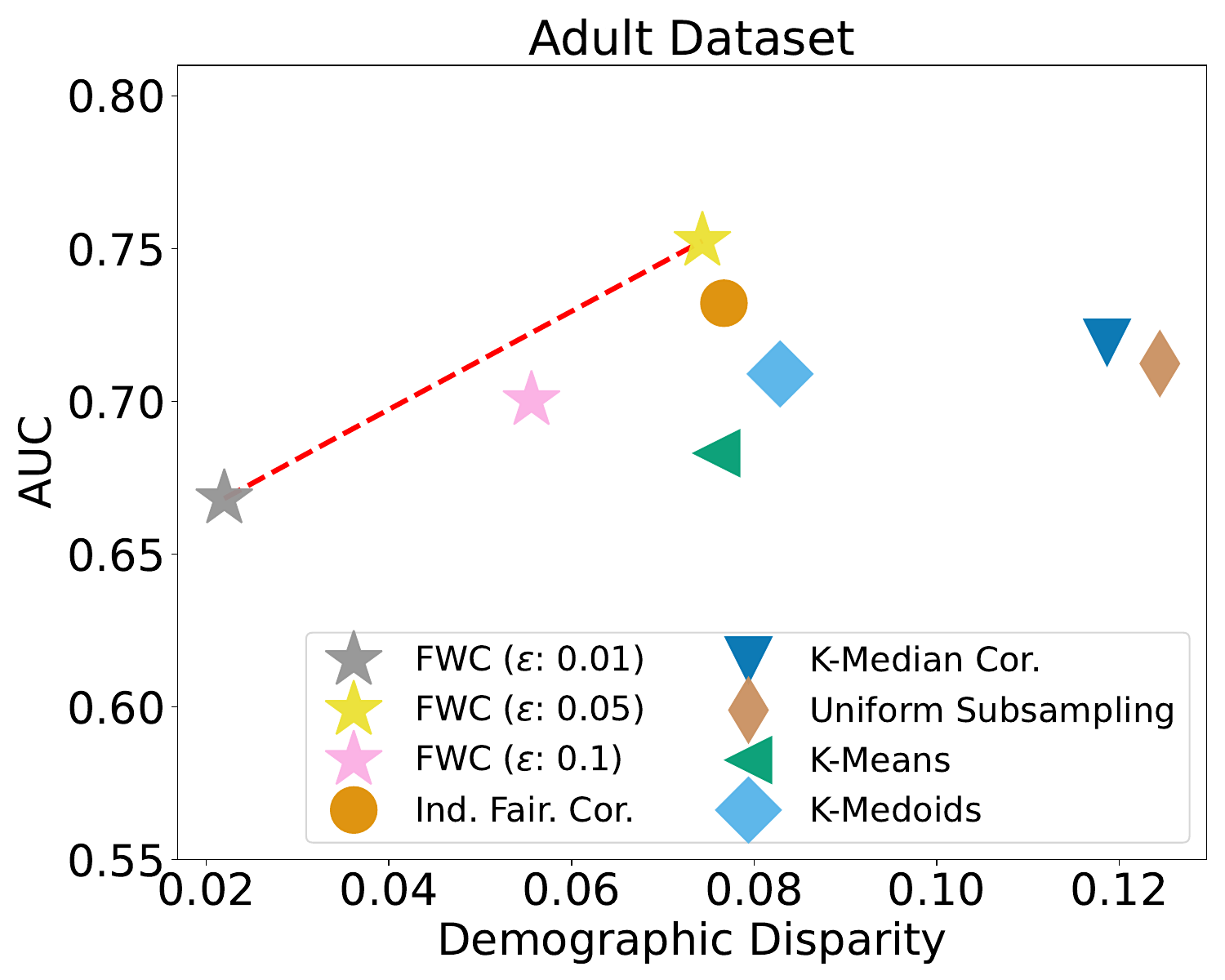}
    \includegraphics[width=0.325\linewidth]{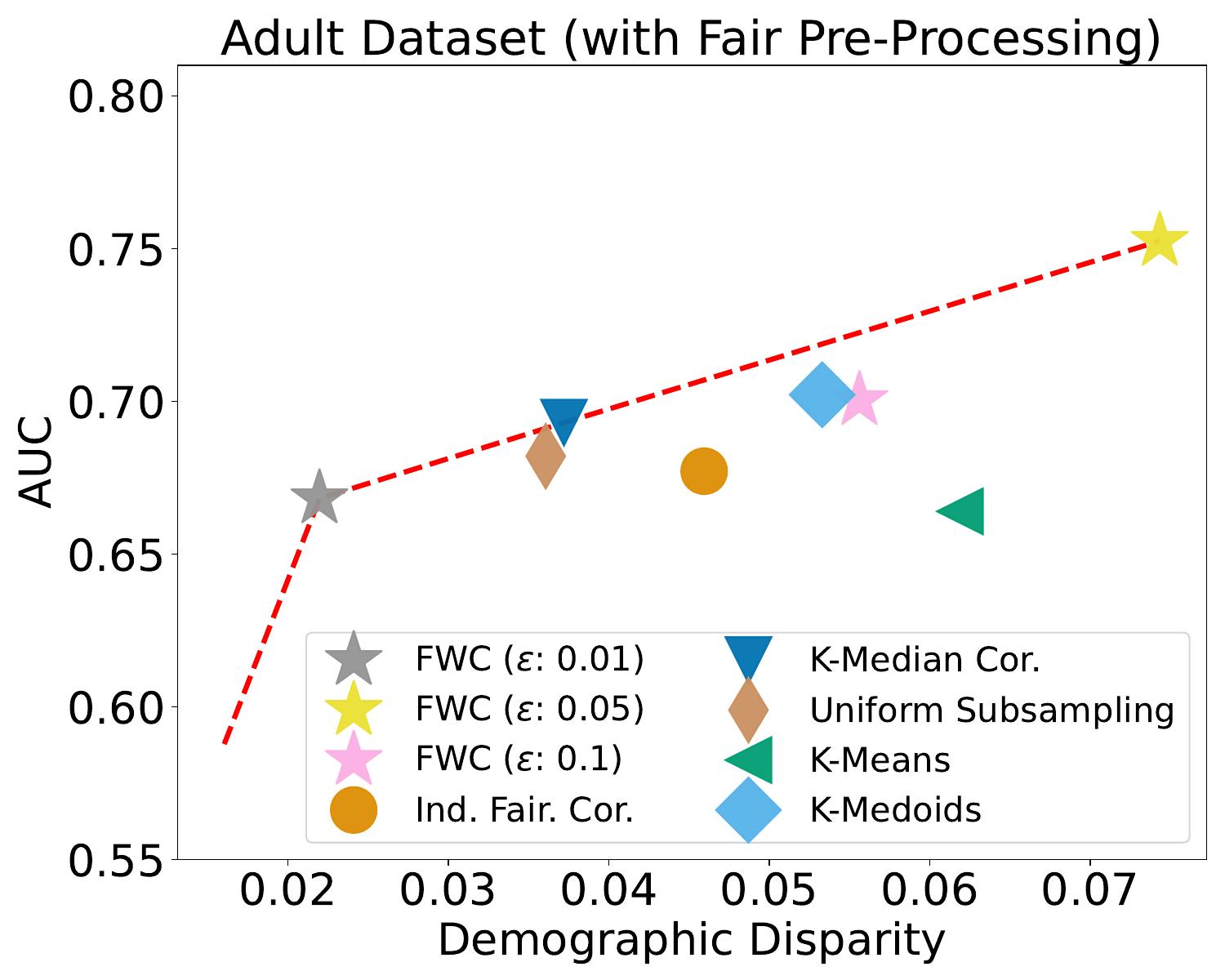} \\
    \includegraphics[width=0.325\linewidth]{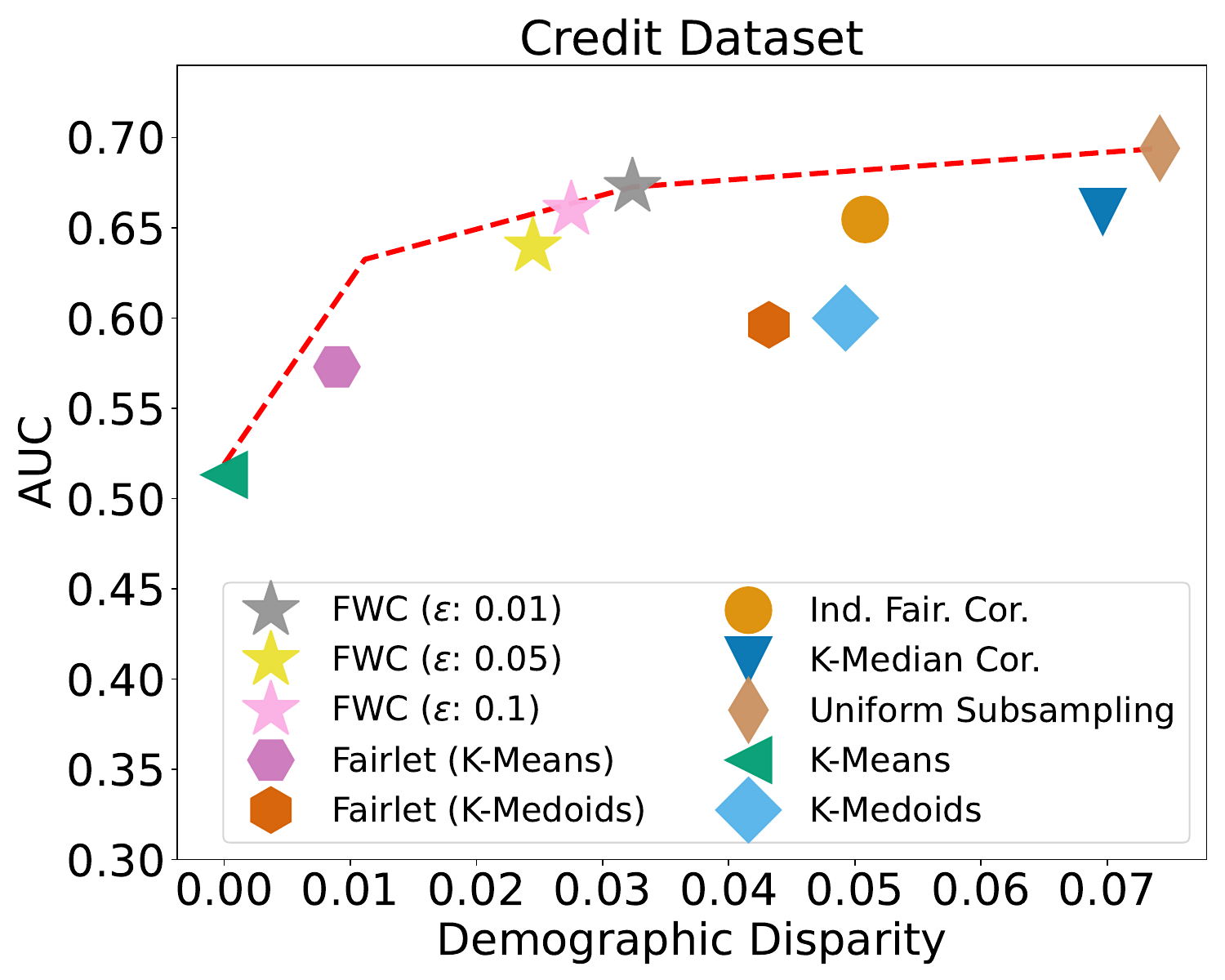}
    \includegraphics[width=0.325\linewidth]{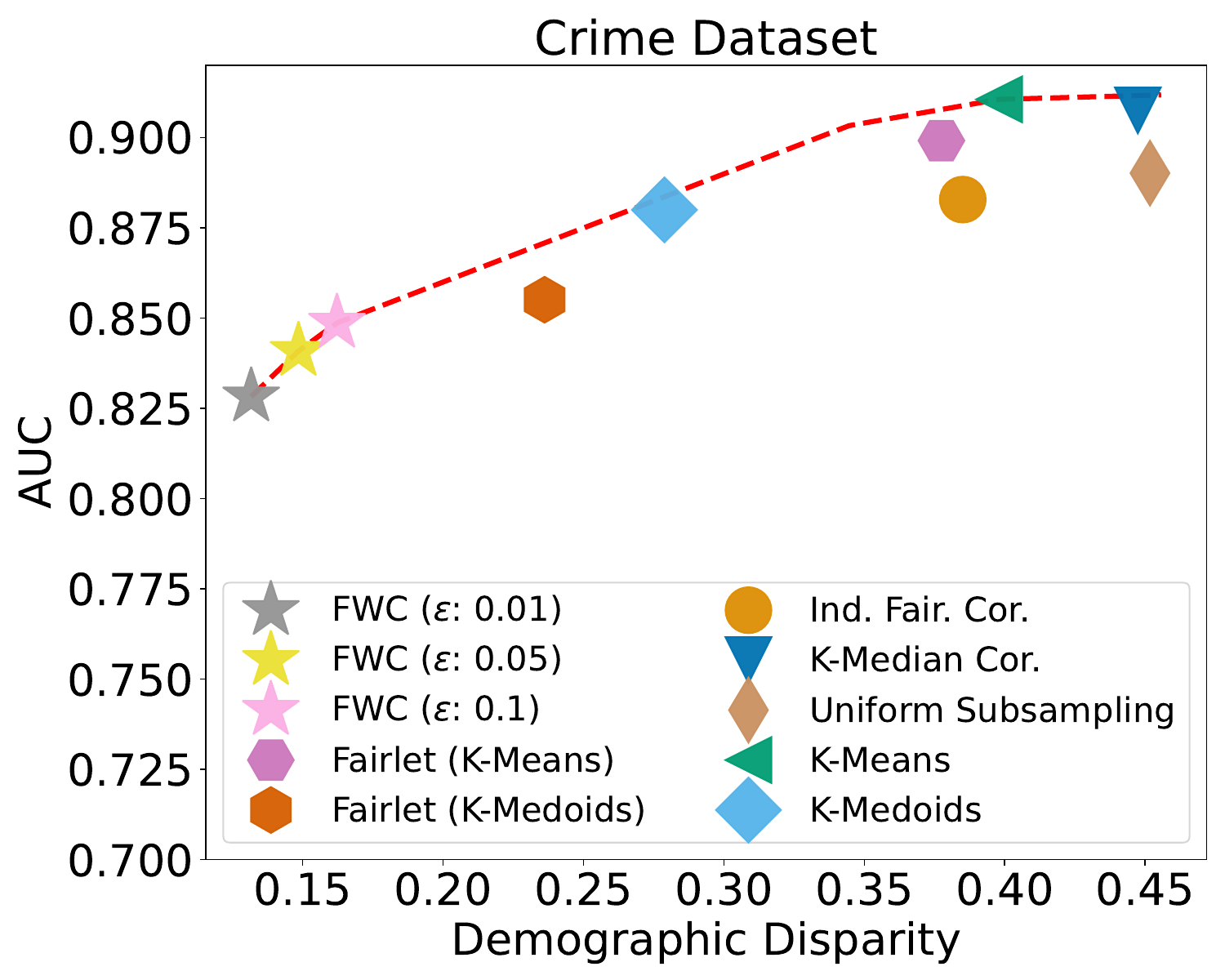}
    \includegraphics[width=0.325\linewidth]{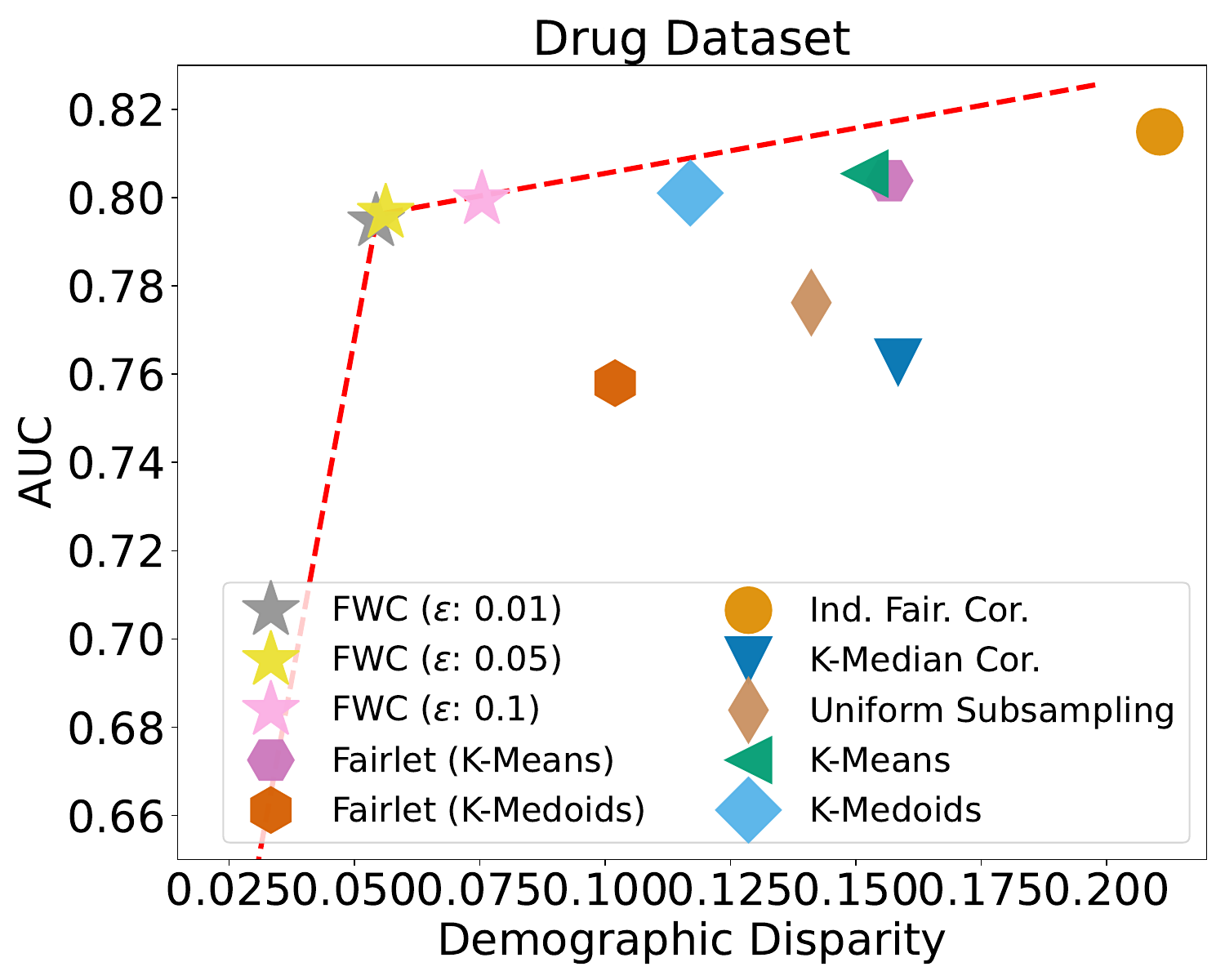}
    \caption{\textit{Top left:} \methodname\ runtime when changing the original dataset size $n$. \textit{Others}: Fairness-utility tradeoff on real datasets for a downstream MLP classifier, selecting the model with the best fairness-utility tradeoff across three different coreset sizes $m$, with averages taken over 10 runs. \methodname\, consistently achieves a comparable/better tradeoff as shown by the Pareto frontier (dashed red line, computed over all models and coreset sizes), even when adjusting the other coresets with a fairness pre-processing technique [33]. See text and Appendix C.2 for more details.}
    \label{fig: main-results}
\end{figure*}

\paragraph{Using \methodname \ to improve fairness for LLM} \cite{wang2023decodingtrust} evaluate GPT-3.5 and GPT-4 for fairness on predictive tasks for the UCI Adult dataset in a zero and few shot setting. We use a similar evaluation setup and use \methodname\ in the few shot setting as examples and evaluate the results for the gender protected attribute. Specifically, we transform the tabular data into language descriptions, and ask GPT-3.5 Turbo and GPT-4 to perform classification tasks on it.
We select $200$ samples to construct the test set and use a set of 16 samples found using \methodname\ as examples. Further details on this experiment are provided in the appendix. The results are shown in Table~\ref{tab:gpt}. Examples provided by ~\methodname\ help reduce demographic disparity more than providing balanced few shot examples, while losing on predictive accuracy (note that the drop in accuracy is similar to the drop observed in \cite{wang2023decodingtrust} and is representative of the fairness-utility trade-off). Owing to the token limitation of LLM's, these representative coresets can evidently be valuable to provide a small set of samples that can help mitigate bias. When accounting for the standard deviations for demographic parity in Table~\ref{tab:gpt}, \methodname\ reduces the LLM bias when compared to zero shot prompting for GPT-4 across all runs. For GPT-3.5 Turbo, while the average disparity is reduced across runs, such consistency is indeed not observed, owing to a diverse set of outputs from the large language model. Due to limited availability of computational resources (associated with querying these models), we leave a more thorough evaluation across different datasets and models to future work.

\begin{table*}[!thbp]
\centering
 \caption{Using the same setup as in \cite{wang2023decodingtrust}, we use GPT-3.5 Turbo and GPT-4 LLM's for fairness evaluations, with a test set of 200 samples with 0.5 base parity (${b_{p}}=0.5$). 
 Few Shot - \methodname\ is used to provide sixteen examples with weights to the model as examples. Accuracy and demographic disparity (DP) are based on the resulting predictions from GPT-3.5 and GPT-4 models.}
 \vspace{0.2cm}
\resizebox{1.0\textwidth}{!}{%
  \begin{tabular}{|c||c|c|c|c|c|c|}
    \hline
    \multirow{2}{*}{\textit{Adult Dataset}} &
      \multicolumn{2}{c|}{\textbf{Zero Shot}} &
      \multicolumn{2}{c|}{\textbf{Few Shot} ($b_{p} =0$) } &
      \multicolumn{2}{c|}{\textbf{Few Shot} (\methodname )}  \\ \cline{2-7}
      & {Accuracy} & {DP} & {Accuracy} & {DP} & {Accuracy} & {DP} \\ \hline
GPT-3.5 Turbo & 53.55 $\pm$ 0.87 & 0.040 $\pm$ 0.017 & 57.99 $\pm$ 1.88 & 0.019 $\pm$ 0.015 & 55.02 $\pm$ 1.26 & \textbf{0.010 $\pm$ 0.03} \\ \hline
GPT-4 & 76.54 $\pm$ 1.63 & 0.42 $\pm$ 0.016 & 74.39 $\pm$ 2.86 & 0.33 $\pm$ 0.09 & 65.20 $\pm$ 0.85 & \textbf{0.27 $\pm$ 0.04} \\
    \hline
  \end{tabular}
}
 \label{tab:gpt}
\end{table*}

\section{Discussion and Conclusions} \label{sec: discussion-conclusions}
We introduce \methodname, a novel coreset approach that generates synthetic representative samples along with sample-level weights for downstream learning tasks. \methodname\ minimizes the Wasserstein distance between the distribution of the original datasets and that of the weighted synthetic samples while enforcing demographic parity. We demonstrate the effectiveness and scalability of \methodname\ through  experiments conducted on both synthetic and real datasets, as well as reducing biases in LLM predictions (GPT 3.5 and GPT 4).
Future extensions include: (i) targeting different fairness metrics such as equalized odds \cite{hardt2016equality, mishler2021fairness} as well as robustness of fairness-performance tradeoff over distribution shifts \cite{mishler2022fair, sharma2023feamoe}, (ii) exploring privacy and explainability properties of \methodname\ \cite{mohassel2020practical, moshkovitz2020explainableclustering}, (iii) utilizing coresets for accelerating gradient descents algorithms and test their convergence \cite{mirzasoleiman20a, sordello2024robust} 
(iv) reformulating the optimization framework to target deep neural network pruning \cite{Mussay2020DataIndependent, ohib2022explicit} and (v) investigate applications of fair synthetic data in the financial sector \cite{potluru2023synthetic}.

\section*{Disclaimer}
This paper was prepared for informational purposes by the Artificial Intelligence Research group of JPMorgan Chase \& Co. and its affiliates ("JP Morgan'') and is not a product of the Research Department of JP Morgan. JP Morgan makes no representation and warranty whatsoever and disclaims all liability, for the completeness, accuracy or reliability of the information contained herein. This document is not intended as investment research or investment advice, or a recommendation, offer or solicitation for the purchase or sale of any security, financial instrument, financial product or service, or to be used in any way for evaluating the merits of participating in any transaction, and shall not constitute a solicitation under any jurisdiction or to any person, if such solicitation under such jurisdiction or to such person would be unlawful.

\bibliographystyle{abbrv}

\newpage

\appendix

\begin{center}
\begin{Large}
\centering
\textbf{Appendix}
\end{Large}
\end{center}

\section{Details on Updating the Surrogate Function (line 4 of Algorithm \ref{alg MM method})} \label{supp: cutting plane}
To update the surrogate function, we need to solve problem (\ref{pro define FC}), which is a huge-scale linear program with $O(n)$ constraints and $O(mn)$ nonnegative variables. In this work, we adapt \texttt{FairWASP} \cite{Xiong2023FairWASP} for cases where $m \neq n$, as opposed to the scenario where $m=n$, which was tackled by \cite{Xiong2023FairWASP}. Before showing the main idea of the algorithm, we rephrase a useful lemma for doing linear minimization on $S_{n,m}\defeq \{P\in \mathbb{R}^{n\times m} : P \textbf{1}_m = \frac{1}{n}\cdot \mathbf{1}_n, P \ge \mathbf{0}\}$. 
\begin{lemma}\label{lm minimizer P}
    For the function $G(C) \defeq \max_{P\in S_{n,m}} \langle C, P \rangle$, it is a convex function of $C$ in $\mathbb{R}^{n\times m}$. For each $i\in[n]$, let $C_{ij_i^\star}$ denote a largest component on the $i$-th row of $C$, then $G(C) = \frac{1}{n} \sum_{i=1}^n C_{ij_i^\star}$. Define the components of $P^\star$ as follows: 
    \begin{equation}\label{eq lm minimizer P}
            P_{ij}^\star = \left\{
    \begin{array}{cc}
         0 & \text{ if } j \neq j_i^\star \\
         \frac{1}{n} & \text{ if } j = j_i^\star
    \end{array}
    \right. 
    \end{equation}
    and then $P^\star \in \arg\max_{P\in S_{n,m}} \langle C, P \rangle$ and $P^\star \in \partial G(C)$.
\end{lemma}
\begin{proof}
    The proof of the above lemma is equivalent with that of Lemma 1 of  \cite{Xiong2023FairWASP} in the case when $m \neq n$, which can be extended directly.
\end{proof}

 With this lemma, now we show how to efficiently solve problem (\ref{pro define FC}) via its dual problem.
Although (\ref{pro define FC}) is of large scale and computationally prohibitive, it is equivalent to the following saddle point problem on the Lagrangian:
    \begin{equation}\label{pro lagragian saddle point}
        \min_{P\in S_{n,m}}\max_{\lambda \in \mathbb{R}^h_+} L(P,\lambda)\defeq \langle C, P \rangle - \lambda^\top A P^\top \mathbf{1}_n
    \end{equation}
where $h$ is the number of rows of $A$, which is at most $2 |\calY| |\calD|$. It should be mentioned that $h$ is significantly smaller than $mn$; for example, for classification tasks with only two protected variables, $h$ is no larger than $8$, independent of the number of samples or features. Since $L(\cdot,\cdot)$ is bilinear, the minimax theorem guarantees that (\ref{pro lagragian saddle point}) is equivalent to $\max_{\lambda \in \mathbb{R}^h_+} \min_{P\in S_{n,m}} L(P,\lambda)$. This is further equal to the dual problem: 
    \begin{equation}\label{pro low dim dual}
        \max_{\lambda \in \mathbb{R}^h_+} - \Big[ G(\lambda)\defeq \max_{P\in S_{n,m}} \Big\langle \sum_{j=1}^h \lambda_j \mathbf{1}_n a_j^\top - C, P \Big\rangle \Big] \ ,
    \end{equation}
    in which $a_j^\top$  denotes the $j$-th row of $A$. 
    Note that the problem (\ref{pro low dim dual}) has much fewer decision variables than that of (\ref{pro lagragian saddle point}) and Lemma \ref{lm minimizer P} ensures the function $G(\cdot)$ is convex and has easily accessible function values and subgradients. Therefore, directly applying a cutting plane method has low per-iteration complexity and solves the problem (\ref{pro low dim dual}) in linear time. We include the details on the cutting plane method in Section~\ref{supp:subsection-cutting-plane} below.
    Finally, the primal optimal solution of (\ref{pro define FC})  can be easily recovered from the dual optimal solution $\lambda^\star$ via solving  $\max_{P\in S_{n,m}} \big\langle \sum_{j=1}^h \lambda_j^\star \mathbf{1}_n a_j^\top - C, P \big\rangle$, under the assumption that this problem has a unique minimizer, which almost always holds in practice for the computed $\lambda^\star$ and is also assumed by \cite{Xiong2023FairWASP}.  In this way, we have shown how the problem (\ref{pro define FC}) can be efficiently solved by applying a cutting plane method on its dual problem.


\subsection{Details of the Cutting Plane Method for Solving (\ref{pro low dim dual})}\label{supp:subsection-cutting-plane}

The cutting plane method is designed for convex problems where a \textit{separation oracle} can be employed \cite{Khachiyan1980polynomial}. For any $\lambda \in \mathbb{R}^m$, a separation oracle operates by generating a vector $g$ which satisfies $g^\top \lambda \ge g^\top \lambda^\star$ for all $\lambda^\star$ in the set of optimal solutions. By repeatedly applying the separation oracle to cut down the potential possible feasible set, the cutting plane method progressively narrows down the feasible solution space until it reaches convergence. The specific steps of the cutting plane algorithm are detailed in Algorithm~\ref{alg:cutting plane method}, with the key distinctions among different versions of this method lying in how lines 3 and 4 are implemented.

\begin{algorithm}[htbp]
\caption{General Cutting Plane Method for (\ref{pro low dim dual})}
\label{alg:cutting plane method}
\begin{algorithmic}[1]
\STATE Choose a bounded set $E_0$ that contains an optimal solution
\FOR{$k$ from $0$ to $n$}
    \STATE Choose an interior point $\lambda^k$ of $E_k$; 
    \STATE Compute $g \in \mathbb{R}^m$ such that 
    $$
    g^\top \lambda^k \ge g^\top \lambda^\star \text{ for any optimal solution }\lambda^\star; 
    $$
    \STATE Choose the next bounded set $E_{k+1} \supseteq \{\lambda\in E_k: g^\top \lambda \le g^\top \lambda^k\}$;
\ENDFOR
\end{algorithmic}
\end{algorithm}

For the problem~(\ref{pro low dim dual}), a separation oracle (line 4 in Algorithm \ref{alg:cutting plane method}) can directly use the vector of subgradients, which are efficiently accessible, as we mentioned in section \ref{sec: majority-minimization}. 
Given that we have shown (\ref{pro low dim dual}) is a low-dimensional convex  program with subgradient oracles, there exist many well-established algorithms that can be used.
Suppose that the norm of an optimal $\lambda^\star$ is bounded by $R$, to the best of our knowledge, the cutting plane method with the best theoretical complexity is given by \cite{jiang2020improved}, who proposed an improved cutting plane method that only needs $O((h\cdot \operatorname{SO} + h^2 )\cdot \log(hR/\eps))$ flops. Here SO denotes the complexity of the separation oracle. Note that here $h$ is at most $2|\calD||\calY|$, which is far smaller than $n$ or $m$. 
Here we restate the Corollary 5 of \cite{Xiong2023FairWASP} below for the overall time and space complexity of applying the cutting plane method in the case $m \neq n$. 
\begin{lemma}[Essentially Corollary 5 of \cite{Xiong2023FairWASP}]\label{cor our complexity-app}
    With efficient computation and space management, the cutting plane method solves the problem (\ref{pro low dim dual}) within 
    \begin{equation}\label{eq complexity of cutting plane n=m}
        \tilde{O}\left(
    mn + |\calD|^2|\calY|^2n  \cdot \log(\tfrac{R}{\eps})
    \right) 
    \end{equation}
    flops and $ {O}( n|\calD||\calY| ) $ space. Here we use $\tilde{O}(\cdot)$ to hide $m$, $n$, $|\calD|$, and $|\calY|$ in the logarithm function.
\end{lemma}
The above result is essentially Corollary 5 of \cite{Xiong2023FairWASP} by slightly extending the proof to the general case $m \neq n$.  
Finally, in terms of the implementation, we follow \cite{Xiong2023FairWASP} and use the analytic center cutting plane method.

\section{Theoretical Proofs} \label{supp: theory-proofs}

This section includes the theoretical proofs for Section \ref{sec: background} and Section~\ref{sec: convergence-guarantees}. We first show the Wasserstein distance upper bounds downstream disparity for MLP networks (Proposition~\ref{prop: wass-neuralnet}). We then show (i) the optimization problem can be reduced to optimizing over the $\hat{X}$ rather than $\hat{Z}$ (Lemma~\ref{lem: optimization-redux}), (ii) the surrogate function is convex and a valid upper bound of the optimization objective (Lemma~\ref{lm valid surrogate function}), (iii) our proposed algorithm converges to a first-order stationary point in $\hat{X}$ (Theorem~\ref{theo: convergence}), and (iv) our proposed algorithm terminates in a finite amount of iterations (Theorem~\ref{theo: finite termination}). We also prove the generalization bound for \methodname\ performance in terms of Wasserstein distance and demographic parity to unseen datasets coming from the same (unknown) distribution the original dataset was sampled from. Finally, we include Section~\ref{app: hardness-learning} to note how the downstream learning using \methodname\ can be broken down into two terms, which highlights the challenges in analyzing its theoretical properties.

\begin{lemma}\label{lemma: wass-dist-ub}
    Let $Z=(X,Y), \hat{Z}=(\hat{X}, \hat{Y}) \in \calZ = (\calX \times \calY)$ be two pairs of random variables with joint distributions $p_Z$ and $p_{\hat{Z}}$ and marginal distributions $p_X$, $p_Y$ and $p_{\hat{X}}$ and $p_{\hat{Y}}$ respectively. 
    Let $\Pi(Z, \hat{Z})$ indicate the set of all joint probability distributions over the product space $\calZ \times \calZ$ that admit marginal and conditional distributions over $\calX$ and $\calY$. For any non-negative cost operator $c$:
    $$\mathcal{W}_c(p_X, p_{\hat{X}}) \leq \mathcal{W}_c(p_Z, p_{\hat{Z}}).$$
\end{lemma}

\begin{proof}{Proof of Lemma~\ref{lemma: wass-dist-ub}} Let $\Pi_X(X, \hat{X})$ indicate the set of all marginal (joint) probability distributions over the product space $\calX \times \calX$. For any $\pi \in \Pi(Z, \hat{Z})$ and the corresponding $\pi_x \in \Pi_X(X, \hat{X})$ we have that:

\begin{align}\label{eq:first-passage-lemma}
    \int_{\calZ \times \calZ} c(z, \hat{z}) d\pi(z, \hat{z}) &\geq \int_{\calZ \times \calZ} c(x, \hat{x}) d\pi(z, \hat{z}) \\
    &= \int_{\calX \times \calX} c(x, \hat{x}) d\pi_x(x, \hat{x}) \notag \\
    &\geq \min_{\pi_x \in \Pi_x(x, \hat{x})} c(x, \hat{x}) d\pi_x(x, \hat{x}) = \mathcal{W}_c(p_X, p_{\hat{X}}) \notag
\end{align}

As this is valid for any $\pi \in \Pi(Z, \hat{Z})$, select:

$$\pi^\star = \argmin_{\pi \in \Pi(Z, \hat{Z})} \int_{\calZ \times \calZ} c(z, \hat{z}) d\pi(z, \hat{z}).$$

Then the left-hand-side of (\ref{eq:first-passage-lemma}) is equal to $\mathcal{W}_c(p_Z, p_{\hat{Z}})$, hence proving that $ \mathcal{W}_c(p_Z, p_{\hat{Z}}) \geq \mathcal{W}_c(p_X, p_{\hat{X}})$.

\end{proof}

\begin{proof}{Proof of Proposition~\ref{prop: wass-neuralnet}}
The proof of the upper bound follows from the first part of the proof of the Kantorovich-Rubinstein duality \cite{santambrogio2015optimal}. In this work we follow the proof by \cite{basso2015hitchhiker, thickstun2019kantorovich}, which consider the Lagrangian form of the 1-Wasserstein distance and express it in the following form:

$$
\mathcal{W}(p_{(X, D)}, p_{(\hat{X}, \hat{D})}) = \sup_{f,g: f(x) + g(y) \leq \|x - y\|_2} \left|\mathbb{E}_{(x, d) \sim p_{(X, D)}} f_\theta(x, d)- \mathbb{E}_{(x, d)\sim p_{(\hat{X}, \hat{D})}} f_{\theta}(x, d) \right|,
$$

where $f,g:\mathcal{X} \to \mathbb{R}$ are bounded, measurable functional Lagrangian multipliers. Let $L_{f_\theta}$ be the Lipschitz constant of the MLP $f_{\theta}$, and define the following function:

$$
h(x, d) = \frac{f_\theta(x, d)}{L_{f_\theta}}.
$$

By definition, $h(x, d)$ is 1-Lipschitz. Again following \cite{basso2015hitchhiker, thickstun2019kantorovich}, we know that for 1-Lipschitz functions the following holds:

\begin{align*}
    &\left|\mathbb{E}_{(x, d)\sim p_{(X, D)}} f_\theta(x, d) - \mathbb{E}_{(x, d)\sim p_{(\hat{X}, \hat{D})}} f_{\theta}(x, d) \right| = L_{f_\theta} \left|\mathbb{E}_{(x,d) \sim p_{(X, D)}} h(x, d)- \mathbb{E}_{(x,d) \sim p_{(\hat{X}, \hat{D})}} h(x, d) \right| \\
    &\quad = L_{f_\theta} \left| \int_{(\mathcal{X} \times \mathcal{D}) \times (\mathcal{X} \times \mathcal{D})} h(x_1, d_1, x_2, d_2) d\pi \left(p_{(X, D)}, p_{(\hat{X}, \hat{D})} \right) \right| \\
    &\quad \leq L_{f_\theta} \int_{(\mathcal{X} \times \mathcal{D}) \times (\mathcal{X} \times \mathcal{D})} \left| h(x_1, d_1, x_2, d_2) \right| d\pi \left(p_{(X, D)}, p_{(\hat{X}, \hat{D})} \right) \\
    &\quad \leq L_{f_\theta} \int_{(\mathcal{X} \times \mathcal{D}) \times (\mathcal{X} \times \mathcal{D})} \left\| (x_1, d_1) - (x_2, d_2) \right\|_2 d\pi \left(p_{(X, D)}, p_{(\hat{X}, \hat{D})} \right) \\
    &\quad \leq L_{f_\theta} \mathcal{W}_1(p_{(X,D)}, p_{(\hat{X},\hat{D})}) \leq L_{f_\theta} \mathcal{W}_1(p_Z, p_{\hat{Z}}),\\
\end{align*}

where the last inequality is due to Lemma~\ref{lemma: wass-dist-ub}. The result can be obtained by using the upper bound in 
\cite[Section 6.1]{virmaux2018lipschitz}, which shows that $L_{f_\theta} \leq L_k$ for $K$-layer MLPs with ReLu activations.\\
\end{proof}

\begin{proof}[Proof of Lemma~\ref{lem: optimization-redux}]
    Once we generate $m|\calD||\calY|$ data points, the feasible set of the latter Wasserstein coreset contains the feasible set of the former Wasserstein coreset.
\end{proof}

\begin{proof}[Proof of Lemma~\ref{lm valid surrogate function}]
The convexity follows directly from the convexity of $C(\hat{X})$, as the $P_k^\star \ge 0$ in (\ref{eq surrogate function}).

Before proving it is an upper bound, we show some important properties of $F(C)$ as a function of $C$.
Firstly, $F(C)$ is concave on $C$ because of the concavity  of the minimum LP's optimal objective on the objective vector. Secondly, since the feasible set of problem (\ref{pro define FC}) is bounded, the optimal solution $F(C)$ is continuous with respect to $C$. 
Thirdly,  due to the sensitivity analysis of LP \cite{bertsimas1997introduction}, a supergradient of $F(C)$ at point $C$ is the corresponding optimal solution $P^\star$. Here the definition of supergradients for concave functions is analogous to the definition of subgradients for convex functions.

Now we prove $g(\hat{X}; \hat{X}^k)$  is an upper bound of $F(C(\hat{X}))$.
Because $P_k^k$ is a supergradient of $F(C)$ when $C =C(\hat{X}^k)$,
$$
F(C) + \langle P_k^k, C(\hat{X}) - C \rangle \ge F(C(\hat{X})) \ ,
$$
in which the left-hand side is equal to $g(\hat{X}; \hat{X}^k)$ because $F(C) = \langle P_k^k,  C \rangle$ and $g(\hat{X}; \hat{X}^k) = \langle C(\hat{X}), P_k^\star\rangle$. Therefore, the surrogate function is an upper bound of the objective function $F(C(\hat{X}))$, i.e., $g(\hat{X}; \hat{X}^k) \ge F(C(\hat{X}))$. 
Moreover, due to the definition in (\ref{eq surrogate function}), $g(\hat{X}; \hat{X}^k) = F(C(\hat{X}))$ when $\hat{X} = \hat{X}^k$. 
\end{proof}

\begin{proof}[Proof of Theorem~\ref{theo: convergence}]
    The monotonically decreasing part of the claim follows by:
    \begin{equation}\label{supp: eq decreasing}
    \begin{aligned}
         F(C(\hat{X}^{k+1})) \le   g(\hat{X}^{k+1};\hat{X}^k)   = \arg\min_{\hat{X} \in \calX^m}g(\hat{X}; \hat{X}^k)
        \le  g(\hat{X}^{k};\hat{X}^k)
    = F(C(\hat{X}^k)) \ .
    \end{aligned}
    \end{equation}
    Here the first inequality is due to the fact that $g(\hat{X};\hat{X}^k) \ge F(C(\hat{X}))$ for any $\hat{X}$. The final equality is because $g(\hat{X}; \hat{X}^k) = F(C(\hat{X}))$ when $\hat{X} = \hat{X}^k$.  Once  $g(\hat{X}^{k}; \hat{X}^k) = g(\hat{X}^{k+1}; \hat{X}^k)$ and  thus $\hat{X}^{k} \in \arg\min_{\hat{X}\in\calX^m}g(\hat{X}; \hat{X}^k)$, then $\hat{X}^k$ is a global minimizer of the convex upper bound $g(\cdot;\hat{X}^{k})$ for $F(C(\cdot))$ and the upper bound $g(\hat{X}^k;\hat{X}^{k})$ attains the same function value with $ F(C(\hat{X}^k))$. Therefore, if the surrogate function is smooth at $\hat{X}^k$, which could be achieved if $C(\hat{X}^k)$ is smooth at $\hat{X}^k$, then $X^k$ is a first-order stationary point of (\ref{pro simplified target problem 2}).
\end{proof}


\begin{proof}[Proof of Theorem~\ref{theo: finite termination}]
    Because (\ref{eq minimizer of surrogate function main paper}) has a unique minimizer, the second inequality in (\ref{supp: eq decreasing})  holds strictly when $\hat{X}^{k+1} \neq \hat{X}^k$, or equivalently $g(\hat{X}^{k+1}; \hat{X}^k) \neq  g(\hat{X}^k; \hat{X}^k)$. Once $\hat{X}^{k+1} = \hat{X}^k$, then the algorithm terminates. Note that there are only finite possible optimal basic feasible solution $P^\star$ that could be generated by \texttt{FairWASP}, as shown in  Lemma \ref{lm minimizer P}. However, before the majority minimization converges,  (\ref{supp: eq decreasing}) holds strictly and the corresponding $P_k^\star$ keeps changing. Therefore, after finite iterations, there must be a $P_t^\star$ equal to a previous $P_j^\star$ for $j < t$.  When that happens, because the surrogate functions are the same and thus have the same minimizer, $\hat{X}^{t+1} = \hat{X}^{j+1}$, and 
    the inequalities (\ref{supp: eq decreasing}) then hold at equality when $k = j, j+1,\dots, t$. This implies that $g(\hat{X}^{j+1}; \hat{X}^j) = g(\hat{X}^j; \hat{X}^j)$,     so the algorithm terminates within finite iterations.
\end{proof}



\begin{proof}[Proof of Proposition~\ref{prop: generalization}]

For determining the convergence in Wasserstein distance between $p_{\hat{Z}}$ and $q_Z$, we first use the triangle inequality:
\begin{align*}
    \mathcal{W}_c(p_{\hat{Z}}, q_Z) &\leq \mathcal{W}_c(p_{\hat{Z}}, p_Z) + \mathcal{W}_c(p_Z, q_Z) = \lambda + \mathcal{W}_c(p_Z, q_Z)
\end{align*}
Note that the first term is deterministic, as it is the result of the optimization problem (\ref{pro general optimization}). For the second term, we use the result from \cite{fournier2015rate}, which implies that with probability $1-\alpha$ and for the 1-Wasserstein distance:
\begin{equation*}
    \mathbb{P}(\mathcal{W}_c(p_Z, q_Z) > \xi) \leq \exp \left(-c n \xi^{1/d} \right),
\end{equation*}
and so by setting the right-hand side equal to $\alpha$, or equivalently, setting
$$\xi = \sqrt[d]{ \frac{c \log(1/\alpha)}{n}}, $$ 
we obtain the first result.

For determining the convergence of the disparity between $p_{\hat{Z}}$ and $p_{\hat{Z}} (y,d)$ and $q_Y(y)$, we again first use the triangle inequality from the definition of the disparity $J$:

\begin{align*}
    \sup_{y \in \mathcal{Y}, d \in \mathcal{D}} J(p_{\hat{Z}} (y|d) , q_Y(y)) & = \sup_{y \in \mathcal{Y}, d \in \mathcal{D}} \frac{| p_{\hat{Z}}(y|d) - q_Y(y)|}{q_Y(y)} \\
    &\leq \sup_{y \in \mathcal{Y}, d \in \mathcal{D}} \left( \frac{| p_{\hat{Z}}(y|d) - p_Y(y)|}{q_Y(y)} + \frac{| p_Y(y) - q_Y(y)|}{q_Y(y)} \right) \\
    &\leq \sup_{y \in \mathcal{Y}, d \in \mathcal{D}} \frac{| p_{\hat{Z}}(y|d) - p_Y(y)|}{q_Y(y)} + \sup_{y \in \mathcal{Y}} \frac{| p_Y(y) - q_Y(y)|}{q_Y(y)}
\end{align*}

As the minimum of the marginal distribution of $q_Y(y)$ is bounded away from zero $\min_{y \in \mathcal{Y}} q_Z(y) = \rho >0$, and since by the optimization problem (\ref{pro general optimization}) we have $J(p_{\hat{Z}} (y|d) , p_Y(y)) \leq \epsilon$ for all $y \in \mathcal{Y}, d \in \mathcal{D}$:

\begin{align*}
    \sup_{y \in \mathcal{Y}, d \in \mathcal{D}} J(p_{\hat{Z}} (y|d) , q_Y(y))
    &\leq \sup_{y \in \mathcal{Y}, d \in \mathcal{D}} \frac{| p_{\hat{Z}}(y|d) - p_Y(y)|}{p_Y(y)} \frac{p_Y(y)}{q_Y(y)} + \sup_{y \in \mathcal{Y}} \frac{| p_Y(y) - q_Y(y)|}{q_Y(y)} \\
    &\leq \sup_{y \in \mathcal{Y}, d \in \mathcal{D}} J(p_{\hat{Z}} (y|d) , p_Y(y)) \frac{1}{\rho} + \sup_{y \in \mathcal{Y}} \frac{| p_Y(y) - q_Y(y)|}{q_Y(y)} \\
    &\leq \frac{\epsilon}{\rho} + \sup_{y \in \mathcal{Y}} \frac{| p_Y(y) - q_Y(y)|}{q_Y(y)}.
\end{align*}

Note that the first term is deterministic, while the second one is not, as we need to account for the uncertainty of observing $n$ i.i.d. samples $\left \{ Z_i \right\}_{i=1}^n$. For the second part, we use the Dvoretzky–Kiefer–Wolfowitz (DKW, \cite{dvoretzky1956asymptotic}) inequality:

\begin{align*}
\mathbb{P}\left( \sup_{y \in \mathcal{Y}} \frac{| p_Y(y) - q_Y(y)|}{q_Y(y)} > \xi \right) &\leq \mathbb{P}\left( \sup_{y \in \mathcal{Y}} | p_Y(y) - q_Y(y)| > \xi \rho \right) \\
&\leq 2\exp(-2n\rho^2\xi^2)
\end{align*}

By setting the right-hand side equal to $\alpha$, or equivalently, setting
$$
\xi = \sqrt{\frac{\log(\frac{2}{\alpha})}{2n\rho^2}},
$$
we obtain the second result.

Finally, we note that the assumption on the marginal distribution of $q_Y$ is bounded away from zero, i.e., $\rho >0$, is reasonable as the outcome is a discrete (usually binary) random variable. This assumption would be much more restricting in case $y$ was a continuous random variable (e.g., in regression settings).
\end{proof}

\subsection{Downstream Learning using \methodname} \label{app: hardness-learning}

Overall, the hardness of deriving bounds for the synthetic representatives provided by \methodname\ can be analyzed using the following breakdown, which is adapted from \cite{xu2023fair}. Consider the given dataset $Z=\{(X_i,Y_i,D_i) \}_{i=1}^n$, the synthetic representatives obtained using \methodname\ $\hat{Z}=\{(\hat{X}_j,\hat{Y}_j,\hat{D}_j) \}_{j=1}^m$ and $h \in \mathcal{H} = L^2(\mathcal{X} \times \mathcal{Y} \times \mathcal{D})$, the set of measurable square-integrable function in $L^2$. If we consider the downstream learning process using \methodname\ samples over the $L^2$ space:

$$
\inf_{h \in \mathcal{H}} \mathbb{E}_{Y|X, D} \left[ \|Y - h(\hat{X}, \hat{D}) \|^2_2 \right],
$$

then the above can be expanded in two terms, due to the property of the conditional expectation being an orthogonal operator in $\mathcal{H}$:

\begin{align}
    \inf_{h \in \mathcal{H}}&\mathbb{E}_{Y|X, D} \left[ \|Y - h(\hat{X}, \hat{D}) \|^2_2 \right] = \notag \\
    &= \underbrace{\mathbb{E}_{Y|X, D} \left[ \|Y - \mathbb{E}_{X,D}[\hat{Y}|\hat{X},\hat{D}] \|^2_2 \right]}_\text{\methodname\ Approximation Error} +  \underbrace{\inf_{h \in \mathcal{H}} \mathbb{E}_{Y|X, D} \left[ \| \mathbb{E}_{X,D}[\hat{Y}|\hat{X},\hat{D}] - h(\hat{X}, \hat{D}) \|^2_2 \right]}_\text{Learning with \methodname\ Samples}
\end{align}

The first term corresponds to the loss of information in approximating $Y$ with $\hat{Y}$ via the \methodname\ approach, and is actually independent of any downstream learning. This condition requires for the first moment (which for the binary $Y$ case is equivalent to the joint distribution) of $Y$ and $\hat{Y}$ to be as close as possible. Using \methodname\, this is enforced by minimizing the Wasserstein distance. Indeed, if in definition (\ref{pro original wasserstrein distance}) one restricts to couplings that admit marginal and conditional distributions, then the conditional distributions of $Y|X,D$ and $\hat{Y}|\hat{X}, \hat{D}$ are upper bounded in Wasserstein sense by the Wasserstein distance between the joint distribution of $p_Z$ and $p_{\hat{Z}}$ \cite{kim2022conditional}. 

The second terms refers to the training process using \methodname\ samples  Firstly, by using the equivalence in \cite{xu2023fair}, the second term is equivalent to $ \inf_{h \in \mathcal{H}} \mathbb{E}_{Y|X, D} \left[ \|\hat{Y} - f(\hat{X}, \hat{D} \|^2_2 \right]$, which correspond the finding the best $L^2$ function to approximate the distribution of $\hat{Y}$. This fact implies that using \methodname\ samples is indeed mathematically equivalent to the learning task for the original $Y$. However, this second term also highlights the hardness of developing learning bounds, as the \methodname\ synthetic representatives $\hat{Z}=\{(\hat{X}_j,\hat{Y}_j,\hat{D}_j) \}_{j=1}^m$ are not i.i.d., and hence standard bounds are not applicable.

\section{Experiment Details}\label{supp: exp-details}

\subsection{Runtime Analysis on Synthetic Dataset}\label{supp: synthetic-dataset}

As mentioned in Section~\ref{sec: experiments}, we generate a synthetic dataset in which one feature is strongly correlated with the protected attribute $D$ to induce a backdoor dependency on the outcome. We consider a binary protected attribute, $D \in \{0,1\}$, which could indicate e.g., gender or race.
The synthetic dataset contains two features, a feature $X_1$ correlated with the protected attribute and a feature $X_2$ uncorrelated with the protected attribute.
For $D=0$, $X_1$ is uniformly distributed in $[0,10]$, while for $D=1$, $X_1 = 0$. Instead, $X_2$ is $5$ times a random variable from a  normal distribution $\mathcal{N}(0,1)$. Finally, the outcome $Y$ is binary, so $Y = \{0, 1\}$: $Y_i = 1$ when $Y_i > m_x + \eps_i$ and $Y_i = 0$ when $Y_i \le m_x + \eps_i$, where $m_x$ is the mean of $\{(X_1)_i + (X_2)_i\}_i$ and the noise $\eps_i$ comes from a normal distribution $\mathcal{N}(0,1)$.

This experiment visualizes the speed of our method with respect to different numbers of overall samples $n$, number of samples in the  compressed dataset $m$, and the dimensionality of features $p$.
We evaluated the performance of the algorithm under the synthetic data with different configurations of $n$, $p$, and $m$. 
In this experiment, we set compute the fair Wasserstein coreset under the $l_1$-norm distance and we use k-means \cite{lloyd1982least} to initialize the starting coreset $\hat{X}^0$. 
We terminate the algorithm when $\hat{X}^k = \hat{X}^{k-1}$.
The time per iteration and total iterations for varying $n$, $p$, and $m$ are shown in the Figures~\ref{fig: main-results} (top left) and \ref{fig: supp-runtime}. We see that increasing the sample size of the original dataset $n$ increases the runtime and number of iterations, while increasing the number of coresets $m$ or dimensionality of the features $p$ reduces the overall numbers of iterations but increases each iteration's runtime. Additionally, for the setting of Figure~\ref{fig: main-results} (top left), in which we vary the dataset size $n$, Table~\ref{tab: supp-runtime} provides \methodname\ average runtimes from $n=500$ to $n=1,000,000$. We compare \methodname\ runtimes with the runtime at $n=500$ (our lowest dataset size in the experiment) extrapolated (i) linearly, with a factor of 1, (ii) linearly, with a factor of 10 and (iii) quadratically. We can see that the complexity is near linear and less than quadratic with respect to the dataset size $n$, although the rate indeed seem to increase for $n$ at $500,000$ and above, which can be attributed to the increasing number of iterations required to achieve convergence. This is akin to the phenomenon well-known for k-means, for which in larger datasets k-means might take an exponentially large number of iterations to terminate \cite{vattani2009k}. In practice, a fixed number of overall iterations is set to avoid this case: \textit{sklearn} sets it to $300$\footnote{\url{https://scikit-learn.org/stable/modules/generated/sklearn.cluster.KMeans.html}}, \textit{feiss} to $25$\footnote{\url{https://faiss.ai/cpp_api/struct/structfaiss_1_1Clustering.html}} and \textit{Matlab} to $100$\footnote{\url{https://www.mathworks.com/help/stats/kmeans.html}}.

\begin{figure}[!htpb]
    \centering
    \includegraphics[width=0.495\linewidth]{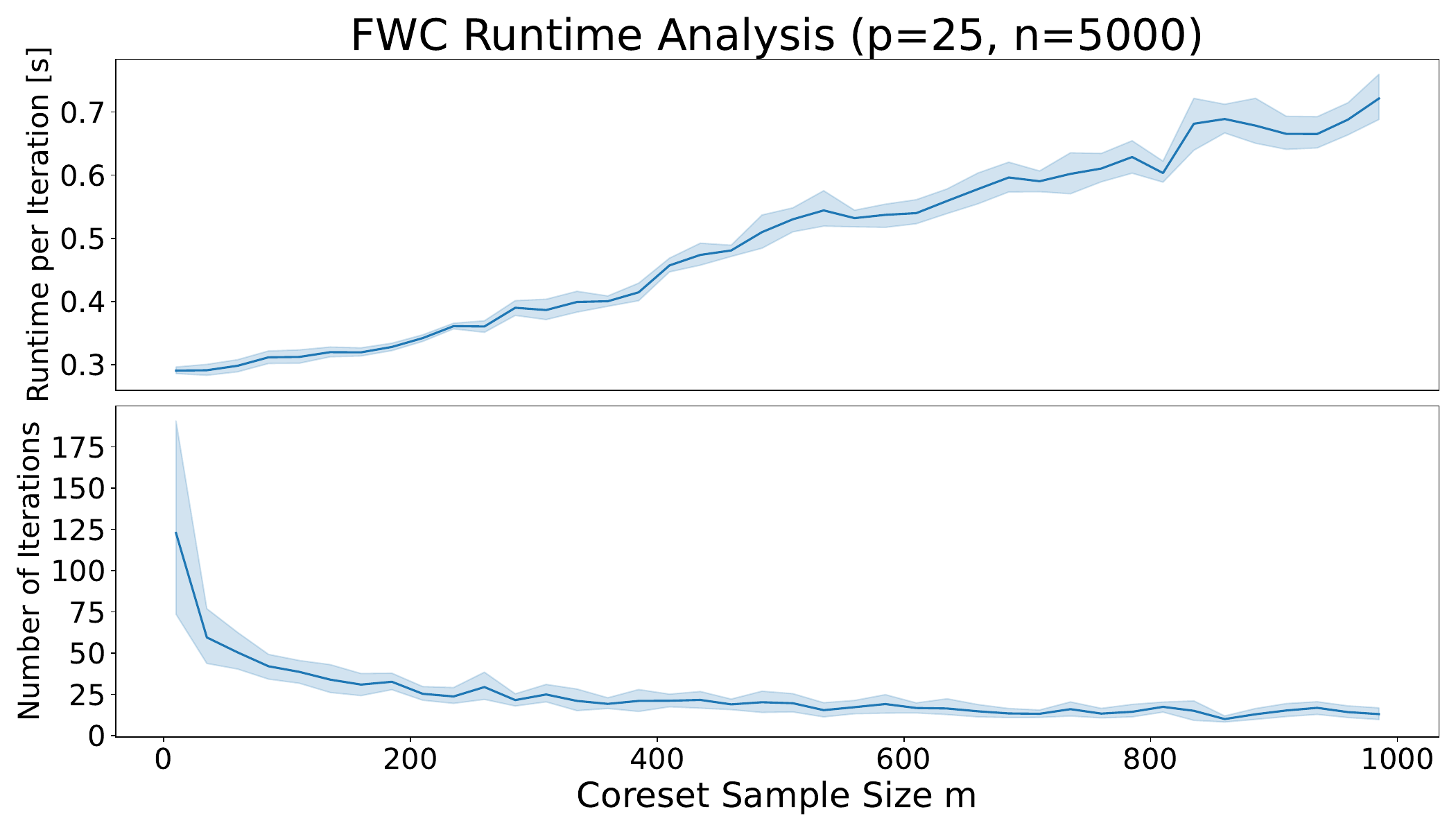}
    \includegraphics[width=0.495\linewidth]{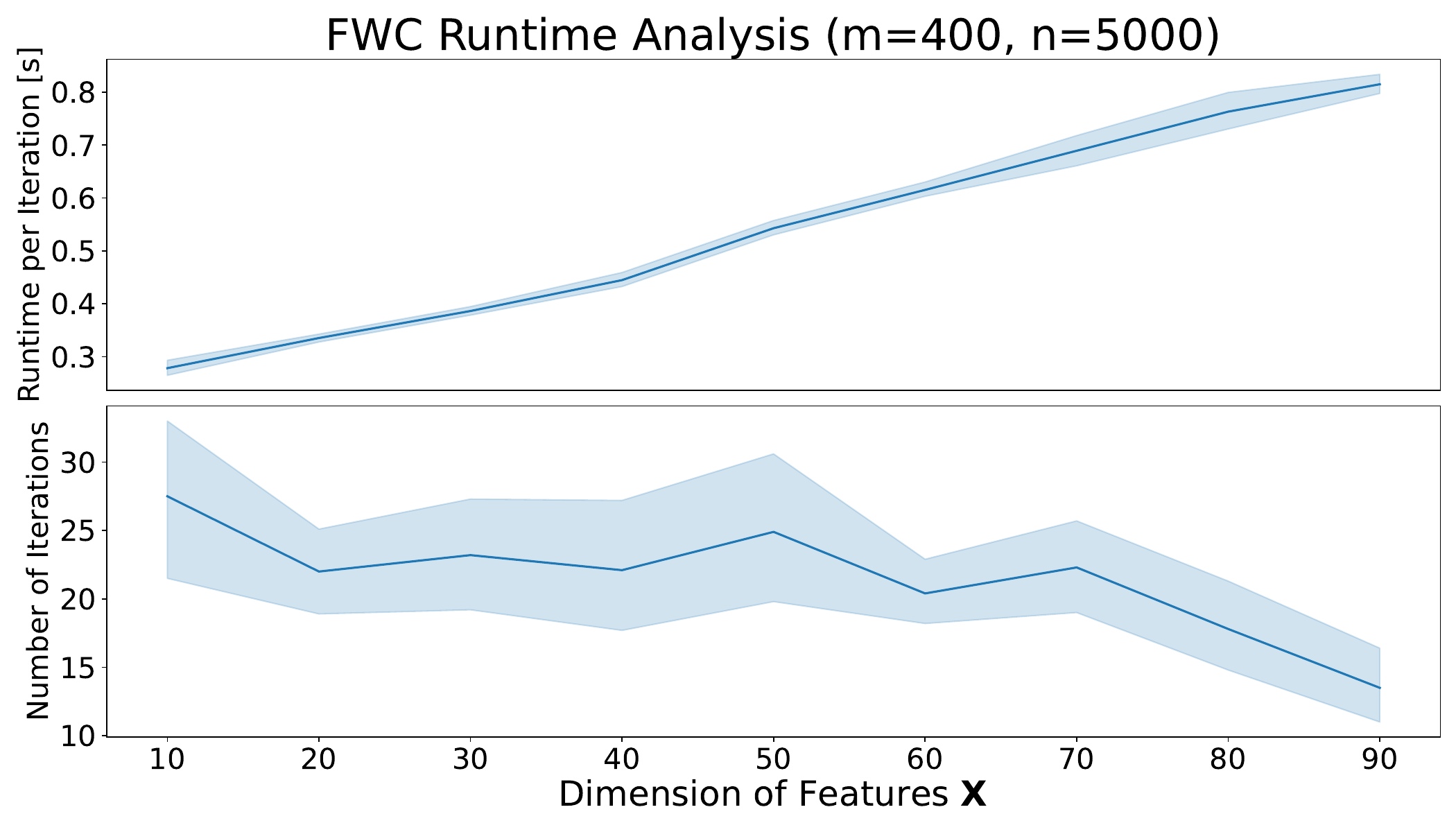}
    \vspace{-0.6cm}
    \caption{Runtime analysis of \methodname\, when varying the size of the coreset $m$ (left) and the dimensionality of the features $p$ (right). We report averages and one standard deviation computed over 10 runs.}
    \label{fig: supp-runtime}
\end{figure}

{\renewcommand{\arraystretch}{1.5}
\begin{table}[!ht]
\centering
\caption{Average runtimes for \methodname\ in the same settings as Figures~\ref{fig: main-results} (top left), varying the dataset size $n$ while fixing $m=250$ and $p=25$, compared with the runtimes for the smallest dataset extrapolated (i) linearly, with a factor of 1, (ii) linearly, with a factor of 10 and (iii) quadratically. \methodname\ enjoys a near linear time complexity, increasing with the largest dataset sizes; this phenomenon is shared with other clustering algorithm such as k-means (see text).}
\resizebox{1\textwidth}{!}{
\begin{tabular}{|c||c||c|c||c|c||c|c|}
\hline
\textbf{Dataset size $n$} & \textbf{Runtime [seconds]} & \textbf{Linear (Factor $1$)} & \textbf{Actual / Linear (Factor $1$)} & \textbf{Linear (Factor $10$)} & \textbf{Actual / Linear (Factor $10$)} & \textbf{Quadratic} & \textbf{Actual / Quadratic} \\ \hline \hline
500 & 8.9e-01 & - & - & - & - & - & - \\ \hline
1,000 & 9.4e-01 & 1.8e+00 & 0.52 & 1.8e+01 & 0.05 & 3.6e+00 & 0.26 \\ \hline
2,500 & 2.2e+00 & 4.5e+00 & 0.49 & 4.5e+01 & 0.05 & 2.2e+01 & 0.10 \\ \hline
5,000 & 1.0e+01 & 8.9e+00 & 1.17 & 8.9e+01 & 0.12 & 8.9e+01 & 0.12 \\ \hline
10,500 & 1.8e+01 & 1.9e+01 & 0.95 & 1.9e+02 & 0.09 & 3.9e+02 & 0.05 \\ \hline
24,500 & 8.7e+01 & 4.4e+01 & 1.99 & 4.4e+02 & 0.20 & 2.1e+03 & 0.04 \\ \hline
48,500 & 3.0e+02 & 8.6e+01 & 3.46 & 8.6e+02 & 0.35 & 8.4e+03 & 0.04 \\ \hline
75,000 & 7.0e+02 & 1.3e+02 & 5.25 & 1.3e+03 & 0.52 & 2.0e+04 & 0.03 \\ \hline
100,000 & 1.0e+03 & 1.8e+02 & 5.79 & 1.8e+03 & 0.58 & 3.6e+04 & 0.03 \\ \hline
250,000 & 3.8e+03 & 4.5e+02 & 8.51 & 4.5e+03 & 0.85 & 2.2e+05 & 0.02 \\ \hline
500,000 & 1.6e+04 & 8.9e+02 & 18.34 & 8.9e+03 & 1.83 & 8.9e+05 & 0.02 \\ \hline
750,000 & 2.9e+04 & 1.3e+03 & 21.64 & 1.3e+04 & 2.16 & 2.0e+06 & 0.01 \\ \hline
1,000,000 & 4.9e+04 & 1.8e+03 & 27.58 & 1.8e+04 & 2.76 & 3.6e+06 & 0.01 \\ \hline
\end{tabular}
}
\label{tab: supp-runtime}
\end{table}

\subsection{Real Datasets}\label{supp: exp-real-dataset-subsec}

We consider the following four real datasets widely used in the fairness literature \cite{fabris2022algorithmic}:

\begin{itemize}
    \item the \textit{Adult dataset} \cite{misc_adult_2}, which reports demographic data from the 1994 US demographic survey about $\sim 49,000$ individuals. We use all the available features for classification apart from the ``\texttt{fnlwgt}'' feature, including gender as the protected attribute $D$ and whether the individual salary is more than USD$50,000$;
    \item the \textit{Drug dataset} \cite{fehrman2017five}, which contains drug usage history for $1,885$ individuals. Features $X$ include the individual age, country of origin, education and scores on various psychological test. We use the individual gender as the protected variable $D$. The response $Y$ is based on whether the individual has reported to have ever used the drug ``cannabis'' or not; 
    \item the \textit{Communities and Crime dataset} \cite{misc_communities_and_crime_183} was put together towards the creation of a software tool for the US police department. The dataset contains socio-economic factors for $\sim 2,000$ communities in the US, along with the proportion of violent crimes in each community. As protected attribute $D$, we include whether the percentage of the black population in the community is above the overall median. For the response $Y$, we use a binary indicator of whether the violent crimes percentage level is above the mean across all communities in the dataset;
    \item the \textit{German Credit dataset} \cite{misc_statlog_(german_credit_data)_144} reports a set of $1,000$ algorithmic credit decisions from a regional bank in Germany. We use all the available features, including gender as protected attribute $D$ and whether the credit was approved as response $Y$.
\end{itemize}

As perfect demographic parity achieves a value of $0$ for the discrepancy $J$, so we have included demographic ``dis''-parity to indicate any deviation from demographic parity. Across all experiments we compute demographic parity as the following absolute difference:
\begin{equation} \label{eq: dd}
    DD \defeq \left|p\left(h(X,D) = 1| D=1\right) - p(h(X,D) = 1| D=0)\right|,
\end{equation}
for a given classifier $h$ and a protected attribute $D$ with two levels; the larger this difference, the larger the disparity.
We also include the implementation and hyper-parameters of the methods used in the fairness-utility tradeoffs throughout the experiments in this paper:

\begin{itemize}
    \item For \methodname\, we set the fairness violation hyper-parameter $\epsilon$ of problem in Equation (\ref{pro general optimization model}) to be $\epsilon=[0.01, 0.05, 0.1]$, hence obtaining three separate  \methodname\, models, \methodname\,$(0.01$),  \methodname\,$(0.05$) and  \methodname\,$(0.1$);
    \item For \texttt{Fairlet} \cite{backurs2019scalable}, we use the implementation available at the following GitHub repository: \url{https://github.com/talwagner/fair_clustering/tree/master}
    \item For \texttt{IndFair} \cite{chhaya2022coresets} and  \texttt{K-Median Coresets} \cite{bachem2018one}, we use the implementation available at the following GitHub repository: \url{https://github.com/jayeshchoudhari/CoresetIndividualFairness/tree/master}
    \item For k-means \cite{lloyd1982least} and k-medoids \cite{maranzana1963location, park2009simple} we use the implementations available in the Python package \texttt{Scikit-Learn} \cite{pedregosa2011scikit}
\end{itemize}

All computations are run on an Ubuntu machine with 32GB of RAM and 2.50GHz Intel(R) Xeon(R) Platinum 8259CL CPU. For all datasets, we randomly split $75\%$ of the data into training/test set, and change the split during each separate run; the training data are further separated into training and validation with $90/10$ to compute early stopping criteria during training. The downstream classifier used is a one-layer deep multi-layer perceptron (MLP) with $20$ hidden layers, ReLu activation function in the hidden layer and softmax activation function in the final layer. Unless stated otherwise, \methodname\ uses the $L^1$ to compute the distance from the original datasets in the optimization problem. For the downstream classifier, we use Adam optimizer \cite{kingma2014adam} with a learning rate set to $10^{-3}$, a batch size of $32$, a maximum number of epochs set to $500$ with early stopping evaluated on the separate validation set with a patience of $10$ epochs and both the features $X$ and the protected attribute $D$ are used for training the classifier. Note that due to the size of the Adult dataset, \texttt{Fairlet} coresets \cite{backurs2019scalable} could not be run due to the RAM memory required exceeding the machine capacity (32GB). Finally, all uncertainties are reported at $\pm 1\sigma$ (one standard deviation) in both figures and tables. Uncertainties are computed over a set of 10 runs where the random seed for the algorithm initialization and train/test split was changed, but consistent across all methods (i.e., all methods in the first run were presented the same train/test split across each datasets).

{\renewcommand{\arraystretch}{1.5}
\begin{table}[!ht]
\centering
\caption{Wasserstein distance of the weighted coresets with respect to the original dataset, with averages and standard deviations obtained over 10 runs. In bold, coresets with the closest distance to the original dataset (i.e., smallest Wasserstein distance) in each coreset size and dataset combination.}
\resizebox{1\textwidth}{!}{
\begin{tabular}{|c||c|c|c|c|c|c|c|c|c|c|c|c|}
\hline
& \multicolumn{12}{c|}{\textbf{Wasserstein Distance (}$\downarrow$\textbf{)}} \\ \hline
\textbf{Method} & \multicolumn{3}{c|}{\textit{Adult} ($\times 10^6$)}  & \multicolumn{3}{c|}{\textit{Credit} ($\times 10^6$) } & \multicolumn{3}{c|}{\textit{Crime}} & \multicolumn{3}{c|}{\textit{Drug}} \\ \hline
\textit{Coreset Size m} & $0.5\%$ & $1\%$ & $2\%$ & $5\%$ & $10\%$ & $20\%$ & $5\%$ & $10\%$ & $20\%$ & $5\%$ & $10\%$ & $20\%$ \\ \hline
\methodname\ ($\epsilon$: 0.01) & 5.72 $\pm$ 0.95 & \textbf{3.76 $\pm$ 0.80} & \textbf{2.70 $\pm$ 0.56} & \textbf{0.40 $\pm$ 0.10} & 0.75 $\pm$ 0.21 & 1.07 $\pm$ 0.28 & \textbf{1.65 $\pm$ 0.08} & 1.76 $\pm$ 0.09 & 1.89 $\pm$ 0.08 & \textbf{3.23 $\pm$ 0.23} & 3.57 $\pm$ 0.24 & 3.98 $\pm$ 0.22  \\ \hline
\methodname\ ($\epsilon$: 0.05) & \textbf{5.57 $\pm$ 0.73} & 3.90 $\pm$ 0.82 & 2.79 $\pm$ 0.58 & 1.07 $\pm$ 0.28 & 0.48 $\pm$ 0.24 & 0.72 $\pm$ 0.17 & 1.87 $\pm$ 0.11 & \textbf{1.66 $\pm$ 0.12} & 1.72 $\pm$ 0.06 & 3.95 $\pm$ 0.30 & \textbf{3.31 $\pm$ 0.34} & 3.49 $\pm$ 0.19  \\ \hline
\methodname\ ($\epsilon$: 0.1) & 6.18 $\pm$ 1.07 & 4.15 $\pm$ 0.82 & 3.04 $\pm$ 0.63 & 0.71 $\pm$ 0.15 & 1.05 $\pm$ 0.30 & 0.47 $\pm$ 0.19 & 1.70 $\pm$ 0.07 & 1.85 $\pm$ 0.11 & \textbf{1.62 $\pm$ 0.08} & 3.50 $\pm$ 0.21 & 3.95 $\pm$ 0.29 & \textbf{3.25 $\pm$ 0.24}  \\ \hline
Fairlet (K-Means) & - & - & - & 10.02 $\pm$ 2.14 & 7.47 $\pm$ 2.77 & 5.00 $\pm$ 2.48 & 2.10 $\pm$ 0.42 & 2.10 $\pm$ 0.44 & 2.00 $\pm$ 0.36 & 4.29 $\pm$ 0.44 & 4.10 $\pm$ 0.45 & 3.83 $\pm$ 0.43  \\ \hline
Fairlet (K-Medoids) & - & - & - & 3.16 $\pm$ 0.81 & 4.48 $\pm$ 0.89 & 5.47 $\pm$ 0.65 & 4.21 $\pm$ 0.17 & 4.14 $\pm$ 0.14 & 4.10 $\pm$ 0.06 & 6.50 $\pm$ 0.45 & 5.90 $\pm$ 0.47 & 5.17 $\pm$ 0.43  \\ \hline
Ind. Fair. Cor. & 12.96 $\pm$ 5.83 & 23.91 $\pm$ 25.39 & 11.48 $\pm$ 8.19 & \textbf{0.41 $\pm$ 0.29} & 0.72 $\pm$ 0.46 & \textbf{0.23 $\pm$ 0.12} & 2.35 $\pm$ 0.16 & 2.60 $\pm$ 0.25 & 2.01 $\pm$ 0.09 & 4.84 $\pm$ 0.33 & 5.51 $\pm$ 0.50 & 4.17 $\pm$ 0.21  \\ \hline
K-Median Cor. & 22.11 $\pm$ 17.42 & 9.37 $\pm$ 11.27 & 12.04 $\pm$ 9.11 & 0.66 $\pm$ 0.74 & \textbf{0.26 $\pm$ 0.17} & 0.36 $\pm$ 0.17 & 2.58 $\pm$ 0.20 & 2.01 $\pm$ 0.11 & 2.35 $\pm$ 0.12 & 5.38 $\pm$ 0.54 & 4.16 $\pm$ 0.22 & 4.80 $\pm$ 0.15  \\ \hline
Uniform Subsampling & 18.01 $\pm$ 18.34 & 14.39 $\pm$ 14.90 & 7.96 $\pm$ 5.42 & 0.74 $\pm$ 0.53 & 0.30 $\pm$ 0.16 & 0.80 $\pm$ 1.15 & 2.60 $\pm$ 0.26 & 1.95 $\pm$ 0.13 & 2.28 $\pm$ 0.24 & 5.46 $\pm$ 0.59 & 4.01 $\pm$ 0.30 & 4.75 $\pm$ 0.42  \\ \hline
K-Means & 43.96 $\pm$ 18.82 & 77.04 $\pm$ 18.06 & 74.50 $\pm$ 12.53 & 13.23 $\pm$ 1.42 & 10.76 $\pm$ 1.55 & 7.59 $\pm$ 1.58 & 2.00 $\pm$ 0.12 & 1.93 $\pm$ 0.11 & 2.04 $\pm$ 0.08 & 3.94 $\pm$ 0.17 & 3.66 $\pm$ 0.20 & 3.62 $\pm$ 0.32  \\ \hline
K-Medoids & 50.87 $\pm$ 4.30 & 53.31 $\pm$ 1.85 & 51.77 $\pm$ 1.86 & 2.59 $\pm$ 0.42 & 4.25 $\pm$ 0.68 & 5.18 $\pm$ 0.07 & 3.15 $\pm$ 0.08 & 3.12 $\pm$ 0.11 & 2.83 $\pm$ 0.06 & 5.14 $\pm$ 0.08 & 4.53 $\pm$ 0.28 & 3.91 $\pm$ 0.09  \\ \hline
\end{tabular}
}
\label{tab: supp-wass_cost_numbers}
\end{table}

{\renewcommand{\arraystretch}{1.5}
\begin{table}[!ht]
\centering
\caption{Clustering cost of the coresets with respect to the original dataset, with averages and standard deviations obtained over 10 runs. In bold, coresets with the smallest clustering cost (i.e., smallest sum of square distances of original dataset samples from the closest generated coreset sample) in each coreset size and dataset combination.}
\resizebox{1\textwidth}{!}{
\begin{tabular}{|c||c|c|c|c|c|c|c|c|c|c|c|c|}
\hline
& \multicolumn{12}{c|}{\textbf{Clustering Cost (}$\downarrow$\textbf{)}} \\ \hline
\textbf{Method} & \multicolumn{3}{c|}{\textit{Adult} ($\times 10^6$)}  & \multicolumn{3}{c|}{\textit{Credit} ($\times 10^4$) } & \multicolumn{3}{c|}{\textit{Crime} ($\times 10^2$)} & \multicolumn{3}{c|}{\textit{Drug} ($\times 10^2$)} \\ \hline
\textit{Coreset Size m} & $0.5\%$ & $1\%$ & $2\%$ & $5\%$ & $10\%$ & $20\%$ & $5\%$ & $10\%$ & $20\%$ & $5\%$ & $10\%$ & $20\%$ \\ \hline
\methodname\ ($\epsilon$: 0.01) & 14.73 $\pm$ 6.43 & 5.44 $\pm$ 2.29 & 2.07 $\pm$ 0.39 & \textbf{1.07 $\pm$ 0.21} & 2.61 $\pm$ 0.59 & 4.31 $\pm$ 0.92 & \textbf{3.86 $\pm$ 0.27} & 4.44 $\pm$ 0.21 & 4.79 $\pm$ 0.19 & \textbf{5.54 $\pm$ 0.41} & 6.45 $\pm$ 0.32 & 7.02 $\pm$ 0.31  \\ \hline
\methodname\ ($\epsilon$: 0.05) & 13.59 $\pm$ 4.21 & 5.81 $\pm$ 2.37 & 2.26 $\pm$ 0.30 & 4.41 $\pm$ 1.52 & 1.39 $\pm$ 0.96 & 2.44 $\pm$ 0.21 & 4.72 $\pm$ 0.36 & \textbf{3.96 $\pm$ 0.41} & 4.37 $\pm$ 0.14 & 6.94 $\pm$ 0.54 & \textbf{5.70 $\pm$ 0.63} & 6.35 $\pm$ 0.24  \\ \hline
\methodname\ ($\epsilon$: 0.1) & 20.29 $\pm$ 12.27 & 6.31 $\pm$ 2.36 & 2.19 $\pm$ 0.29 & 2.48 $\pm$ 0.20 & 4.00 $\pm$ 1.28 & 1.22 $\pm$ 0.41 & 4.36 $\pm$ 0.15 & 4.72 $\pm$ 0.34 & 3.89 $\pm$ 0.30 & 6.36 $\pm$ 0.26 & 6.94 $\pm$ 0.51 & 5.62 $\pm$ 0.45  \\ \hline
Fairlet (K-Means) & - & - & - & 1.36 $\pm$ 0.09 & 0.76 $\pm$ 0.11 & 0.60 $\pm$ 0.31 & 4.76 $\pm$ 0.46 & 4.36 $\pm$ 0.52 & 3.92 $\pm$ 0.56 & 7.08 $\pm$ 0.30 & 6.45 $\pm$ 0.40 & 5.76 $\pm$ 0.63  \\ \hline
Fairlet (K-Medoids) & - & - & - & 5.79 $\pm$ 1.36 & 5.40 $\pm$ 1.08 & 5.04 $\pm$ 0.92 & 6.27 $\pm$ 0.28 & 6.00 $\pm$ 0.27 & 5.59 $\pm$ 0.20 & 8.58 $\pm$ 0.41 & 7.93 $\pm$ 0.42 & 7.09 $\pm$ 0.54  \\ \hline
Ind. Fair. Cor. & 1.23 $\pm$ 0.37 & 4.67 $\pm$ 9.46 & 0.70 $\pm$ 0.04 & 1.50 $\pm$ 0.38 & 2.46 $\pm$ 0.77 & 0.74 $\pm$ 0.10 & 5.24 $\pm$ 0.27 & 5.64 $\pm$ 0.48 & 4.53 $\pm$ 0.22 & 7.29 $\pm$ 0.35 & 7.92 $\pm$ 0.62 & 6.30 $\pm$ 0.41  \\ \hline
K-Median Cor. & 10.45 $\pm$ 14.21 & 0.70 $\pm$ 0.04 & 1.06 $\pm$ 0.19 & 2.47 $\pm$ 0.90 & 0.76 $\pm$ 0.06 & 1.37 $\pm$ 0.29 & 5.64 $\pm$ 0.39 & 4.55 $\pm$ 0.27 & 5.21 $\pm$ 0.16 & 7.75 $\pm$ 0.66 & 6.26 $\pm$ 0.36 & 7.26 $\pm$ 0.20  \\ \hline
Uniform Subsampling & 10.16 $\pm$ 15.13 & 3.49 $\pm$ 1.18 & 1.93 $\pm$ 0.67 & 3.38 $\pm$ 1.42 & 1.09 $\pm$ 0.26 & 2.74 $\pm$ 2.71 & 5.61 $\pm$ 0.45 & 4.46 $\pm$ 0.34 & 5.07 $\pm$ 0.37 & 7.74 $\pm$ 0.73 & 6.06 $\pm$ 0.51 & 7.02 $\pm$ 0.50  \\ \hline
K-Means & \textbf{0.72 $\pm$ 0.01} & \textbf{0.58 $\pm$ 0.01} & \textbf{0.48 $\pm$ 0.01} & 1.28 $\pm$ 0.21 & \textbf{0.68 $\pm$ 0.1}2 & \textbf{0.51 $\pm$ 0.29} & 4.52 $\pm$ 0.16 & 4.09 $\pm$ 0.25 & \textbf{3.76 $\pm$ 0.39} & 6.86 $\pm$ 0.28 & 6.13 $\pm$ 0.41 & \textbf{5.61 $\pm$ 0.67}  \\ \hline
K-Medoids & 83.53 $\pm$ 13.86 & 75.10 $\pm$ 7.22 & 56.95 $\pm$ 9.17 & 4.93 $\pm$ 0.53 & 4.70 $\pm$ 0.48 & 4.19 $\pm$ 0.13 & 5.52 $\pm$ 0.07 & 5.15 $\pm$ 0.28 & 4.49 $\pm$ 0.25 & 7.65 $\pm$ 0.12 & 6.83 $\pm$ 0.43 & 5.93 $\pm$ 0.30  \\ \hline
\end{tabular}
}
\label{tab: supp-clust_cost_numbers}
\end{table}

\paragraph{Closeness to the original dataset and clustering performance} Tables~\ref{tab: supp-wass_cost_numbers} and \ref{tab: supp-clust_cost_numbers} include the numerical values (means and standard deviations computed over 10 runs) for all methods in terms of Wasserstein distance from the original dataset and clustering cost, for the three coreset sizes $m = [5\%, 10\%, 20\%]$ (apart from the Adult dataset, in which coreset sizes are set to $m = [0.5\%, 1\%, 2\%]$ due to the large size of the original dataset).  Clustering cost is computed as the sum of the squared distance of each point in the original dataset from the closest coreset representative, while the Wasserstein distance is computed solving the optimal transport between the empirical distribution of the original dataset and the one of the coresets, using the $L^1$ norm as cost function. \methodname\, consistently provides the closest distributional distance to the original dataset in Wasserstein distance, with the only exception of the Credit dataset, in which the large number of discrete features makes the optimization non-smooth in feature space, resulting in a potentially imprecise solution of Equation (\ref{pro simplified target problem 2}). In addition, \methodname\,, while not naturally minimizing clustering costs, seems to achieve competitive clustering costs in smaller datasets while not performing as well on larger datasets such as Adult. Finally, we note that the Wasserstein distance might not always decrease with a higher coreset size, which is due to the parity violation constraint $\epsilon$. In other words, the coreset samples not only have to be close to the original dataset distribution but also respect the hard fairness constraint; the lower the $\epsilon$, the tighter this constraint is (Equation~(\ref{pro general optimization})). Indeed, when $\epsilon$ is the largest ($\epsilon = 0.1$), coresets of size $20\%$ (or $2\%$ for the Adult dataset, i.e., the largest) consistently has a smaller Wasserstein distance to the original dataset than coresets of size $5\%$ ($0.5\%$ for the Adult dataset, i.e., the smallest).

\paragraph{Fairness-utility tradeoff when using coresets for training downstream models} Figure~\ref{fig: supp-all-results} expands the results provided in Figure~\ref{fig: main-results} and shows all the fairness-utility tradeoffs for all methods across the four datasets, both with (right column) and without (left column) using a pre-processing fairness approach \cite{kamiran2012data} (excluding \methodname\ , to which no fairness modification is applied after coresets have been generated). For each method, the coreset size that achieves the best fairness-utility tradeoff is shown (which is not necessarily the coreset with the largest size). \methodname\, achieves a competitive fairness-utility tradeoff with respect to other competing methods, when using the generated coresets to train a downstream MLP classifier model. \methodname\ consistently reduces disparity in the downstream classification with respect to other approaches, and often maintains the same utility (indicated by the AUC). For completeness, Figure~\ref{fig: supp-all-results-standard-deviations} also reports standard deviations for the fairness-utility tradeoff; standard deviations for the Adult and Credit datasets are large due to the MLP classifier becoming trivial (i.e., always returning $0$s or $1$s), which yields very low performance but has no demographic disparity (by definition, as all test samples are assigned the same outcome). Tables~\ref{tab: supp-learning-comp-numbers} and \ref{tab: supp-learning-comp-reweighing} report the numbers shown in Figures \ref{fig: supp-all-results} and \ref{fig: supp-all-results-standard-deviations}, including one number to quantify the fairness-utility tradeoff, computed as the Euclidean distance in the Figure from the $(0,1)$ point (which would be a fair classifier with perfect performance). In other words, for the $k$-th method achieving a disparity of $d_k$ with uncertainty $\Delta d_k$ and performance $a_k$ with uncertainty $\Delta a_k$, the tradeoff $t_k$ and associated uncertainty $\Delta t_k$ are quantified as:

\begin{equation} \label{eq: tradeoff}
    t_k = \sqrt{(1-a_k)^2 + d_k^2} \quad , \quad \Delta t_k = \sqrt{\left(\frac{d_k}{t_k} \Delta d_k  \right)^2 + \left(\frac{a_k - 1}{t_k} \Delta a_k \right)^2}
\end{equation}

Finally, the average reduction in disparity was computed from Tables~\ref{tab: supp-learning-comp-numbers} and \ref{tab: supp-learning-comp-reweighing}, by averaging the improvement obtained by \methodname\ samples against all methods and across datasets. \methodname\ result in an average reduction in disparity of $53\%$ and $18\%$ for the scenario without and with fairness pre-processing, respectively.

{\renewcommand{\arraystretch}{1.5}
\begin{table}[!ht]
\centering
\caption{Demographic disparity (Equation (\ref{eq: dd})), AUC and fairness-utility tradeoff (Equation (\ref{eq: tradeoff})) of downstream MLP classifier trained using all fair coresets/clustering methods across the four real datasets. The best method across the 3 different coreset sizes is shown, and the best performing method for each metric in each dataset is bolded. Averages and standard deviations taken over 10 runs. For the Credit dataset, K-means reaches low disparities due to the classifier being trivial, i.e., returning the same prediction regardless of input features.}
\resizebox{1\textwidth}{!}{
\begin{tabular}{|c||c|c|c||c|c|c||c|c|c||c|c|c||}
\hline
\textbf{Method} & \multicolumn{3}{c||}{\textit{Adult Dataset}}  & \multicolumn{3}{c||}{\textit{Credit Dataset}} & \multicolumn{3}{c||}{\textit{Crime Dataset}} & \multicolumn{3}{c||}{\textit{Drug Dataset}} \\ \hline
& DD ($\downarrow$) & AUC ($\uparrow$) & Tradeoff ($\downarrow$) & DD ($\downarrow$) & AUC ($\uparrow$) & Tradeoff ($\downarrow$) & DD ($\downarrow$) & AUC ($\uparrow$) & Tradeoff ($\downarrow$) & DD ($\downarrow$) & AUC ($\uparrow$) & Tradeoff ($\downarrow$) \\ \hline \hline
\methodname\ ($\epsilon$: 0.01) & \textbf{0.02 $\pm$  0.02} &  0.67 $\pm$  0.10 & 0.33 $\pm$  0.10 & 0.03 $\pm$  0.04 &  0.67 $\pm$  0.11 & 0.33 $\pm$  0.11 & \textbf{0.13 $\pm$  0.05} &  0.83 $\pm$  0.02 & \textbf{0.22 $\pm$  0.04}& \textbf{0.05 $\pm$  0.04} &  0.79 $\pm$  0.02 & \textbf{0.21 $\pm$  0.02} \\ \hline
\methodname\ ($\epsilon$: 0.05) & 0.07 $\pm$  0.03 &  \textbf{0.75 $\pm$  0.08} & \textbf{0.26 $\pm$  0.07} & 0.02 $\pm$  0.02 &  0.64 $\pm$  0.11 & 0.36 $\pm$  0.11& 0.15 $\pm$  0.04 &  0.84 $\pm$  0.02 & \textbf{0.22 $\pm$  0.03}& 0.06 $\pm$  0.02 &  0.80 $\pm$  0.02 & \textbf{0.21 $\pm$  0.02} \\ \hline
\methodname\ ($\epsilon$: 0.1) & 0.06 $\pm$  0.03 &  0.70 $\pm$  0.08 & 0.30 $\pm$  0.08& 0.03 $\pm$  0.03 &  0.66 $\pm$  0.12 & 0.34 $\pm$  0.12& 0.16 $\pm$  0.03 &  0.85 $\pm$  0.01 & \textbf{0.22 $\pm$  0.03}& 0.08 $\pm$  0.03 &  0.80 $\pm$  0.01 & \textbf{0.21 $\pm$  0.02} \\ \hline
Fairlet (K-Means) & - & - & - & \textbf{0.01 $\pm$  0.01} &  0.57 $\pm$  0.11 & 0.43 $\pm$  0.11& 0.38 $\pm$  0.09 &  0.90 $\pm$  0.02 & 0.39 $\pm$  0.09& 0.16 $\pm$  0.06 &  0.80 $\pm$  0.02 & 0.25 $\pm$  0.04 \\ \hline
Fairlet (K-Medoids) & - & - & - & 0.04 $\pm$  0.06 &  0.60 $\pm$  0.09 & 0.41 $\pm$  0.09& 0.24 $\pm$  0.16 &  0.86 $\pm$  0.04 & 0.28 $\pm$  0.14& 0.10 $\pm$  0.09 &  0.76 $\pm$  0.05 & 0.26 $\pm$  0.06 \\ \hline
Ind. Fair. Cor. & 0.08 $\pm$  0.07 &  0.73 $\pm$  0.08 & 0.28 $\pm$  0.08& 0.05 $\pm$  0.05 &  0.65 $\pm$  0.11 & 0.35 $\pm$  0.11& 0.39 $\pm$  0.10 &  0.88 $\pm$  0.02 & 0.40 $\pm$  0.10& 0.21 $\pm$  0.06 &  \textbf{0.81 $\pm$  0.02} & 0.28 $\pm$  0.05 \\ \hline
K-Median Cor. & 0.12 $\pm$  0.08 &  0.72 $\pm$  0.09 & 0.30 $\pm$  0.09& 0.07 $\pm$  0.07 &  0.66 $\pm$  0.12 & 0.35 $\pm$  0.12& 0.45 $\pm$  0.05 &  \textbf{0.91 $\pm$  0.01} & 0.46 $\pm$  0.05& 0.16 $\pm$  0.10 &  0.76 $\pm$  0.06 & 0.29 $\pm$  0.08 \\ \hline
Uniform Subsampling & 0.12 $\pm$  0.06 &  0.71 $\pm$  0.10 & 0.31 $\pm$  0.10& 0.07 $\pm$  0.09 &  \textbf{0.69 $\pm$  0.11} & \textbf{0.31 $\pm$  0.11} & 0.45 $\pm$  0.08 &  0.89 $\pm$  0.02 & 0.46 $\pm$  0.08& 0.14 $\pm$  0.11 &  0.78 $\pm$  0.03 & 0.26 $\pm$  0.06 \\ \hline
K-Means & 0.08 $\pm$  0.04 &  0.68 $\pm$  0.10 & 0.33 $\pm$  0.10& 0.00 $\pm$  0.00 $^\star$ &  0.51 $\pm$  0.02 & 0.49 $\pm$  0.02& 0.40 $\pm$  0.03 &  \textbf{0.91 $\pm$  0.01} & 0.41 $\pm$  0.03& 0.15 $\pm$  0.04 &  \textbf{0.81 $\pm$  0.02} & 0.25 $\pm$  0.03 \\ \hline
K-Medoids & 0.08 $\pm$  0.04 &  0.71 $\pm$  0.12 & 0.30 $\pm$  0.11& 0.05 $\pm$  0.05 &  0.60 $\pm$  0.07 & 0.40 $\pm$  0.07& 0.28 $\pm$  0.10 &  0.88 $\pm$  0.04 & 0.30 $\pm$  0.09& 0.12 $\pm$  0.05 &  0.80 $\pm$  0.02 & 0.23 $\pm$  0.03 \\ \hline
\end{tabular}
}
\label{tab: supp-learning-comp-numbers}
\end{table}

{\renewcommand{\arraystretch}{1.5}
\begin{table}[!ht]
\centering
\caption{Demographic disparity (Equation (\ref{eq: dd})), AUC and fairness-utility tradeoff (Equation (\ref{eq: tradeoff})) of downstream MLP classifier trained using all fair coresets/clustering methods across the four real datasets. All methods apart from \methodname\ have been corrected for fairness using a preprocessing fairness technique by \cite{kamiran2012data}. The best method across the 3 different coreset sizes is shown, and the best performing method for each metric in each dataset is bolded. Averages and standard deviations taken over 10 runs.}
\resizebox{1\textwidth}{!}{
\begin{tabular}{|c||c|c|c||c|c|c||c|c|c||c|c|c||}
\hline
\textbf{Method} & \multicolumn{3}{c||}{\textit{Adult Dataset}}  & \multicolumn{3}{c||}{\textit{Credit Dataset}} & \multicolumn{3}{c||}{\textit{Crime Dataset}} & \multicolumn{3}{c||}{\textit{Drug Dataset}} \\ \hline
& DD ($\downarrow$) & AUC ($\uparrow$) & Tradeoff ($\downarrow$) & DD ($\downarrow$) & AUC ($\uparrow$) & Tradeoff ($\downarrow$) & DD ($\downarrow$) & AUC ($\uparrow$) & Tradeoff ($\downarrow$) & DD ($\downarrow$) & AUC ($\uparrow$) & Tradeoff ($\downarrow$) \\ \hline \hline
\methodname\ ($\epsilon$: 0.01) & \textbf{0.02 $\pm$  0.02} &  0.67 $\pm$  0.10 & 0.33 $\pm$  0.10& 0.03 $\pm$  0.04 &  0.67 $\pm$  0.11 & \textbf{0.33 $\pm$  0.11}& 0.13 $\pm$  0.05 &  0.83 $\pm$  0.02 & 0.22 $\pm$  0.04& \textbf{0.05 $\pm$  0.04} &  0.79 $\pm$  0.02 & \textbf{0.21 $\pm$  0.02} \\ \hline
\methodname\ ($\epsilon$: 0.05) & 0.07 $\pm$  0.03 &  \textbf{0.75 $\pm$  0.08} & \textbf{0.26 $\pm$  0.07}& 0.02 $\pm$  0.02 &  0.64 $\pm$  0.11 & 0.36 $\pm$  0.11& 0.15 $\pm$  0.04 &  0.84 $\pm$  0.02 & 0.22 $\pm$  0.03& 0.06 $\pm$  0.02 &  0.80 $\pm$  0.02 & \textbf{0.21 $\pm$  0.02} \\ \hline
\methodname\ ($\epsilon$: 0.1) & 0.06 $\pm$  0.03 &  0.70 $\pm$  0.08 & 0.30 $\pm$  0.08& 0.03 $\pm$  0.03 &  0.66 $\pm$  0.12 & 0.34 $\pm$  0.12& 0.16 $\pm$  0.03 &  0.85 $\pm$  0.01 & 0.22 $\pm$  0.03& 0.08 $\pm$  0.03 &  0.80 $\pm$  0.01 & \textbf{0.21 $\pm$  0.02} \\ \hline
Fairlet (K-Means) & - & - & - & 0.02 $\pm$  0.03 &  0.60 $\pm$  0.08 & 0.40 $\pm$  0.08& 0.23 $\pm$  0.05 &  0.87 $\pm$  0.02 & 0.26 $\pm$  0.04& 0.12 $\pm$  0.06 &  \textbf{0.82 $\pm$  0.02} & \textbf{0.21 $\pm$  0.04} \\ \hline
Fairlet (K-Medoids) & - & - & - & 0.04 $\pm$  0.05 &  0.59 $\pm$  0.09 & 0.41 $\pm$  0.09& \textbf{0.11 $\pm$  0.14} &  0.81 $\pm$  0.08 & 0.22 $\pm$  0.10& 0.09 $\pm$  0.07 &  0.76 $\pm$  0.04 & 0.26 $\pm$  0.04 \\ \hline
Ind. Fair. Cor. & 0.05 $\pm$  0.03 &  0.68 $\pm$  0.08 & 0.33 $\pm$  0.08& 0.04 $\pm$  0.05 &  0.63 $\pm$  0.09 & 0.37 $\pm$  0.09& 0.17 $\pm$  0.04 &  0.85 $\pm$  0.02 & 0.23 $\pm$  0.03& 0.08 $\pm$  0.08 &  0.80 $\pm$  0.03 & \textbf{0.21 $\pm$  0.04} \\ \hline
K-Median Cor. & 0.04 $\pm$  0.04 &  0.69 $\pm$  0.08 & 0.31 $\pm$  0.08& \textbf{0.01 $\pm$  0.01} &  0.60 $\pm$  0.09 & 0.40 $\pm$  0.09& 0.23 $\pm$  0.07 &  0.85 $\pm$  0.02 & 0.27 $\pm$  0.06& 0.10 $\pm$  0.05 &  0.79 $\pm$  0.02 & 0.23 $\pm$  0.03 \\ \hline
Uniform Subsampling & 0.04 $\pm$  0.03 &  0.68 $\pm$  0.10 & 0.32 $\pm$  0.10& 0.03 $\pm$  0.05 &  \textbf{0.68 $\pm$  0.12} & \textbf{0.33 $\pm$  0.12}& 0.29 $\pm$  0.05 &  \textbf{0.88 $\pm$  0.01} & 0.31 $\pm$  0.05& 0.06 $\pm$  0.05 &  0.79 $\pm$  0.05 & 0.22 $\pm$  0.05 \\ \hline
K-Means & 0.06 $\pm$  0.03 &  0.66 $\pm$  0.09 & 0.34 $\pm$  0.09& 0.02 $\pm$  0.04 &  0.53 $\pm$  0.04 & 0.47 $\pm$  0.04& 0.23 $\pm$  0.05 &  \textbf{0.88 $\pm$  0.01} & 0.26 $\pm$  0.04& 0.09 $\pm$  0.03 & \textbf{0.82 $\pm$  0.01} & \textbf{0.21 $\pm$  0.02} \\ \hline
K-Medoids & 0.05 $\pm$  0.03 &  0.70 $\pm$  0.12 & 0.30 $\pm$  0.11& 0.01 $\pm$  0.02 &  0.61 $\pm$  0.08 & 0.39 $\pm$  0.08& 0.12 $\pm$  0.05 &  0.85 $\pm$  0.04 & \textbf{0.20 $\pm$  0.04}& 0.11 $\pm$  0.05 &  0.81 $\pm$  0.02 & 0.22 $\pm$  0.03 \\ \hline
\end{tabular}
}
\label{tab: supp-learning-comp-reweighing}
\end{table}

\paragraph{Fairness-utility tradeoff when using coresets for data augmentation} We also evaluate the performance of \methodname\ in reducing the downstream demographic disparity when doing data augmentation, i.e., adding the synthetic representatives to the training data when training a downstream model. We use the data augmentation scheme adopted by \cite[Section 2.1]{sharma2020data}, which first uses k-means on the original dataset and then sorts the synthetic representatives based on the distance of each synthetic representative to the nearest k-mean centroid with the same combination of protected attribute and outcome $D$ and $Y$. We generate a set of synthetic datasets of size equal to $50\%$ of the original dataset and look at the fairness-utility of a downstream model trained augmented with such synthetic representative in increments of $5\%$ ($20\%$ and $2.5\%$ respectively for the Adult dataset, given the large dataset size). Figures~\ref{fig: supp-aug-fwc-std} and \ref{fig: supp-aug-allmethods} show the fairness-utility tradeoff of the downstream MLP classifier when doing data augmentation, selecting the best model across various degrees of data augmentation, along with the performance of the baseline MLP classifier with no data augmentation (averages and standard deviations over 10 runs). Table~\ref{tab: supp-learning-comp-numbers-aug} shows the numerical values for the downstream fairness-utility tradeoff, including the tradeoff value computed as in Equation (\ref{eq: tradeoff}). In all datasets the data augmentation with \methodname\ seems to either increase the performance or reduce the demographic disparity, with the only exception being the Drug dataset. Upon further investigation, Figure~\ref{fig: supp-drug-dataset} shows this effect does not appear if the protected attribute $D$ (gender) is not included in the features used to train the downstream MLP classifier. This phenomenon indicates the protected attribute provides a strong predictive effect on the outcome (whether the individual has tried cannabis or not), which might potentially be mediated by unmeasured confounders, i.e., other features regarding the recorded individuals that are not available in the Drug dataset. This would require either in-training or post-processing fairness approaches to be alleviated; see \cite{hort2022BiasMitigationMachinea} for a comprehensive review on different potential approaches. Finally, as in Figure~\ref{fig: supp-all-results-standard-deviations}, the standard deviations for Adult and Credit dataset are large due to the downstream model becoming trivial.

\begin{figure}[!htpb]
    \centering
    \includegraphics[width=1\linewidth]{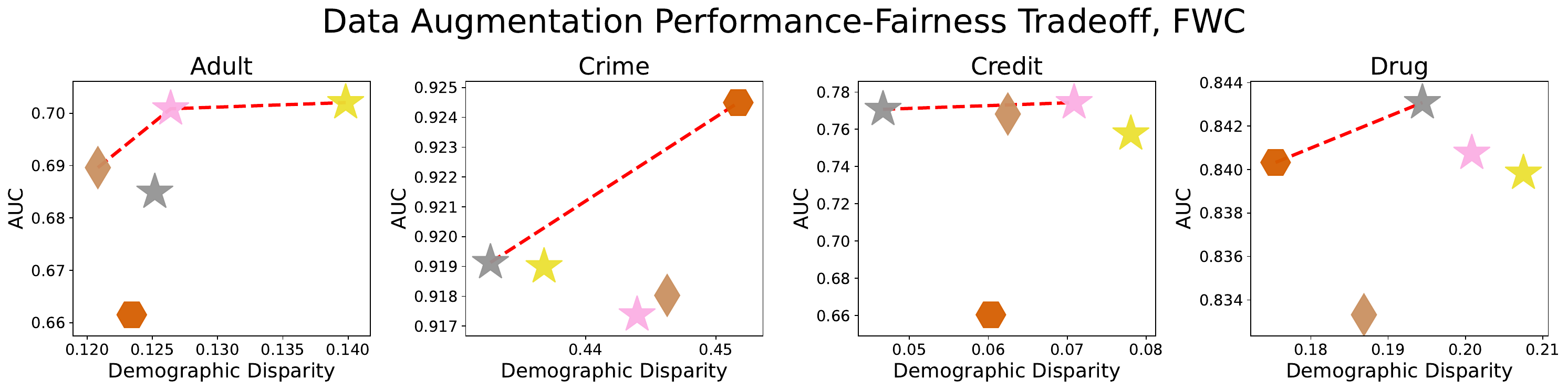}
    \includegraphics[width=1\linewidth]{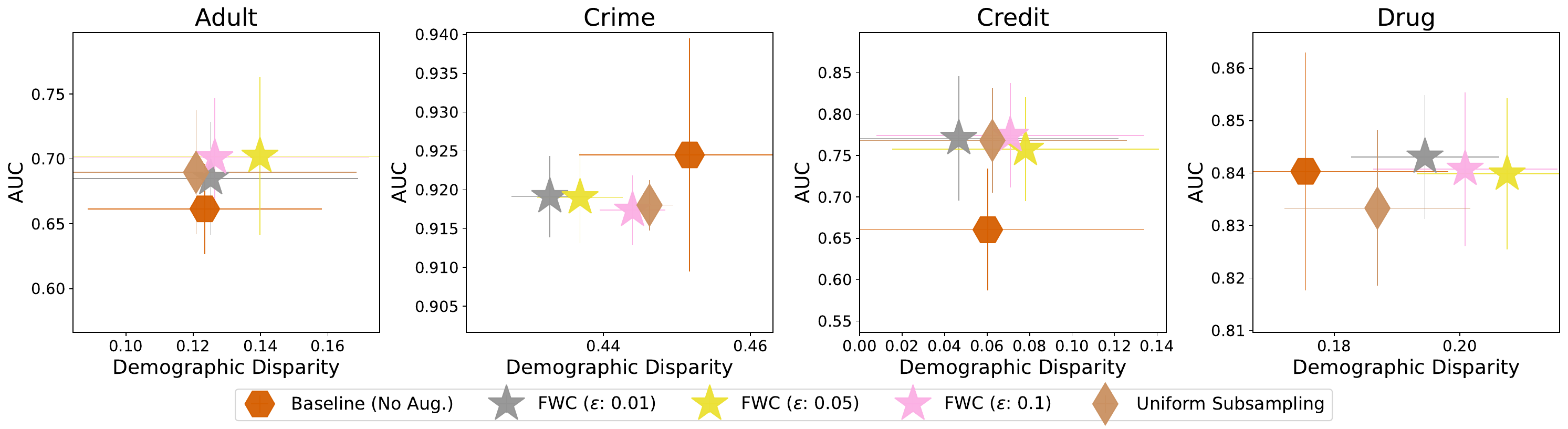}
    \vspace{-0.5cm}
    \caption{Fairness-utility tradeoff of downstream MLP classifier trained using the original training set augmented with coresets representatives, following the augmentation strategy from \cite{sharma2020data}. Each point shows the best model in terms of fairness-utility tradeoff over various degrees of data augmentation, in addition to the baseline model with no augmentation. Means and standard deviations taken over 10 runs, with the computed Pareto frontier indicated by the dashed red line.}
    \label{fig: supp-aug-fwc-std}
\end{figure}

\begin{figure}[!htpb]
    \centering
    \includegraphics[width=1\linewidth]{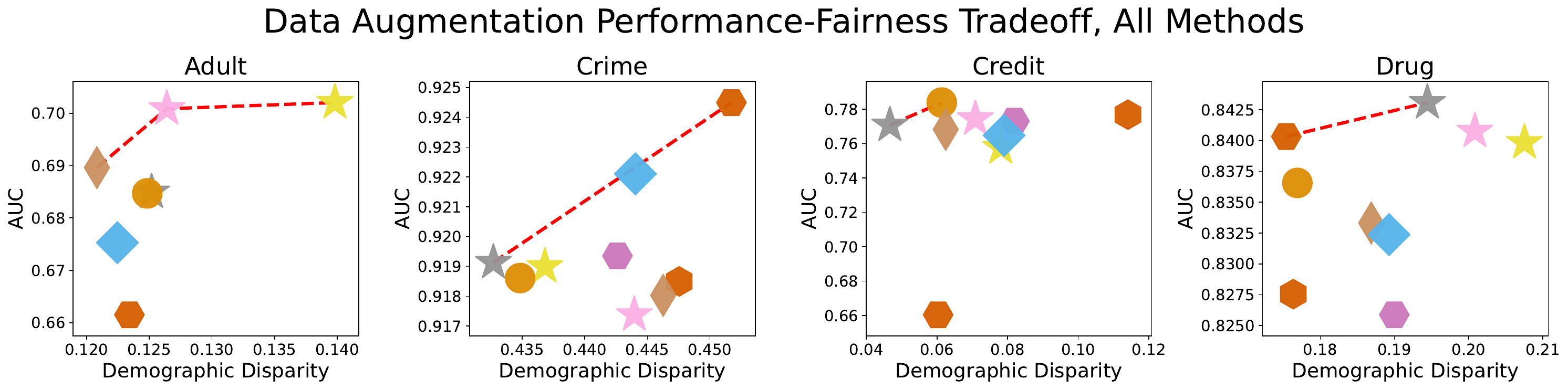}
    \includegraphics[width=1\linewidth]{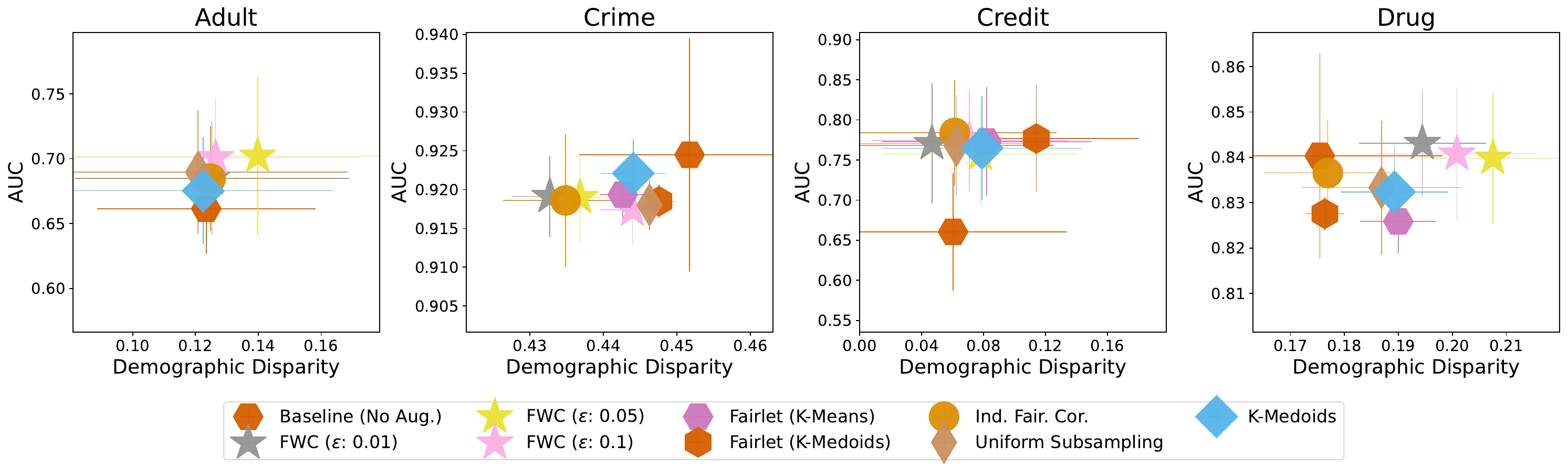}
    \vspace{-0.5cm}
    \caption{Fairness-utility tradeoff of downstream MLP classifier trained using the original training set augmented with coresets representatives, following the augmentation strategy from \cite{sharma2020data}, including all methods mentioned in Section~\ref{sec: experiments}. Each point shows the best model in terms of fairness-utility tradeoff over various degrees of data augmentation, in addition to the baseline model with no augmentation. Averages and standard deviations computed over 10 runs, with the top panel showing just means and the bottom panel combining both means and standard deviations, with the computed Pareto frontier indicated by the dashed red line.}
    \label{fig: supp-aug-allmethods}
\end{figure}

\begin{figure}[!htpb]
    \centering
    \includegraphics[width=1\linewidth]{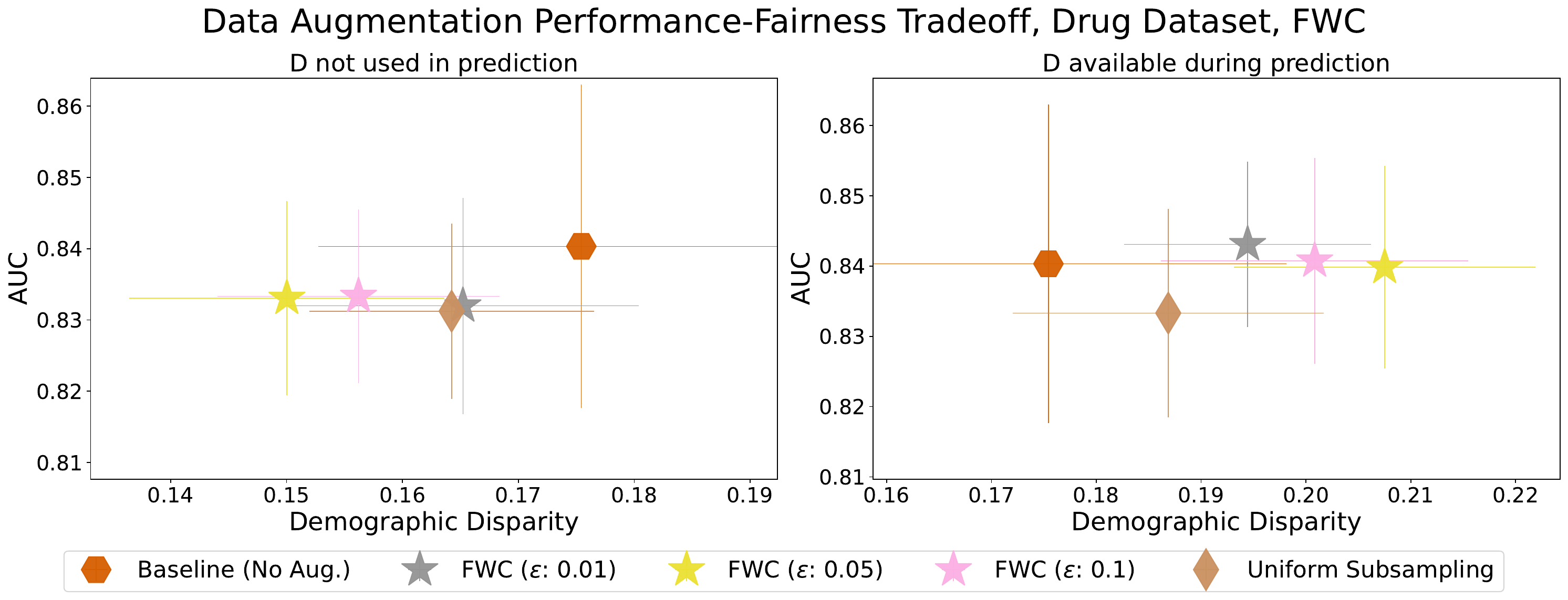}
    \vspace{-0.2cm}
    \caption{Data augmentation fairness-utility tradeoff of downstream MLP classifier for the Drug dataset when the protected attribute $D$ (gender) is either not included (left) or included (right) as feature in the learning process. As in Figure~\ref{fig: supp-aug-fwc-std}, the best model across various data augmentation degrees is reported, with averages and standard deviations obtained over 10 runs. \methodname\ manages to successfully reduce the demographic disparity when gender is not used as a feature, but fail to do so when gender is used, indicating that gender provides strong predictive power for the outcome in question, which would require enforcing fairness either during model training or by post-processing the outputs.}
    \label{fig: supp-drug-dataset}
\end{figure}

{\renewcommand{\arraystretch}{1.5}
\begin{table}[!ht]
\centering
\caption{Demographic disparity (Equation (\ref{eq: dd})), AUC and fairness-utility tradeoff (Equation (\ref{eq: tradeoff})) of downstream MLP classifier trained via data augmentation with all fair coresets/clustering methods, across the four real datasets. The best methods across various degrees of data augmentation is shown. The best performing methods for every column is bolded, with averages and standard deviations taken over 10 runs.}
\resizebox{1\textwidth}{!}{
\begin{tabular}{|c||c|c|c||c|c|c||c|c|c||c|c|c||}
\hline
\textbf{Method} & \multicolumn{3}{c||}{\textit{Adult Dataset}}  & \multicolumn{3}{c||}{\textit{Credit Dataset}} & \multicolumn{3}{c||}{\textit{Crime Dataset}} & \multicolumn{3}{c||}{\textit{Drug Dataset}} \\ \hline
& DD ($\downarrow$) & AUC ($\uparrow$) & Tradeoff ($\downarrow$) & DD ($\downarrow$) & AUC ($\uparrow$) & Tradeoff ($\downarrow$) & DD ($\downarrow$) & AUC ($\uparrow$) & Tradeoff ($\downarrow$) & DD ($\downarrow$) & AUC ($\uparrow$) & Tradeoff ($\downarrow$) \\ \hline \hline
Baseline (No Aug.) & \textbf{0.12 $\pm$  0.03} &  0.66 $\pm$  0.05 & 0.36 $\pm$  0.05& 0.06 $\pm$  0.06 &  0.66 $\pm$  0.11 & 0.34 $\pm$  0.11& 0.45 $\pm$  0.05 &  \textbf{0.924 $\pm$  0.015} & 0.46 $\pm$  0.04& \textbf{0.18 $\pm$  0.03} &  \textbf{0.84 $\pm$  0.02} & \textbf{0.24 $\pm$  0.03} \\ \hline
FWC ($\epsilon$: 0.01) & 0.13 $\pm$  0.03 &  0.68 $\pm$  0.07 & 0.34 $\pm$  0.06& \textbf{0.05 $\pm$  0.03} &  0.77 $\pm$  0.11 & \textbf{0.23 $\pm$  0.11}& \textbf{0.43 $\pm$  0.04} &  0.919 $\pm$  0.005 & \textbf{0.44 $\pm$  0.04}& 0.19 $\pm$  0.03 &  \textbf{0.84 $\pm$  0.01} & 0.25 $\pm$  0.02 \\ \hline
FWC ($\epsilon$: 0.05) & 0.14 $\pm$  0.03 &  \textbf{0.70 $\pm$  0.09} & 0.33 $\pm$  0.08& 0.08 $\pm$  0.05 &  0.76 $\pm$  0.09 & 0.25 $\pm$  0.09& 0.44 $\pm$  0.03 &  0.919 $\pm$  0.006 & \textbf{0.44 $\pm$  0.03}& 0.21 $\pm$  0.02 &  \textbf{0.84 $\pm$  0.01} & 0.26 $\pm$  0.02 \\ \hline
FWC ($\epsilon$: 0.1) & 0.13 $\pm$  0.02 &  \textbf{0.70 $\pm$  0.07} & \textbf{0.32 $\pm$  0.06} & 0.07 $\pm$ 0.05 &  0.77 $\pm$  0.09 & 0.24 $\pm$  0.09& 0.44 $\pm$  0.03 &  0.917 $\pm$  0.004 & 0.45 $\pm$  0.03& 0.20 $\pm$  0.03 &  \textbf{0.84 $\pm$  0.01} & 0.26 $\pm$  0.02 \\ \hline
Fairlet (K-Means) & - & - & - & 0.08 $\pm$  0.06 &  0.77 $\pm$  0.10 & 0.24 $\pm$  0.10& 0.44 $\pm$  0.03 &  0.919 $\pm$  0.003 & 0.45 $\pm$  0.03& 0.19 $\pm$  0.04 &  0.83 $\pm$  0.01 & 0.26 $\pm$  0.03 \\ \hline
Fairlet (K-Medoids) & - & - & - & 0.11 $\pm$  0.07 &  \textbf{0.78 $\pm$  0.10} & 0.25 $\pm$  0.09& 0.45 $\pm$  0.04 &  0.918 $\pm$  0.002 & 0.45 $\pm$  0.04& \textbf{0.18 $\pm$  0.05} &  0.83 $\pm$  0.00 & 0.25 $\pm$  0.04 \\ \hline
Ind. Fair. Cor. & \textbf{0.12 $\pm$  0.03} &  0.68 $\pm$  0.06 & 0.34 $\pm$  0.06& 0.06 $\pm$  0.04 &  \textbf{0.78 $\pm$  0.10} & 0.22 $\pm$  0.10& \textbf{0.43 $\pm$  0.04} &  0.919 $\pm$  0.009 & \textbf{0.44 $\pm$  0.04}& \textbf{0.18 $\pm$  0.02} &  \textbf{0.84 $\pm$  0.01} & \textbf{0.24 $\pm$  0.02} \\ \hline
Uniform Subsampling & \textbf{0.12 $\pm$  0.02} &  0.69 $\pm$  0.07 & 0.33 $\pm$  0.07& 0.06 $\pm$  0.05 &  0.77 $\pm$  0.09 & 0.24 $\pm$  0.09& 0.45 $\pm$  0.04 &  0.918 $\pm$  0.003 & 0.45 $\pm$  0.04& 0.19 $\pm$  0.03 &  0.83 $\pm$  0.01 & 0.25 $\pm$  0.02 \\ \hline
K-Medoids & \textbf{0.12 $\pm$  0.04} &  0.68 $\pm$  0.06 & 0.35 $\pm$  0.06& 0.08 $\pm$  0.05 &  0.76 $\pm$  0.10 & 0.25 $\pm$  0.09& 0.44 $\pm$  0.03 &  0.922 $\pm$  0.004 & 0.45 $\pm$  0.03& 0.19 $\pm$  0.02 &  0.83 $\pm$  0.01 & 0.25 $\pm$  0.02 \\ \hline
\end{tabular}
}
\label{tab: supp-learning-comp-numbers-aug}
\end{table}

\newpage

\paragraph{Using \methodname\ to correct biases in LLMs} 

To query GPT models, test data with 200 samples is provided with a base parity of 0.5 (similar to \cite{wang2023decodingtrust}) and additionally, examples are provided in the case of the few shot settings. Data is fed in as text. One of the tabular data, for example, is: ``\textit{A person in 1996 has the following attributes: Age: $21.0$, workclass: Private, education:  Some-college, highest education level: $10.0$, marital status: Never-married, occupation: Other-service, relationship: Own-child, race: White, sex: Female, capital gain: $0.0$, capital loss: $0.0$, hours per week: $25.0$, native country: United-States}". 

We use the following prompts:

\begin{itemize}
    \item \textbf{Zero shot}: ``\textit{Using the provided data, will this person from 1996 be hired at greater than 50,000 USD per year? You must only respond with the word 'yes' or 'no'. Here are 0 examples with the correct answer.}"
    \item \textbf{Few shot}: ``\textit{Given the provided data, will this person from 1996 be hired at greater than 50,000 USD per year? You must only respond with the word 'yes' or 'no'. Here are $n$ examples with the correct income level for a person in 1996. Make sure you use the examples as a reference}", where $n$ is the number of demographically balanced samples.
    \item \textbf{Few shot} (\methodname): ``\textit{Given the provided data, will this from 1996 be hired at greater than 50,000 USD per year? You must only respond with the word 'yes' or 'no'.
        Here are $n$ examples with the correct income level for a person in 1996, along with weights in the column weight. Weights are between 0 (minimum) and 1 (maximum). The more the weight, the more important the example is. Make sure you use the examples as a reference.}"
\end{itemize}

In the few shot settings, 16 examples are provided in both cases due to the LLM token limitation; passing fewer examples yields similar results to the ones in Table~\ref{tab:gpt}. For~\methodname, we run a separate coreset generation run (differently from other experiments in Section~\ref{sec: experiments}), where we selected $m=16$ and ensured that the positive class ($Y=1$) has an equal number of male and female samples. We also note that while the results from GPT-4 for the zero shot and few shot cases are similar to what was observed by \cite{wang2023decodingtrust}, the accuracies reported for the GPT-3.5 Turbo model appear to be lower in our experiments, which points to a potential difference in the exact backend LLM model used for inference.


\section{Limitations of \methodname} \label{app: limitation}


\textbf{Coreset support and non-convex feature spaces} \methodname\ representative do not need to be within the original $n$ samples, but they do need to share the same support. As shown in Section~\ref{sec: update-feature-vec}, solving line 5 in Algorithm~\ref{alg MM method} in cases 1. and 2. means the synthetic representatives could fall outside the original dataset. In case 3., the solution has to be selected from the data points already existing in the original dataset (akin to k-medoids). In general, non-convex feature spaces $\mathcal{X}$ might represent a challenge, as representatives might be generated in zero-density regions (a simple example could be a dataset distributed as a hollow circle or two moons). However, this criticism is also more generally applicable to the existing fair coresets/clustering literature, as well as the k-means algorithms, for which specific adjustments have been developed \cite{raykov2016k} and could be indeed extended to \methodname\ . \\

\textbf{Limiting the total number of iterations} \methodname\ is not of polynomial time complexity and the total number of iterations might grow faster than linear time when the dataset size is very large, as shown in the synthetic data experiment in Section~\ref{supp: synthetic-dataset} and Table~\ref{tab: supp-runtime}. This phenomenon is shared with k-means, which is also not of polynomial time complexity and is known to potentially take an exponentially large number of iterations to terminate \cite{vattani2009k}. As mentioned in Section~\ref{supp: synthetic-dataset}, a common practice for clustering is set a fixed number of maximum iterations, after which the algorithm is stopped. \\

\textbf{Computational bottlenecks} The main complexity term for \methodname\ is $\mathcal{O}(mn)$, which comes from establishing the cost matrix in the beginning of the solution of problem (\ref{pro define FC}). This complexity is comparable with what occurs in Lloyd's algorithm for k-means and k-medians. This might be problematic if the cost matrix is too large to be stored directly in memory. In practice, we do not actually need to store the entire matrix, as we only need to compute the largest component for each row of $C$ for solving problem (\ref{pro define FC}) (see Lemma~\ref{lm minimizer P}), so one could further improve the cost of storing the cost matrix. Another option would also be leverage the same approaches used for k-means such as, e.g., cost matrix sketching \cite{yin2022randomized}. Finally, \methodname\ would also benefit from GPU implementations akin to k-means and k-medians, which would substantially accelerate the runtime speed of \methodname\ . \\

\textbf{Connection between $\epsilon$ and downstream learning} In our definition of demographic parity in Equation ~\ref{type1}, the hyper-parameter $\epsilon$ effectively controls how different the outcome rates across sensitive feature groups $D$ of the weighted coreset distribution $p_{\hat{Z}, \theta}$ can be from the overall outcome rates in the original dataset $p_{Y_T}$. In our experiments (Section~\ref{sec: experiments}), we empirically show that limiting the fairness violation in the coresets results in a fairer downstream model. However, when training a downstream model using \methodname\ we induce a distribution shift between the train set and the test set, as the coresets distribution is never identical to the original dataset distribution. Although we have provided some results on the generalization properties of \methodname\ (Proposition~\ref{prop: generalization}) as well as some intuition about developing downstream learning bounds for \methodname\ in non-i.i.d. settings (Section~\ref{app: hardness-learning}), theoretically characterize the connection between the fairness violation parameter $\epsilon$ remains an open question. In essence, the analysis is challenging as the induced distribution shift is dependent on the biases in the original dataset distribution, the coreset size $m$, the metric chosen for the cost matrix and, ultimately, the fairness violation parameter $\epsilon$. For this reason, although we have shown that restricting the fairness violation improves downstream models fairness, an explicit characterization of the downstream effects of $\epsilon$, as well as other hyper-parameters, would require significant further work, beyond the scope of this paper.

\textbf{Implications for other fairness measures} As highlighted by \cite[Chapter 3]{barocas-hardt-narayanan}, fairness notions in classification settings can be categorized into notions of independence, separation, and sufficiency. Demographic parity falls in the class of the independence notion, and hence, other measures of fairness that are closely related, e.g., disparate impact, would also improve when optimizing for demographic parity. However, other notions of fairness such as separation or sufficiency may not simultaneously be satisfied \cite{barocas-hardt-narayanan}. As \methodname\ targets demographic parity, it cannot guarantee an improvement in these other measures.
To test this, we compute the equalized odds, which falls under the notion of separation, for the downstream classifier in Section~\ref{sec: experiments} and check the performance of \methodname\ compared to the other approaches. Table~\ref{tab: eq-odds} indicates the datasets in which \methodname\ is part of the Pareto frontier for both demographic parity and equalized odds. When considering equalized odds, \methodname\ is not a part of the Pareto frontier for the Drug dataset, and more generally, \methodname\ performance is not as competitive. This is in contrast to demographic parity: in Figure~\ref{fig: main-results}, \methodname\ sits on the Pareto frontier across all datasets for fairness-performance trade-off in downstream classification.

\begin{table}[!ht]
\caption{Presence on the Pareto frontier for \methodname\ across different fairness violation hyper-parameter values ($\epsilon = \{0.01, 0.05, 0.1\}$), for both demographic parity (left) and equalized odds (right) in the downstream learning settings of Section~\ref{sec: experiments}. As equalized odds is not an independence notion of fairness as demographic parity, constraints on demographic parity do not guarantee an improvement in equalized odds, resulting in \methodname\ not performing as well for downstream performance-fairness tradeoff when using equalized odds.}
\resizebox{1.0\textwidth}{!}{%
\begin{tabular}{|c||c|c|c||c|c|c|} \hline
& \multicolumn{3}{c|}{\textbf{Pareto Frontier, Demographic Parity}} & \multicolumn{3}{|c|}{\textbf{Pareto Frontier, Equalized Odds}} \\ \hline
\textbf{Dataset} & \methodname\ ($\epsilon=0.01$) & \methodname\ ($\epsilon=0.05$) & \methodname\ ($\epsilon=0.1$) & \methodname\ ($\epsilon=0.01$)             & \methodname\ ($\epsilon=0.05$)             & \methodname\ ($\epsilon=0.1$)            \\ \hline
Adult            & $\checkmark$          & $\checkmark$          & $\xmark$             & $\checkmark$                      & $\xmark$                      & $\xmark$                         \\ \hline
Drug             & $\checkmark$          & $\checkmark$          & $\checkmark$         & $\xmark$                          & $\xmark$                      & $\xmark$                     \\ \hline
Crime            & $\checkmark$          & $\checkmark$          & $\checkmark$         & $\xmark$                          & $\xmark$                          & $\checkmark$                         \\ \hline
Credit           & $\checkmark$          & $\xmark$              & $\xmark$             & $\checkmark$                          & $\xmark$                          & $\checkmark$ \\ \hline                    
\end{tabular}
}
\label{tab: eq-odds}
\end{table}

\section{Broader Impact} \label{app: broader impact}
Our work presents a novel approach to obtain coresets (synthetic representative samples) of a given dataset while reducing biases and disparities in subgroups of the given dataset. As other approaches in the field of algorithmic fairness, our efforts may help populations that would otherwise face disadvantages from a model or decision process. Importantly, our approach refrains from exploiting biases inherent in the data itself; rather, it seeks to mitigate biases in data-driven decision systems. It is crucial to note that our method does not claim to address all sources or types of bias. In addition, while our tools enable a malicious modeler to manipulate algorithmic fairness methods to amplify disparities instead of reducing them, for instance, by reversing the fairness constraint (replacing $\leq$ with $\geq$), the unfairness of a trained model can be detected by assessing it over a separate test set from the original dataset.

\begin{figure}[!htpb]
    \centering
    \includegraphics[width=0.49\linewidth]{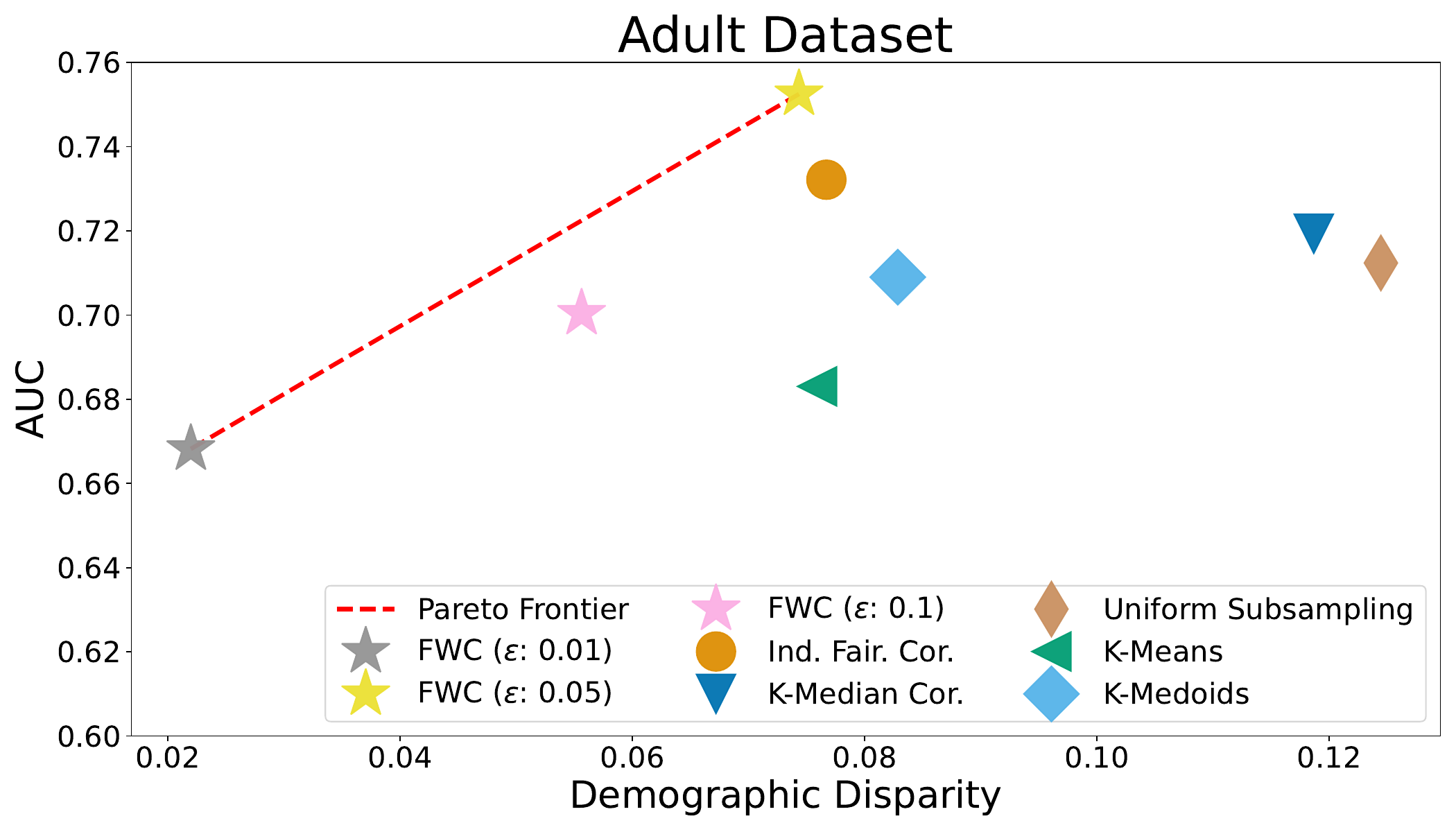}
    \includegraphics[width=0.49\linewidth]{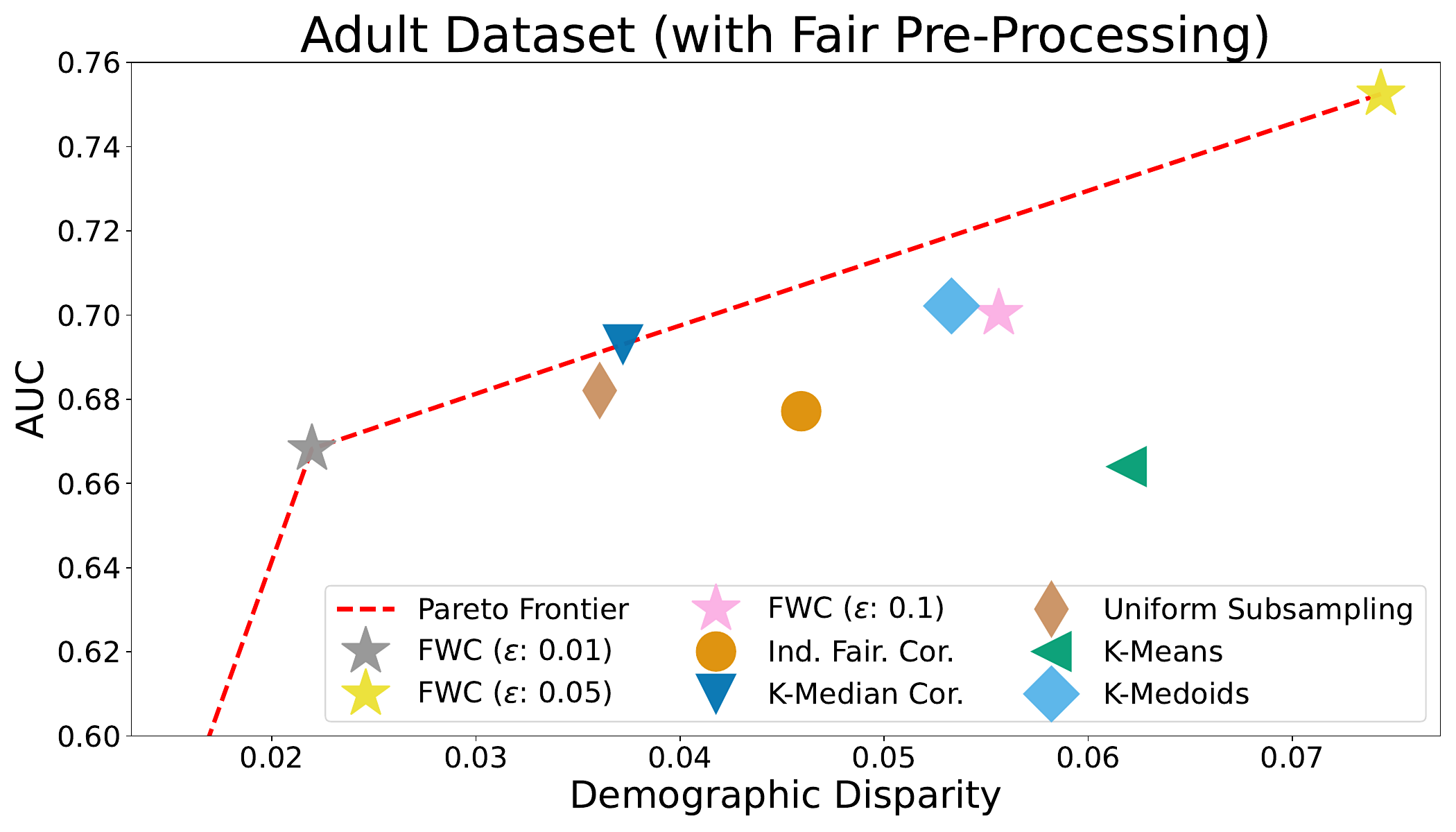}
    \includegraphics[width=0.49\linewidth]{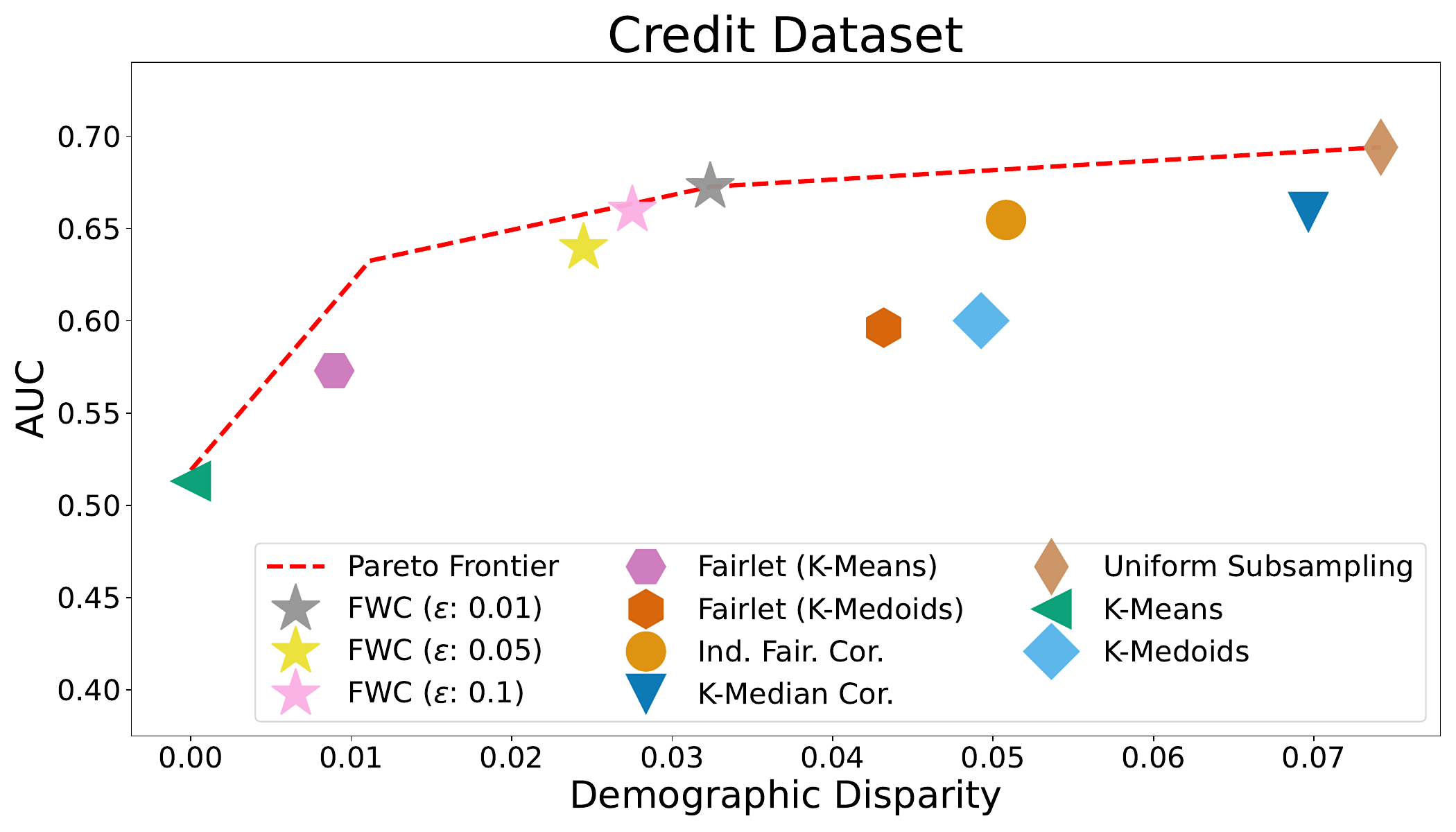}
    \includegraphics[width=0.49\linewidth]{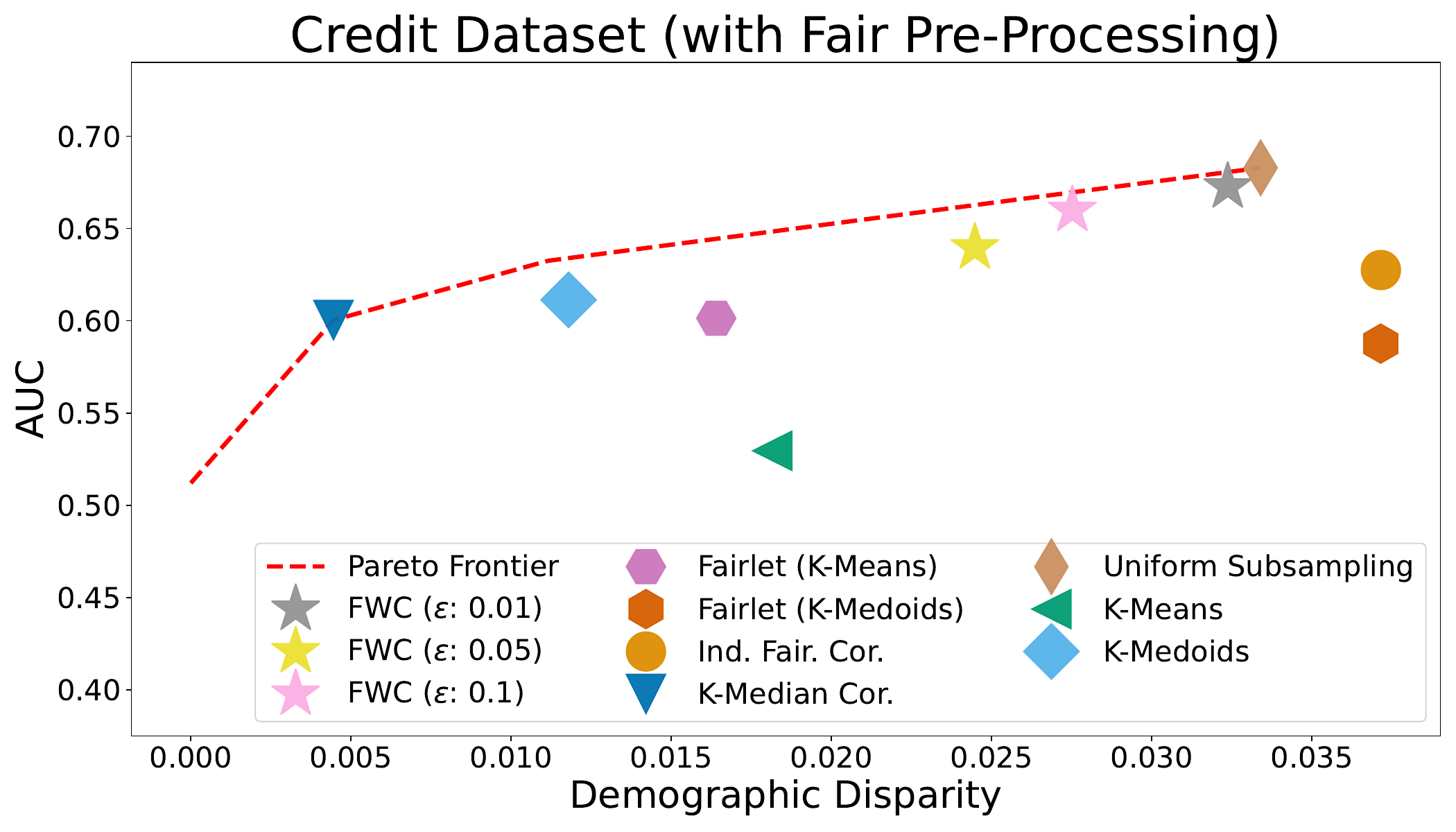}
    \includegraphics[width=0.49\linewidth]{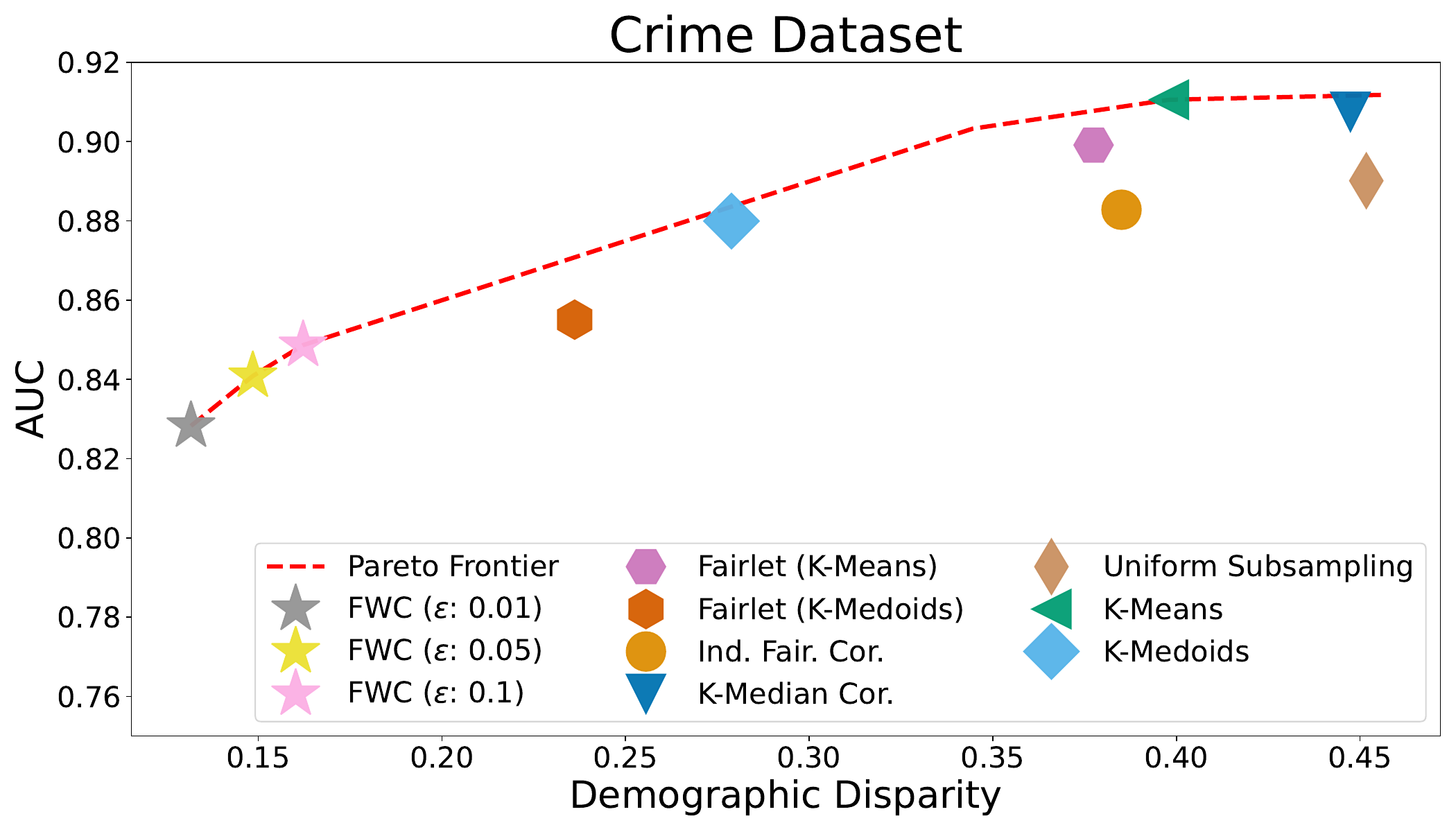}
    \includegraphics[width=0.49\linewidth]{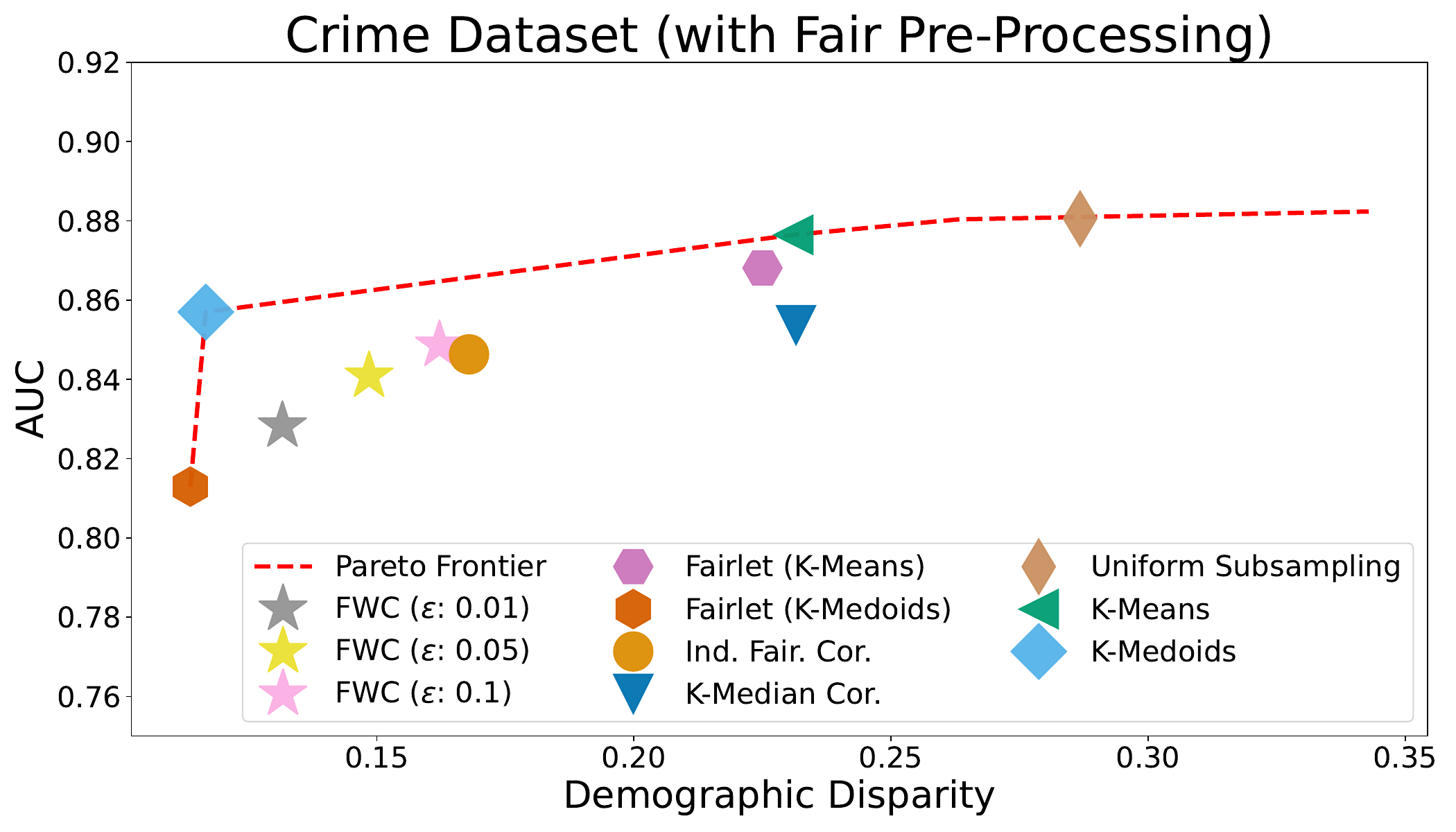}
    \includegraphics[width=0.49\linewidth]{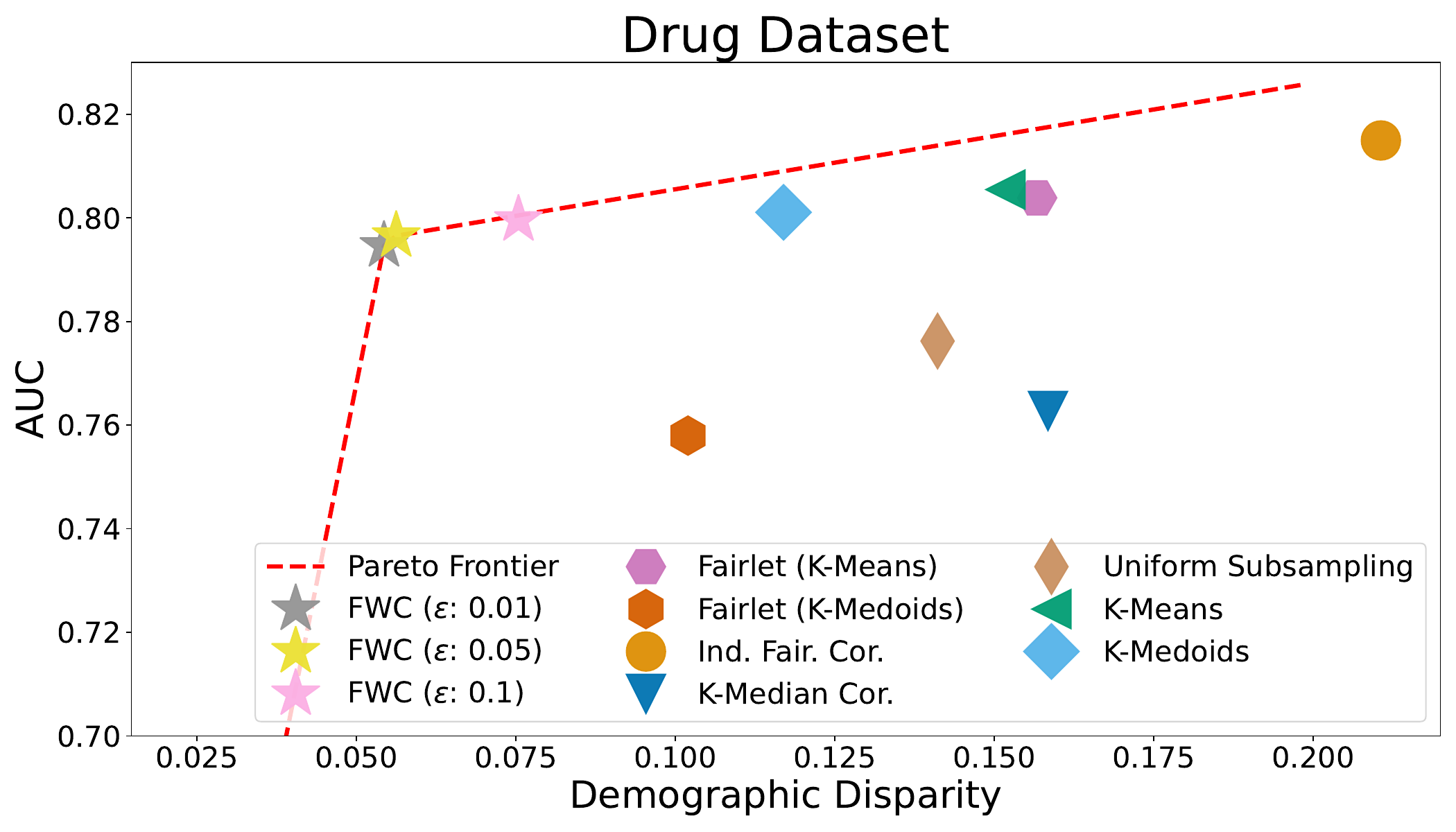}
    \includegraphics[width=0.49\linewidth]{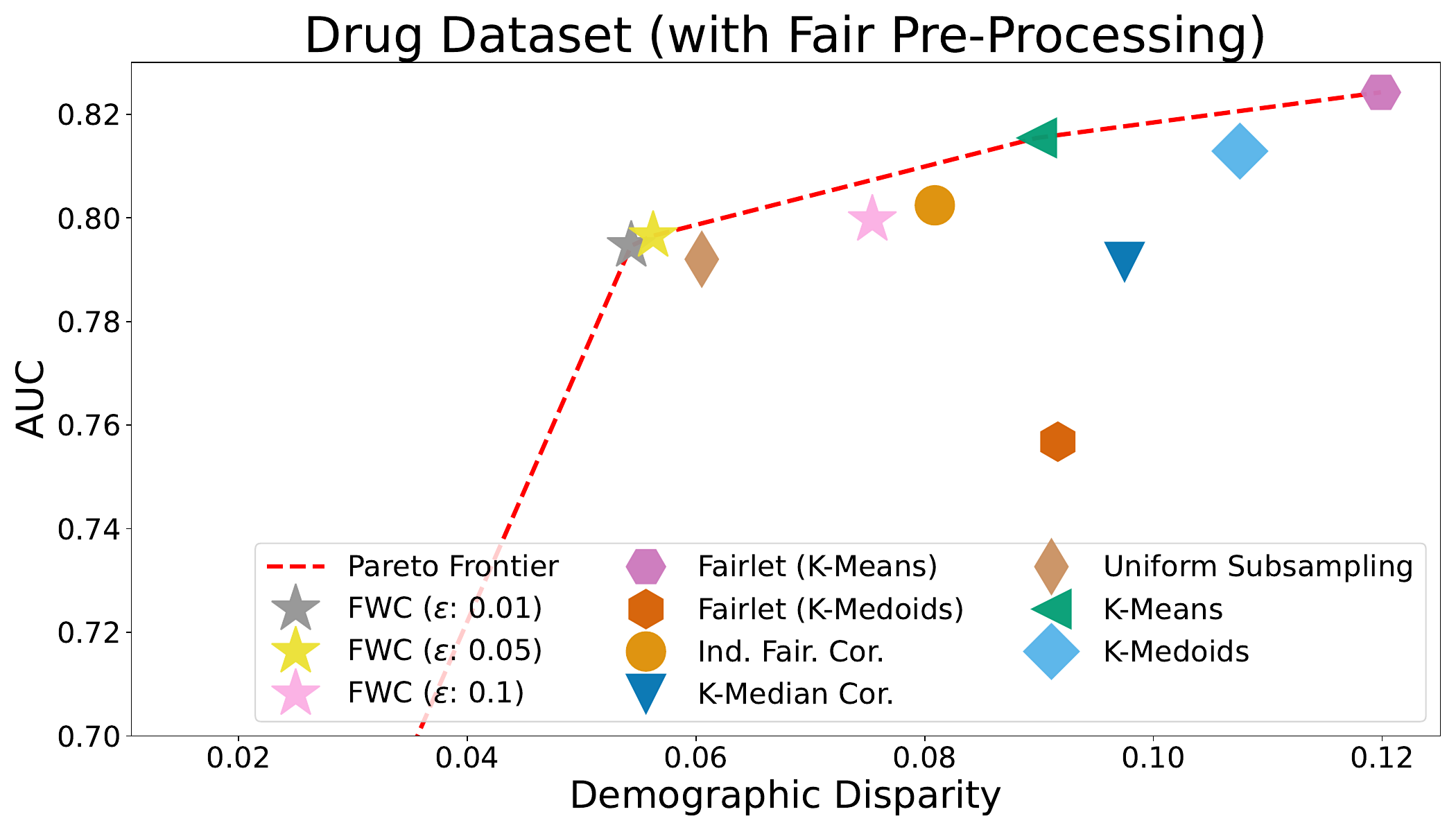}
    \vspace{0.1cm}
    \caption{Fairness-utility tradeoff of all methods, indicated by AUC and demographic disparity of a downstream MLP classifier across all datasets (rows) and without (left column) or with (right column) fair pre-processing \cite{kamiran2012data} (excluding \methodname\ , to which no fairness modification is applied after coresets have been generated). The Pareto frontier, indicated with a dashed red-line, is computed across all models and coreset sizes. We report the averages over 10 separate train/test splits.}
    \label{fig: supp-all-results}
\end{figure}

\begin{figure}[!htpb]
    \centering
    \includegraphics[width=0.49\linewidth]{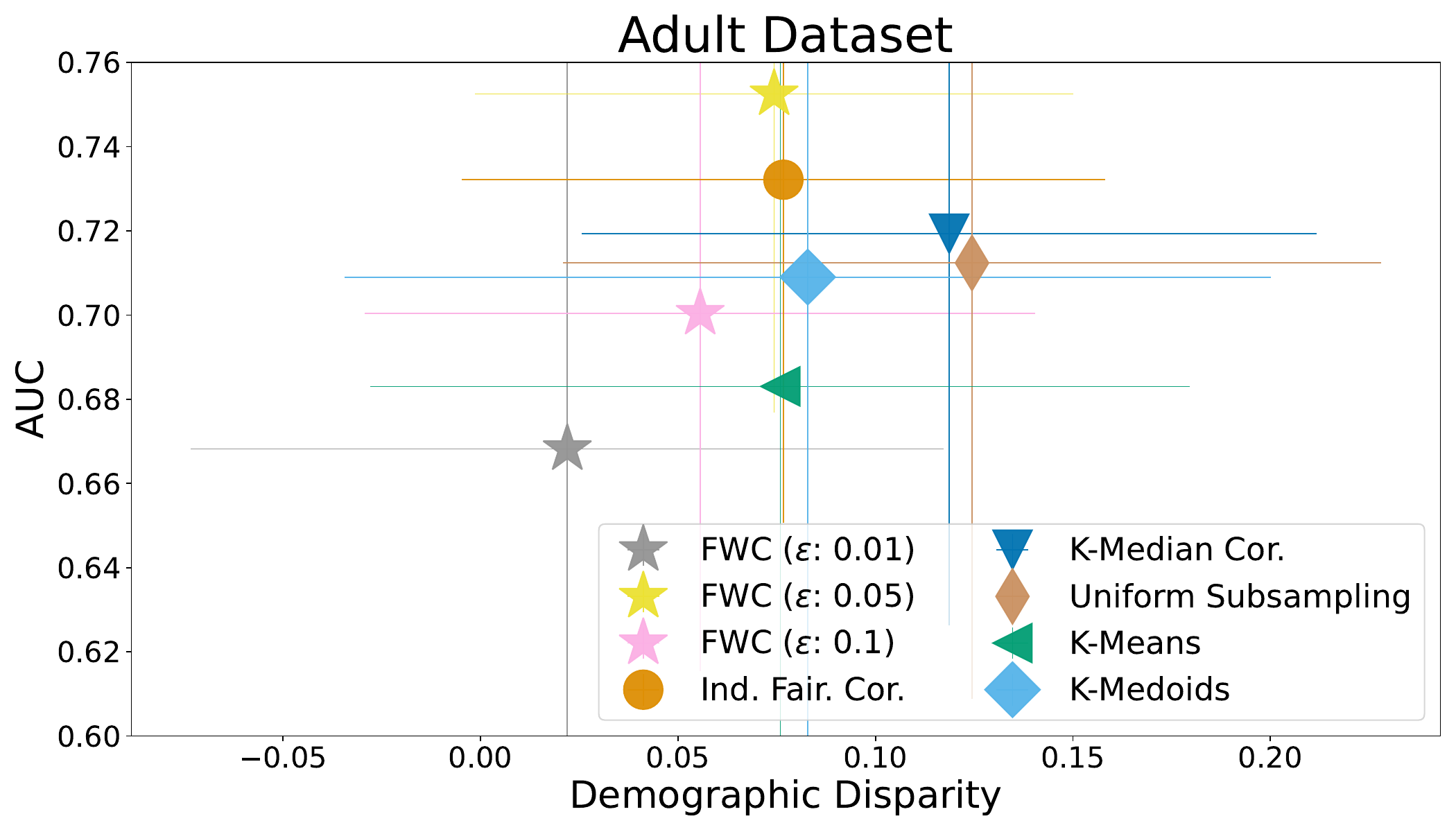}
    \includegraphics[width=0.49\linewidth]{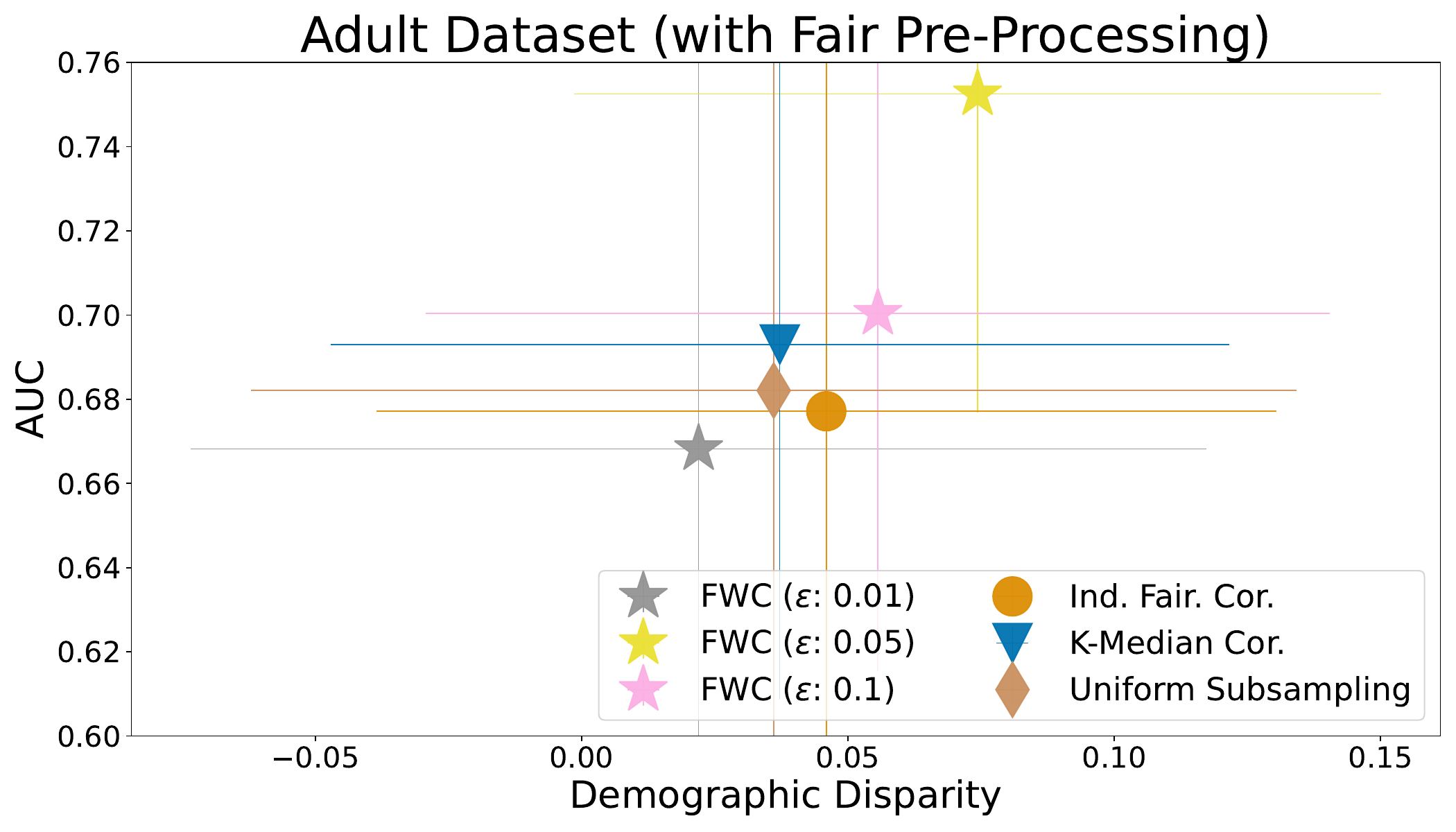}
    \includegraphics[width=0.49\linewidth]{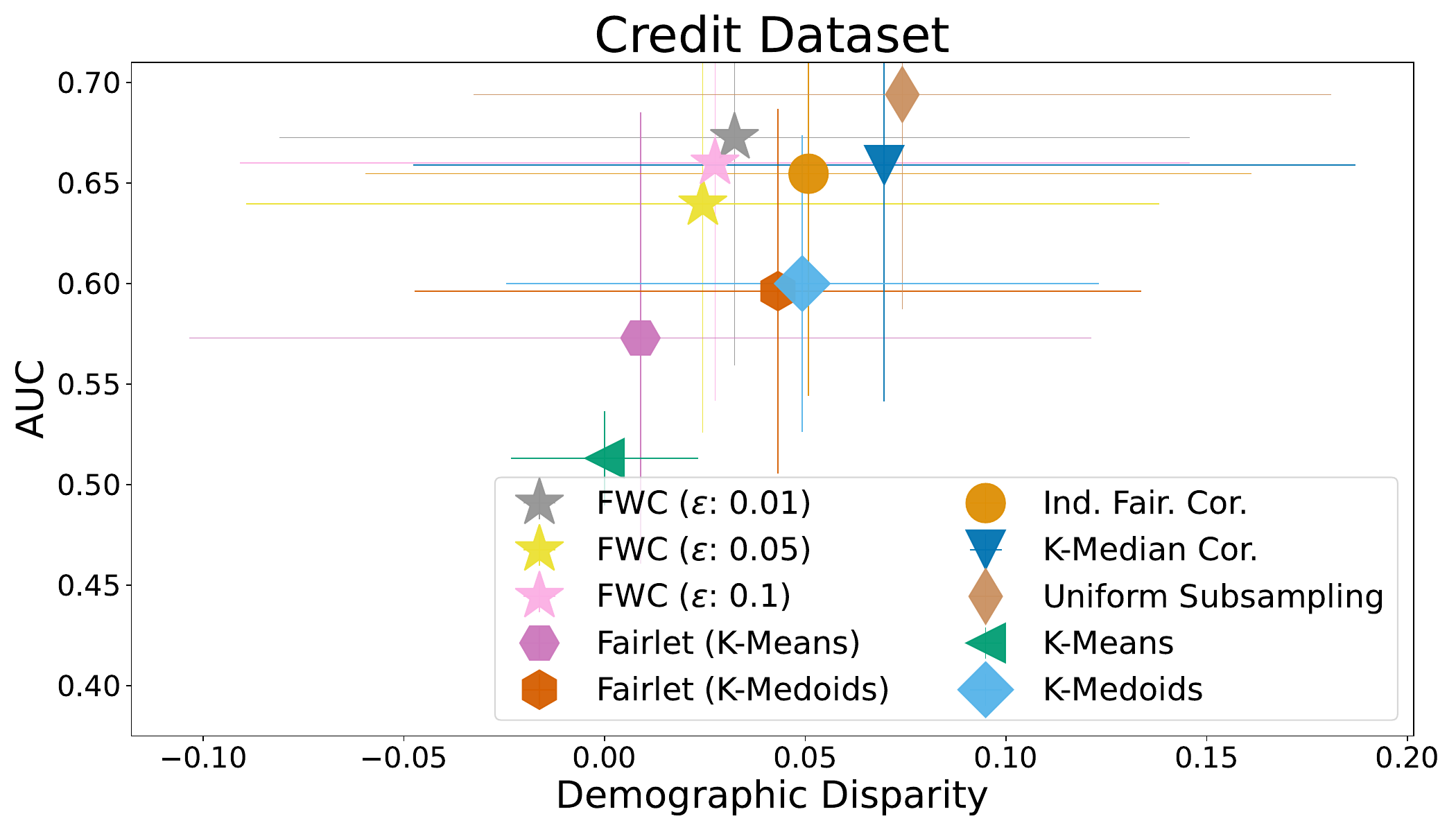}
    \includegraphics[width=0.49\linewidth]{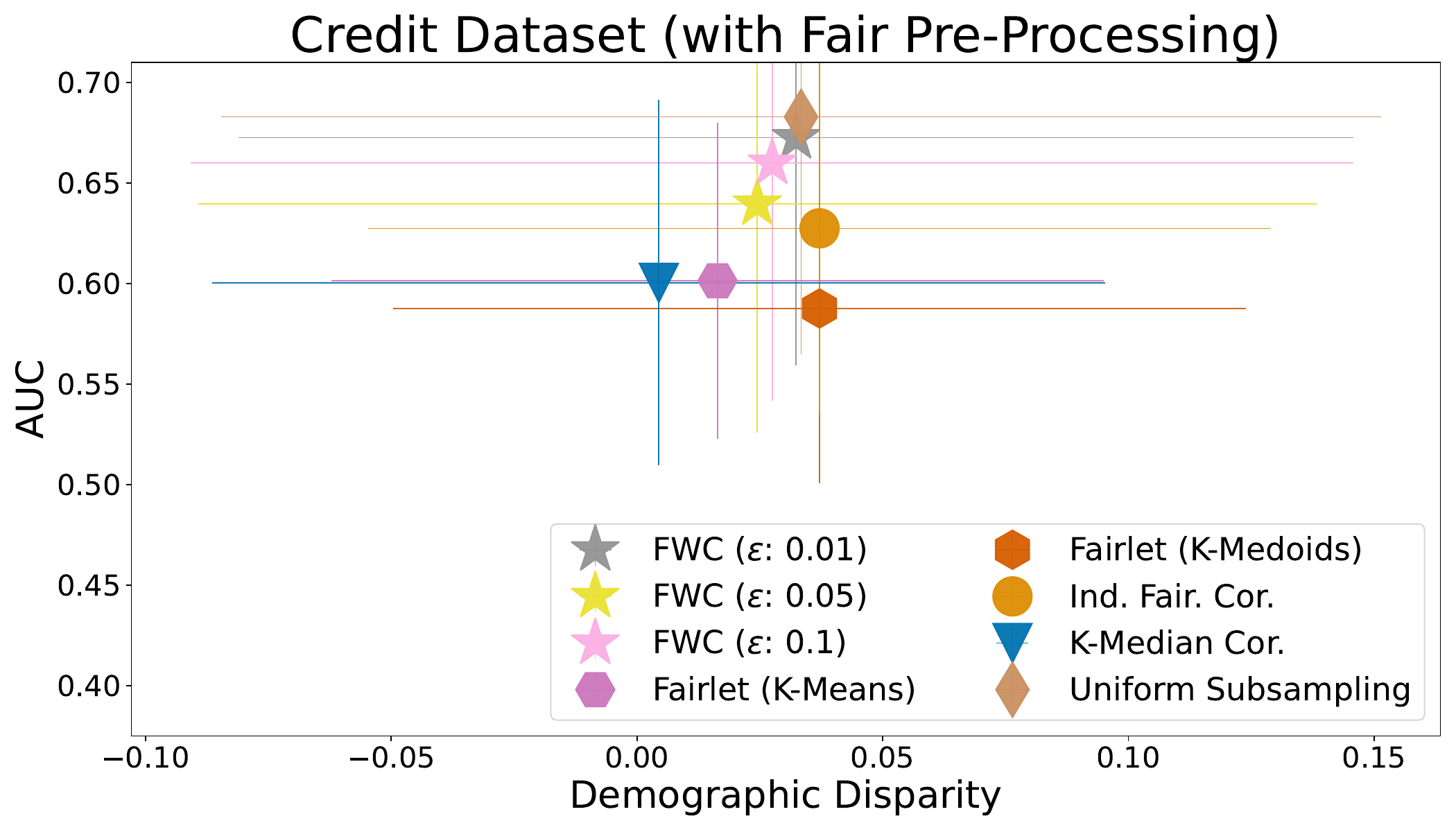}
    \includegraphics[width=0.49\linewidth]{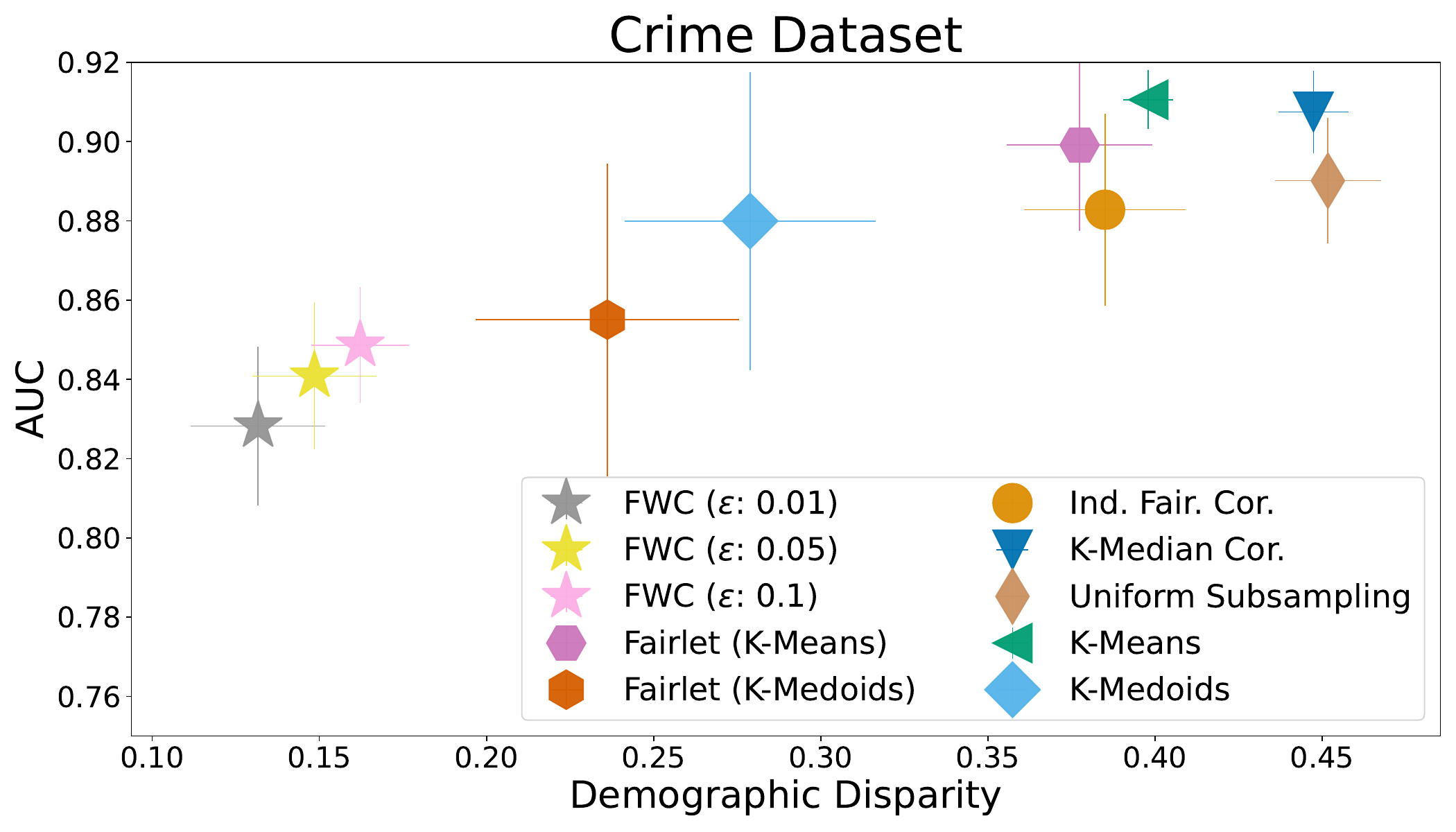}
    \includegraphics[width=0.49\linewidth]{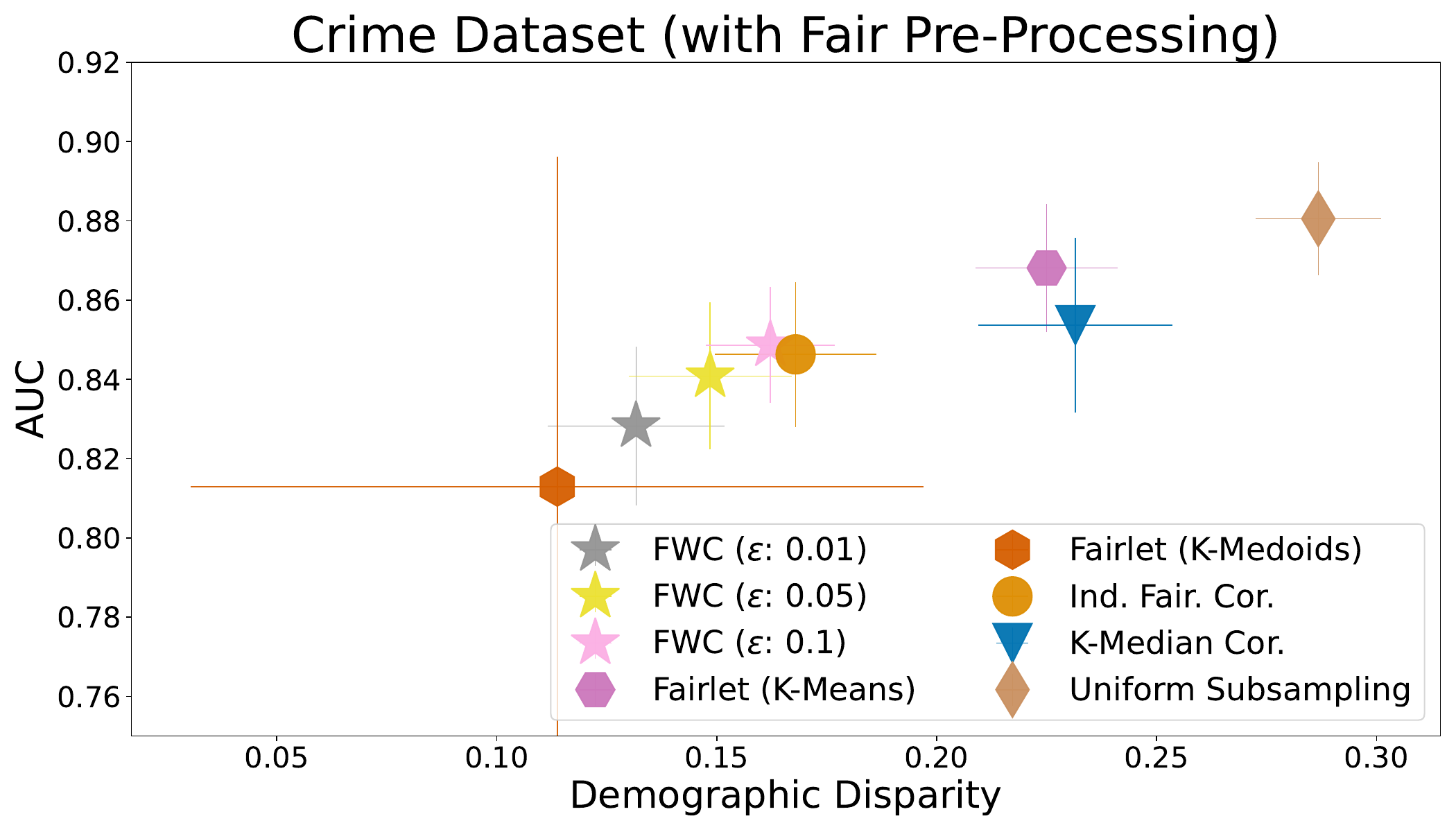}
    \includegraphics[width=0.49\linewidth]{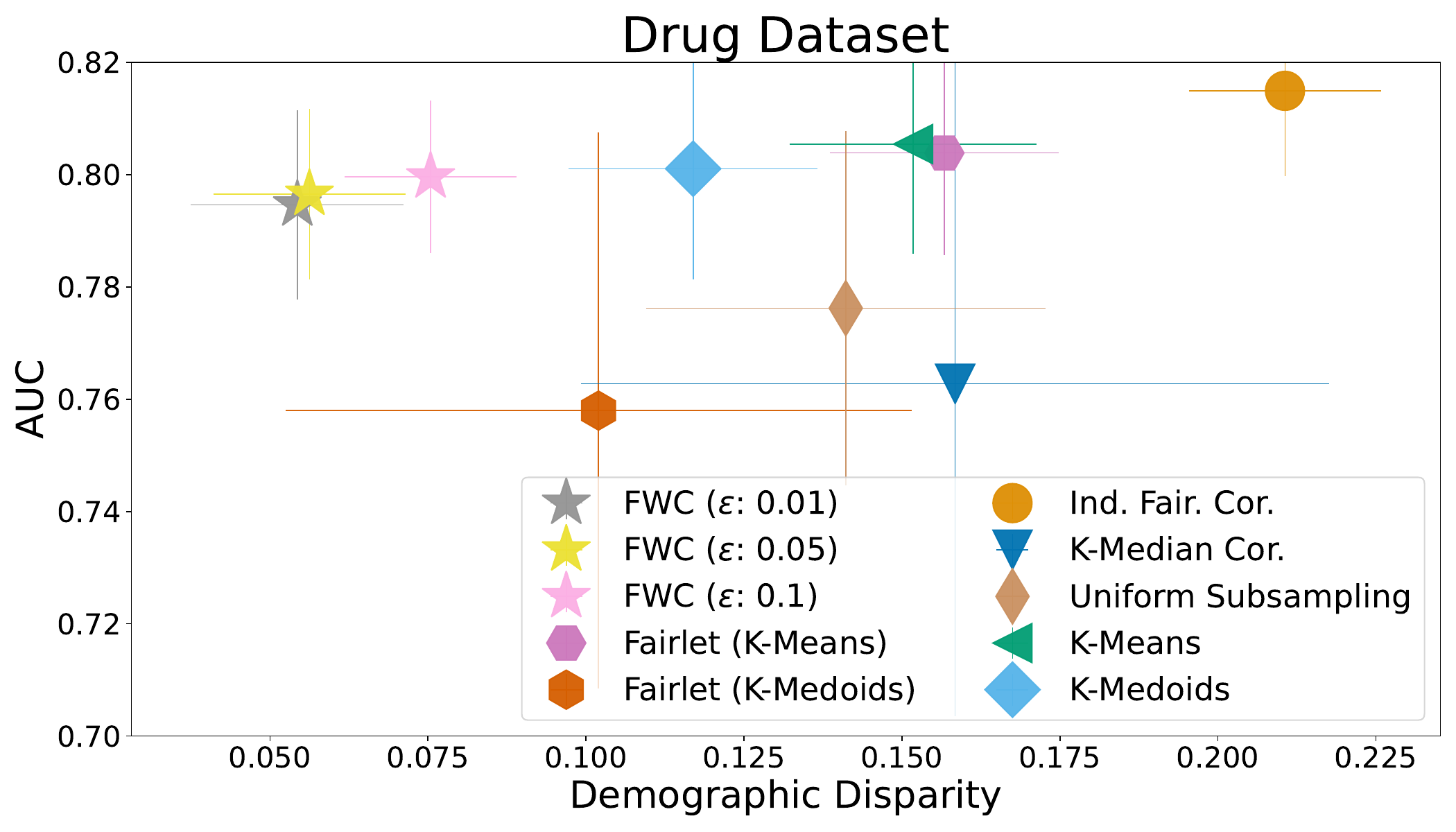}
    \includegraphics[width=0.49\linewidth]{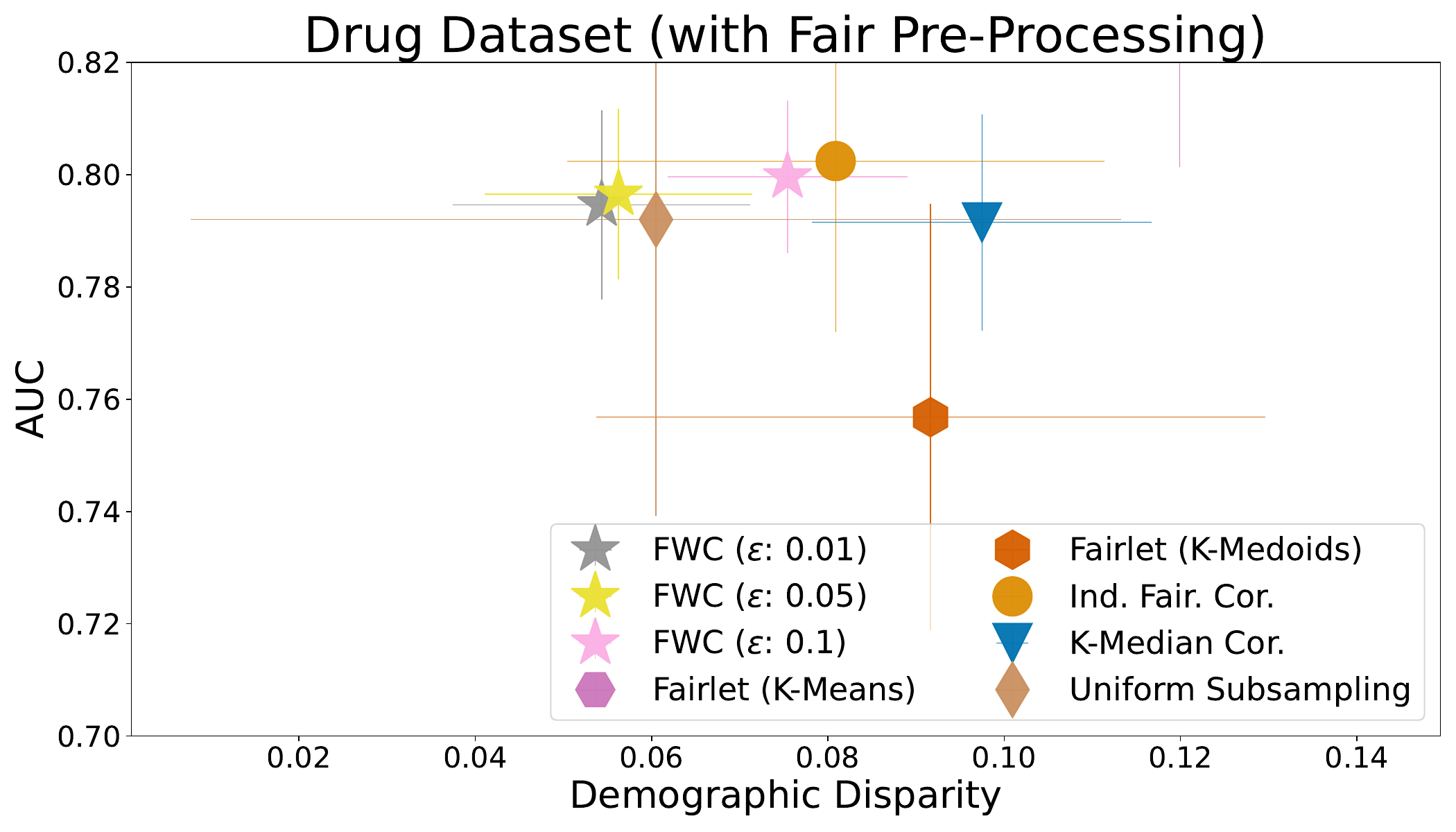}
    \vspace{0.1cm}
    \caption{Similarly to Figure~\ref{fig: supp-all-results}, we report the means and standard deviations over 10 runs of the fairness-utility tradeoff of all methods, indicated by AUC and demographic disparity of a downstream MLP classifier across all datasets (rows) and without (left column) or with (right column) fair pre-processing \cite{kamiran2012data} (excluding \methodname\ , to which no fairness modification is applied after coresets have been generated).}
    \label{fig: supp-all-results-standard-deviations}
\end{figure}

\end{document}

%% file: math_commands.tex
\usepackage{amsmath,amsfonts,bm}

\def\eqref#1{equation~\ref{#1}}

\def\1{\bm{1}}

\def\eps{{\epsilon}}

\DeclareMathAlphabet{\mathsfit}{\encodingdefault}{\sfdefault}{m}{sl}
\SetMathAlphabet{\mathsfit}{bold}{\encodingdefault}{\sfdefault}{bx}{n}

\DeclareMathOperator*{\argmin}{arg\,min}